\documentclass{tlp}

\usepackage{epsfig}
\usepackage{amsmath,amssymb}
\usepackage[latin1]{inputenc}
\usepackage{theorem}
\usepackage{algorithm}
\usepackage{algorithmic}
\usepackage{picinpar}
\usepackage{xspace}
\usepackage{color}

\newtheorem{Theorem}{Theorem}
\newtheorem{Definition}{Definition}
\newtheorem{Example}{Example}

\newtheorem{Corollary}{Corollary}
\newtheorem{Remark}{Remark}
\newtheorem{Proposition}{Proposition}
\newtheorem{Lemma}{Lemma}

\renewenvironment{proof}{
  \textit{\underline{Proof}: }}{\hfill$\Box$\medskip}

\renewcommand{\algorithmicrequire}{\textbf{Input:}}
\renewcommand{\algorithmicensure}{\textbf{Output:}}

\newcommand{\ojo}[1]{#1}

\newcommand{\hugo}[1]{\mbox{$#1$}\xspace}

\newcommand{\tritri}{\scriptsize\hugo{\!\!\triangle}}
\newcommand{\argExp}{\hugo{+^{{\tritri}}}}

\newcommand{\OPnwaccc}{\hugo{-^{\omega}}}

\newcommand{\Arg}{\hugo{\mathcal{A}}}

\newcommand{\Brg}{\hugo{\mathcal{B}}}
\newcommand{\Crg}{\hugo{\mathcal{C}}}
\newcommand{\Drg}{\hugo{\mathcal{D}}}

\newcommand{\ie}{\emph{i.e.,~}}
\newcommand{\eg}{\emph{e.g.,~}}
\newcommand{\ifff}{{\sf iff}\xspace}

\newcommand{\interf}{\lambda^{-}}
\newcommand{\support}{\lambda^{+}}

\newcommand{\Lines}[1]{\hugo{\mathfrak{Lines}_{\scriptsize #1}}}
\newcommand{\exLines}[1]{\hugo{\mathfrak{ALines}_{\scriptsize #1}}}

\newcommand{\allTrees}[1]{\hugo{\mathfrak{Trees}_{\scriptsize #1}}}
\newcommand{\accTrees}[1]{\hugo{\mathfrak{ATrees}_{\scriptsize #1}}}

\newcommand{\bundle}{\hugo{\mathcal{S}_{\scriptsize \PP}(\Arg)}}
\newcommand{\bundleSet}[2]{\hugo{\mathcal{S}_{\scriptsize #2}(#1)}}
\newcommand{\htree}{\hugo{\mathcal{H}_{\scriptsize \PP}}}
\newcommand{\upsegm}[2]{\hugo{#1^{\uparrow}(#2)}}

\newcommand{\upsegmeq}[2]{\lambda^{\uparrow}_{#1}[#2]}
\newcommand{\upsegmeqP}[2]{\lambda'^{\uparrow}_{#1}[#2]}
\newcommand{\upsegmeqPP}[2]{#1^{\uparrow}[#2]}
\newcommand{\attFunct}{\hugo{\ell_{\scriptsize\PP}(\Arg)}}

\newcommand{\dtree}[2]{\hugo{\mathcal{T}_{\scriptsize #2}(#1)}}

\newcommand{\deactsel}{\hugo{\gamma}}
\newcommand{\incise}{\hugo{\sigma}}

\newcommand{\colinc}[2]{#1^{\scriptsize(#2)}}

\newcommand{\powerSet}[1]{\hugo{2^{\scriptsize#1}}}

\newcommand{\selcrit}{\mbox{$\prec_{\tiny\lambda}$}\xspace}
\newcommand{\selcritP}[1]{\mbox{$\prec_{\tiny#1}$}\xspace}
\newcommand{\selcritsub}[1]{\mbox{$\prec_{\tiny\lambda_{#1}}$}\xspace}

\newcommand{\lessAlter}{\hugo{\prec_{\scriptsize[\dtree{\Arg}{\tiny\PP}]}}}
\newcommand{\lesschg}{\hugo{\prec}}

\newcommand{\ALINES}[1]{\hugo{Att(#1)}}
\newcommand{\ALINESPP}{\hugo{Att(\dtree{\Arg}{\PP})}}
\newcommand{\ALINESPPprime}{\hugo{Att(\dtree{\Arg}{\PP'})}}

\newcommand{\incision}{\mbox{$\sigma$}}

\newcommand{\prolog}{{\sc prolog}}
\newcommand{\delp}{de.l.p.}
\newcommand{\DLP}{{\sc DeLP}}

\newcommand{\no}{\mbox{$\sim$}}

\newcommand{\PP}{\hugo{\mathcal{P}}}
\newcommand{\DD}{\mbox{$\Delta$}}
\newcommand{\SSet}{\mbox{$\Pi$}}
\newcommand{\SD}{\mbox{$(\SSet,\DD)$}}

\newcommand{\pair}[2]{\mbox{$(#1,#2)$}}
\newcommand{\SyA}{\mbox{\SSet\ $\cup$ \Arg}}

\newcommand{\Aalpha}{\mbox{\ensuremath{\langle \Arg,\alpha \rangle}}}
\newcommand{\Bbeta}{\ensuremath{\langle \Brg,\beta \rangle}}

\newcommand{\srule}[2]{\mbox{$#1\!\leftarrow#2$}}
\newcommand{\drule}[2]{\mbox{$#1\,\defleftarrow\!#2$}}
\newcommand{\defleftarrow}{{\raise.5pt\hbox{\scriptsize\defleft}}}
\newcommand{\defleft}{\mbox{\bf--\hspace{-1.1pt}\raise-.3pt\hbox{$\prec$} }}

\newcommand{\MT}[2]{\mbox{$\mathcal T_{\scriptsize #1}(#2)$}}

\newcommand{\Ar}[2]{\mbox{$\langle #1,#2 \rangle $}}

\newcommand{\Barg}{\mbox{${\mathcal B}$}}
\newcommand{\Carg}{\mbox{${\mathcal C}$}}
\newcommand{\Darg}{\mbox{${\mathcal D}$}}

\newcommand{\Bq}{\mbox{$\langle \Barg,\beta \rangle $}}
\newcommand{\Ah}{\mbox{$\langle \Arg,\alpha \rangle $}}
\newcommand{\Ap}{\mbox{$\langle \Arg',\alpha' \rangle $}}

\newcommand{\Dnode}{\mbox{$D$}}
\newcommand{\Unode}{\mbox{$U$}}

\newcommand{\OPwparp}{\hugo{*^{\omega}}}

\newcommand{\markline}{\mbox{${\sf m}$}}
\newcommand{\EXT}{\hugo{\mathfrak{Xargs_{\scriptsize\PP}}}}

\newcommand{\ARGS}{\hugo{\mathfrak{Args_{\scriptsize\PP}}}}
\newcommand{\ARGSP}[1]{\hugo{\mathfrak{Args_{\scriptsize #1}}}}

\newcommand{\Ld}{\hugo{\mathcal{L}^{\scriptsize d}}}
\newcommand{\Ls}{\hugo{\mathcal{L}^{\scriptsize s}}}

\newcommand{\alteration}[2]{\hugo{\Lambda_{\scriptsize#2}(#1)}}
\newcommand{\aware}[2]{\hugo{\Theta_{\scriptsize#2}(#1)}}

\newcommand{\open}{\hugo{\circledcirc_{\scriptsize[\PP,\Arg]}}}
\newcommand{\closed}{\hugo{\odot_{\scriptsize[\PP,\Arg]}}}
\newcommand{\closedContext}{\hugo{\widehat{\closed}}}
\newcommand{\openContext}{\hugo{\widehat{\open}}}

\newcommand{\Incisions}{\hugo{\Sigma_{\incise}}}

\begin{document}

\title{Dynamics of Knowledge in \DLP\ through Argument Theory Change}

\author[Moguillansky, Rotstein, Falappa, Garc\'ia and Simari]
	 {Mart\'in O. Moguillansky, Nicol\'as D. Rotstein, \\
	  {\normalsize\normalfont Marcelo A. Falappa, Alejandro J. Garc\'ia and Guillermo R. Simari}\\ \\
          National Research Council (CONICET), AI R\&D Lab (LIDIA) \\
	  Department of Computer Science and Engineering (DCIC) \\
	  Universidad Nacional del Sur (UNS), Argentina, \\
	  \email{\{mom, maf, ajg, grs\}@cs.uns.edu.ar, nicorotstein@gmail.com}}

\submitted{24 June 2010}
\revised{15 August 2011}
\accepted{14 November 2011}

\maketitle

\begin{abstract}
\footnote{To appear in Theory and Practice of Logic Programming (TPLP).}
This article is devoted to the study of methods to change defeasible logic programs (\delp s) which are the knowledge bases used by the Defeasible Logic Programming (\DLP) interpreter. \DLP\ is an argumentation formalism that allows to reason over potentially inconsistent \delp s. 
Argument Theory Change (ATC) studies certain aspects of belief revision in order to make them suitable for abstract argumentation systems. In this article, abstract arguments are rendered concrete by using the particular rule-based defeasible logic adopted by \DLP. 
The objective of our proposal is to define prioritized argument revision operators \textit{\`a la} ATC for \delp s, in such a way that the newly inserted argument ends up undefeated after the revision, thus warranting its conclusion. In order to ensure this warrant, the \delp\ has to be changed in concordance with a minimal change principle. To this end, we discuss different minimal change criteria that could be adopted. Finally, an algorithm is presented, implementing the argument revision operations.
\end{abstract}


\section{Introduction and Background}

The integration of Logic Programming and Defeasible Reasoning has
produced several Knowledge Representation and Reasoning
tools~\cite{GDS2009}.
Among them, Defeasible Logic Programming (\DLP)~\cite{GS04delp} has
been extensively studied and developed in several
applications like~\cite{BDIaaai,explanations07,2007:bookchap:argum,Thimm:2008a,BlackHunter09,Gomez2010},
among others. 
\DLP\ combines an extended Logic Programming representation language
with a dialectical procedure applied to the relevant arguments to
obtain the supported conclusions (see Section~\ref{DeLP}).

By choosing \DLP\ as our formalism, we follow a defeasible form of
reasoning over potentially inconsistent knowledge bases (KBs).
Here, the notion of warrant in argumentation plays the role of the consequence relation: the warranting process evaluates conflicting pieces of knowledge deciding which ones prevail despite the existence of beliefs in opposition. 
A warrant is also identified as the argument's acceptance criterion corresponding to the adopted
argumentative semantics. Among the most influential works on argumentation semantics, we may refer to those over graphs of arguments~\cite{Dung95}, and more recently \cite{baroni.semantics}.
However, the argumentation semantics we adopt follows the idea of dialectical argumentation~\cite{PraVree,Carlos}: arguments trees, namely \textit{dialectical trees}, are built from the argumentation framework with nodes as arguments and edges as attacks, which stand for
sources of inconsistency obtained from the KB. The use of dialectical trees allows to concentrate only on a specific query to build ``on demand'' only those arguments that are somehow related to the query. This kind of semantics allows to construct practical approaches, avoiding the analysis of the complete graph of arguments.


Belief revision~\cite{AGM85,Han99} studies the dynamics of knowledge, coping with the problem of how to change the information standing for the conceptualization of a modeled world, to reflect its evolution. \textit{Revisions}, as the most important change operations, concentrate on the incorporation of new beliefs and their interaction with older ones. A basic set of postulates is usually specified to characterize a rational behavior of the proposed change operations. Among them, minimal change and success have concentrated much research. \textit{Success} specifies the main objective of the change operation (in the case of revisions, the inference of the new belief to incorporate). On the other hand, \textit{minimal change} ensures that the least possible amount of information is modified in order to achieve success.

Argument Theory Change (ATC)~\cite{comma08} applies belief revision concepts to the field of abstract argumentation~\cite{Dung95}. (In abstract argumentation, the logic for arguments and their inner structure is abstracted away.) The main contribution provided by ATC is a revision operator at argument level that revises a theory by an argument seeking for its warrant. To such end, the theory --and thus the set of arguments obtained from it-- is modified in order to guarantee success: the new argument should be accepted by the argumentation semantics. Consequently, different criteria of minimal change should be considered to guarantee a rational behavior of the operator proposed. Among the most relevant uses of ATC, we may refer to hypothetical reasoning, dynamics in negotiation, persuasion, dialogues, strategies, planning, judicial contexts in law, and more.

In this article we apply ATC to handle KBs' dynamics. More specifically, we rely on ATC to make evolve potentially inconsistent KBs without consistency restoration. To such end, we reify the abstract theory of ATC into a particular sort of KBs: defeasible logic programs (\delp s), which are managed by the Defeasible Logic Programming (\DLP) formalism\footnote{The \DLP\ interpreter is available online at {\tt http://lidia.cs.uns.edu.ar/delp\_client}}\cite{GS04delp}. \DLP\ is a rule-based argumentation formalism in which arguments are built from a subset of rules to infer claims. A preliminary approach on the matter of revising \delp s through ATC was introduced in \cite{aaai08}.

The definition of a revision operator over a KB within the classic theory of belief revision usually involves the removal of beliefs from the KB in such a way that a new belief could be consistently incorporated afterwards. Consequently, the new belief ends up inferred by the resulting KB in a consistent manner. In some cases it is necessary to reverse the order in which the revision operator is defined~\cite{HAN05}. That is, the new belief is incorporated possibly introducing an inconsistent intermediate state to the KB, restoring the consistency afterwards by a series of belief removals. Such reversion is used to define the ATC's argument revision. This is necessary, given that in order to pursue warrant of the new argument, its interaction with other arguments from the worked argumentation system needs to be analyzed. Thus, we first incorporate the new argument into the worked set of arguments. However, we are not concerned about the appearance of an intermediate inconsistent state --as is the case in classic belief revision-- but on an intermediate state in which the acceptance of the new argument is not provided. In such a case, ATC provides the necessary elements to change the theory to finally accept the new argument.

Analogously to the usual definition of a revision operator, in which a belief is added to the KB, in the argumentative model of change that we propose, a belief is added to a \delp\ along with the argument that supports it. (An argument is said to support a belief, namely the claim, by considering a minimal set of beliefs inferring it.) In order to accept the new belief by the argumentation semantics, we pursue warrant of the argument supporting such belief. In consequence, the revision theory proposes additional modifications to the program (if necessary) altering the set of arguments that originally interfered with the warrant of the new argument.

The theory of change here proposed is inspired from the AGM model \cite{AGM85}. Due to the usage of argumentation as the base formalism, ATC has to deal with additional, inherent complications arising from the interaction of arguments throughout the warrant process. Chained removal of arguments and undesirable side effects bring about even more difficulty. Moreover, considering warrant as a notion of consequence, ATC also has to take into account non-monotonicity, which is not present when revision is performed over classical logic. Throughout the paper it will be clear that this fact implies a greater amount of theoretical elements, hence the conceptual and notational difficulty is consequently increased.

This article provides a practical approach towards implementing ATC through \DLP. To this end, a \prolog-like algorithm is given, which manipulates rules from a \delp\ following ATC definitions. 
The reader should be aware that this article does not pursue a full formalization according to the classical theory of belief revision. Thus, no representation theorems are to be defined here. Instead, we look at the process of change from the argumentation standpoint taking into account the usual principles of minimal change and success. The full axiomatic characterization of ATC exceeds the scope of this article, however, the interested reader could refer to \cite{jigpal} for ATC's full characterization applied for handling dynamics of knowledge/arguments in propositional bases/argumentation.

This article is organized as follows: \ojo{Section~\ref{sec.motivation} discusses some real examples in which ATC may bring an novel alternative to handle dynamics of knowledge in inconsistent bases.} Section~\ref{overviews} introduces the necessary background upon which our theory relies. Section~\ref{BR}, gives a very brief intuition about belief revision, and Section~\ref{DeLP} describes in a slightly deeper way the \DLP\ formalism. Section~\ref{sec.marking} provides a detailed study on the adopted argumentation semantics, \ie dialectical trees, and their characteristics regarding non-acceptance (rejection) of arguments. Section~\ref{sec.atc} details the complete ATC theory concentrating on its basic elements (Sect.~\ref{sec.atc.basic.elements}), identifying the portions of dialectical trees to be changed and putting particular attention on the minimal change principle (Sect.~\ref{sec.alteration} and Sect.~\ref{sec.atc.set.interrelation}), and providing the argument change operators (Sect.~\ref{sec.operators}). 
In Section~\ref{sec.principles}, some algorithms are given towards a full implementation of this theory upon \DLP. Related work is discussed in Section~\ref{sec:related.work}, and finally, Section~\ref{sec.conclusion} points out some final remarks.

\section{\ojo{Motivation}}\label{sec.motivation}

Dealing with inconsistencies is of utmost importance in areas like medicine and law. For instance, in law trials, two parties to a dispute present contradictory information in a tribunal, standing in favor or against the dispute (in criminal trials this is normally the presumption of innocence). The tribunal decision resolves afterwards the dispute upon presented evidence. This shows the need to consider some kind of paraconsistent semantics in order to appropriately reason over KBs containing contradictory information. 

For some settings, it will be also necessary to provide services for handling dynamics of knowledge with capabilities to tolerate inconsistencies from the KB under consideration. An interesting one arises in promulgation of laws. This usually involves a long process in which articles and principles from previous laws, and even evidence taken from the current state of affairs, may enter in conflict with articles composing the new law. Imagine a base containing knowledge about the complete legal system of a nation, including the National Constitution, the international law, and other political fundamental principles --such as the civil and penal codes, and other minor local laws. Such KB is required to evolve in a way that it incorporates the information conforming the new law, ensuring it to be constitutional. To this end, it is necessary to identify a set of articles and/or principles to be derogated, or amended, as part of the process of promulgation.

As an example, we will refer in a very brief manner to the Argentinean broadcasting media law reformed during 2009. The previous media law, promulgated by the latter \textit{de facto} regime, empowered the government to regulate the different media allowing total control of news. When democracy was restituted, the regulation of media was extended to private investment groups. As years went by, these groups took over majorities of types of media, conforming monopolies in some cases. This brought excessive power to groups with partial interests, allowing them to manipulate the social opinion about the actual government, and even to condition politicians, thus striking to national sovereignty. Article 161 of the new media law became one of the most controversial points, since it forces monopolistic enterprises to get rid of part of their assets in a maximum period of one year. Some enterprises warned that they would be forced to sell off their assets at very low prices. This violates article 17 of the National Constitution which speaks about private property rights. Moreover, some members of the Supreme Court think that article 161 recalls the control over the media exercised by totalitarian regimes, which would violate article 1 of the National Constitution. In fact, such situation could evolve to a distrust state on the principle of legal security. These are just some of the controversial points for which the new media law keeps being studied by the Argentinean parliament at the time of this submission.

Belief revision studies the dynamics of knowledge relying upon consistency restoration. That is, change is applied through change operations like \textit{revisions}, by ensuring a consistent resulting base. 
Observe however, that for the aforementioned case of promulgation of laws, it is mandatory to keep most inconsistencies from the original KB to make it evolve appropriately. This led us to investigate new approaches of belief revision which operate over paraconsistent semantics in order to avoid consistency restoration. Argument Theory Change (ATC) arises with such objective, applying belief revision to argumentation systems.

Among the most relevant uses of ATC, we may refer to hypothetical reasoning, dynamics in negotiation, persuasion, dialogues, strategies, planning, and more. For instance, in scheduling, consider the development of a company's task scheduler. Assume employee assignments are managed by an agent interpreting a KB. The central authority incorporates new tasks to the KB. An agent uses this information to decide to which employees should the new requirements be assigned. Argumentation could deal with such a problem since it would be necessary to reason over inconsistency: conflicts would appear between the relevance of tasks and employees availability. 
A new task with a high level of relevance could be sent to a specific employee for a matter of trust, provoking the reallocation of his previous tasks to other employees. ATC can be useful to implement the re-scheduling process by recognizing how new assignments are in conflict with preexisting ones.

It is important to mention that, unlike typical KB revision models in which a base is revised by a sentence, here we are concerned with the operation of revising a \delp\ by a given argument. This could be a good alternative to revise a KB by a piece of information of higher conceptual complexity. For instance, considering the example given before on promulgation of laws, an ATC model could revise the legal system by including an argument standing for the new law to be promulgated. In this case, arguments to be removed from the original argumentation system would contain different articles from preexisting laws. Thus, the new law is ensured to be constitutional by proposing to derogate other laws or amend them by removing specific articles which are part of arguments to be removed. Naturally, removals from the base are expected to be of less importance than the one for the new law. That is, laws to be amended should never correspond to the National Constitution unless the promulgation of the new law is expected to reform it. In case no minor laws arise to be amended, it is clear that the new one cannot be included as it is and thus, it necessarily needs to be modified to be accepted by the legal system.

\section{Overviews}\label{overviews}

In this section we give an overview of the necessary theoretical background for the ATC machinery. Firstly some elements of the classic belief revision theory are introduced, and afterwards we present the main definitions of \DLP.

\subsection{Belief Revision Overview}\label{BR}

\textit{Belief revision} studies the process of changing beliefs from an epistemic state to accept new information. An \textit{epistemic state} --to which change operations are applicable-- accounts for knowledge in the form of either a belief base or a belief set. A \textit{belief base} (\textit{knowledge base}) is an epistemic state represented by a set of sentences not necessarily closed under logical consequence. On the other hand, a \textit{belief set} is a set of sentences closed under logical consequence. In general, a belief set is infinite, being this the main reason of the impossibility to deal with this kind of sets in a computer. Instead, it is possible to characterize the properties that each change operation should satisfy on any finite representation of an epistemic state.

The classic change operations as seen in the AGM model of theory change \cite{AGM85} --named after their proponents Alchourr\'on, G\"ardenfors, and Makinson-- are known as expansions, contractions, and revisions. An \textit{expansion} incorporates a new belief without guaranteeing a consistent resulting epistemic state. A \textit{contraction} eliminates a belief $\alpha$, and other beliefs making possible its inference, from the epistemic state. Finally, a \textit{revision} adds a new belief $\alpha$ to the epistemic state guaranteeing a consistent outcome always that $\alpha$ is consistent as well. This means that the revision includes a new belief and possibly eliminates others in order to avoid inconsistencies. Revisions are usually defined by means of contractions and expansions: assuming a revision operator ``$*$'', a new belief $\alpha$, and an epistemic state $K$; the resulting epistemic state $K*\alpha$ is ensured to be consistent (unless $\alpha$ is inconsistent) through a contraction ``$-$'' by the complement of the new belief (\ie $\no\alpha$) and an expansion ``$+$'' by $\alpha$. If the epistemic state does not imply $\no\alpha$, then $\alpha$ can be incorporated without loss of consistency. This composition of sub-operations defines a revision operator through the \textit{Levi identity:} $K*\alpha = (K-\no\alpha)+\alpha$ \cite{Lev77,Gar81}. In~\cite{HAN05} the reverse of this identity was studied. This will be helpful for the definition of the model of change proposed here.

The changes applied to the epistemic state might conform to the \emph{minimal change principle}. This notion is one of the \textit{rationality principles of change}. Furthermore, these principles are formalized through a set of \textit{rationality postulates} which are used to guarantee that a change operation behaves in a rational manner. In this paper we concentrate on the specification of the model of change, and only two principles will be taken into account: minimal change and success --a successful revision operation refers to the primacy of new information. 

The AGM model of change specifies an array of theoretical tools to perform revision over belief sets. Nonetheless, since in this work we aim at performing change over defeasible logic programs --which can be seen as a kind of belief bases-- we preferred to rely on Hansson's kernel sets~\cite{HAN94} which were proposed to deal with practical approaches. A \textit{kernel set} is a minimal set of beliefs inferring $\alpha$ (namely, an \textit{$\alpha$-kernel}) from the epistemic state. The \textit{kernel contraction}~\cite{HAN94} --applicable both to belief bases and belief sets-- specifies an operator capable of selecting and eliminating beliefs from each $\alpha$-kernel in $K$, in order to avoid inferring $\alpha$. The relation between the notions of kernel sets and arguments will be clear in the following section.

\subsection{Defeasible Logic Programming Overview}\label{DeLP}

\textit{Defeasible Logic Programming} (\DLP) combines results of \textit{logic programming} and \textit{defeasible argumentation}. A brief description of \DLP\ is included below --for a detailed presentation see~\cite{GS04delp}.
\ojo{A defeasible logic program (or \delp\ for short) is a finite set of facts, strict rules and defeasible rules.} \textit{Facts} are ground literals representing atomic information or the negation of atomic information using the strong negation ``\no''. \textit{Strict rules} represent non-defeasible information noted as $\srule{\beta_0}{\beta_1, \ldots, \beta_n}$, where $\beta_{i}$ for $i \geq 0$ is a ground literal. \textit{Defeasible rules} represent tentative information noted as $\drule{\beta_0}{\beta_1, \ldots, \beta_n}$, where $\beta_{i}$ for $i\geq 0$ is a ground literal. The particular case of a defeasible rule with empty body is called a \emph{presumption}. 
A query is a ground literal that can be posed to a \delp\ to find out whether it is warranted.
The domain of all defeasible rules is denoted as \Ld, and that containing all the strict rules and facts, as \Ls.

When required, a \delp\ \PP\ is denoted \SD, distinguishing the subset \SSet\ of facts and strict
rules, and the subset \DD\ of defeasible rules (see Example~\ref{ex:delp}), \ojo{however, such a distinction is not mandatory for representing \delp s, being also correct the use of a common set of facts, strict rules, and defeasible rules.}

\begin{Example}\label{ex:delp}
Consider the \delp\ $\PP_{\ref{ex:delp}} = \pair{\Pi_{\ref{ex:delp}}}{\Delta_{\ref{ex:delp}}}$:
\begin{center}
$\Pi_{\ref{ex:delp}}$=$\left\{\begin{array}{c} t,z,\\(\srule{p}{t})
\end{array}\right\} \Delta_{\ref{ex:delp}}$=$\left\{\begin{array}{c}
(\drule{\no a}{y}), (\drule{y}{x}), (\drule{x}{z}), \\
(\drule{y}{p}), (\drule{a}{w}), (\drule{w}{y}), \\
(\drule{\no w}{t}), (\drule{\no x}{t}), (\drule{x}{p})
\end{array}\right\}$
\end{center}
\end{Example}

In \DLP, literals can be derived from rules and facts, being a \emph{defeasible derivation} one that uses, at least, a single defeasible rule; and a \emph{strict derivation} one that only uses strict rules or facts.
{\it Strong negation} is allowed in the head of program rules, and hence may be used to represent contradictory
knowledge. We say that two literals are contradictory if they are complementary with respect to strong negation (\eg $a$ and $\no a$). Hence, a set of rules and facts is contradictory if two contradictory literals can be derived from it.
It is important to note that the set \SSet\ (which is used to represent non-defeasible information) must possess certain internal coherence, and therefore, \SSet\ must be a non-contradictory set. 
However, from a program \SD, contradictory literals could be derived by using both kinds of rules, \eg from
\pair{\SSet_{\ref{ex:delp}}}{\DD_{\ref{ex:delp}}} of Example~\ref{ex:delp} it is possible to defeasibly derive $a$ and $\no a$. 

\DLP\ incorporates an argumentation formalism for the treatment of contradictory knowledge that can be defeasibly derived. This formalism allows the identification of contradictory pieces of knowledge, where a \emph{dialectical process} decides which information prevails as warranted. This dialectical process involves the construction and evaluation of \textit{arguments} that either support or interfere with the query under analysis. From a \delp\ $\PP=\SD$, an \emph{argument} \Ah\ is conformed by a minimal set $\Arg\subseteq\Delta$ of defeasible rules that, along with the set $\Pi$ of strict rules and facts from \PP, is not contradictory and derives a certain conclusion $\alpha$.

\begin{Definition}[Argument Structure]\label{def.argument}
Let $\PP = \SD$ be a \delp, \Ah\ is an \textbf{argument structure} for a ground literal $\alpha$ from \PP,
if \Arg\ is a minimal set of defeasible rules ($\Arg \subseteq \DD$) such that:

1. there exists a derivation for $\alpha$ from \SyA, and\label{def.argument.item.claim}

2. the set \SyA\ is non-contradictory.\label{def.argument.item.consistency}

The domain of all argument structures from \PP\ is denoted \ARGS.
If \Ah\ is an argument structure we will also say that \Arg\ is an argument, and that \Arg\ supports $\alpha$.

\end{Definition}


\begin{Example}\label{ex:arg}
From $\PP_{\ref{ex:delp}}$ in Example~\ref{ex:delp}, we can build the following arguments:\\
$\Ar{\Barg_1}{\no a}=\Ar{$\{\drule{\no a}{y},\drule{y}{x},\drule{x}{z}\}$}{\no a}$\\
$\Ar{\Barg_2}{\no a}=\Ar{$\{\drule{\no a}{y},\drule{y}{p}\}$}{\no a}$\\
$\Ar{\Barg_3}{a}=\Ar{$\{\drule{a}{w},\drule{w}{y},\drule{y}{p}\}$}{a}$\\
$\Ar{\Barg_4}{\no w}=\Ar{$\{\drule{\no w}{t}\}$}{\no w}$\\
$\Ar{\Barg_5}{\no x}=\Ar{$\{\drule{\no x}{t}\}$}{\no x}$ \\
$\Ar{\Barg_6}{x}=\Ar{$\{\drule{x}{p}\}$}{x}$
\end{Example}

A \DLP-query $\alpha$  succeeds, \ie it is warranted from a program \PP, if it is possible
to build an argument \Arg\ that supports $\alpha$ and \Arg\ is found to be \textit{undefeated} by the warrant procedure.
This process implements an exhaustive dialectical analysis that involves the construction and evaluation of arguments
that either support or interfere with the query under analysis.
That is, given an argument \Arg\ that supports $\alpha$,  the warrant procedure will evaluate if there
are other arguments that counter-argue or attack  \Arg\ or a sub-argument of \Arg\
(\Carg\ is a \emph{sub-argument} of \Arg\ if $\Carg \subseteq\Arg$).
\ojo{An argument \Bq\ is a \emph{defeater} (or \emph{counter-argument}) for \Ah\ at literal $\beta$ if $\Arg\cup\Brg\cup\Pi$ is contradictory; that is, 
if there exists a sub-argument \Ap\ of \Ah\ such that $\alpha'$ and $\beta$ disagree.}
Two literals disagree when there exist two contradictory literals that have a strict derivation from $\SSet \cup \{ \alpha', \beta\}$.
The literal $\alpha'$ is referred to as the counter-argument point and \Ap\ as the disagreement sub-argument.

\begin{Proposition}\label{prop.ext.empty}
For any \delp\ \PP and any argument $\Aalpha\in\ARGS$, if $\Arg=\emptyset$ then $\Aalpha$ has no defeaters from \PP.
\end{Proposition}
\begin{proof}
%
Since $\Aalpha\in\ARGS$ and $\Arg = \emptyset$, from Def.~\ref{def.argument} we have that $\SSet$ derives $\alpha$. By \emph{reductio ad absurdum}, let us assume that there is an argument $\Bbeta\in\ARGS$ such that \Brg defeats \Arg, hence $\Arg\cup\Brg\cup\SSet$ is contradictory (see counter-argument). Afterwards, since $\Arg = \emptyset$, we have $\Arg\cup\Brg\cup\SSet = \Brg\cup\SSet$. This is absurd, since \Brg is not compliant with cond.~\ref{def.argument.item.consistency} in Def.~\ref{def.argument}. Finally, $\Aalpha$ has no defeaters from \PP.
\end{proof}

To establish if \Arg\ is a non-defeated argument, \emph{defeaters} for \Arg\ are considered.
A counter-argument \Darg\ is a \textit{proper} defeater for \Arg\ if \Darg\ is preferred to \Arg,
or it is a \emph{blocking} defeater if they either have the same strength or are incomparable with respect to the preference criterion used.
It is important to note that in \DLP\ the argument comparison criterion is modular and thus, the most
appropriate criterion for the domain that is being represented can be selected. 
Nevertheless, in the examples given in this paper we will abstract away from this criterion, since it introduces unnecessary complexity. 
Thus, preference between counter-arguments will be given explicitly by enumerating defeats between arguments. 

Since defeaters are arguments, there may exist defeaters for them,
and defeaters for these defeaters, and so on. This will determine sequences of arguments which are referred to as \emph{argumentation lines}.

\begin{Definition}[Argumentation Line]\label{def:argum:line}
Given a \delp\ \PP, and arguments $\Brg_1,\dots,$ $\Brg_n$ from \ARGS, an \textbf{argumentation line} $\lambda$ is any (non-empty) finite sequence $[\Brg_1, \ldots, \Brg_n]$ such that $\Brg_i$ is a defeater of $\Brg_{i-1}$, for $1<i \leq n$. We will say that $\lambda$ is \textbf{rooted in} $\Brg_1$, and  that $\Brg_n$ is the \textbf{leaf} of $\lambda$. 
\end{Definition}

Since argumentation lines are an exchange of opposing arguments, we could think of it as two parties engaged in a dispute, which we call \emph{pro} and \emph{con}.

\begin{Definition}[Set of Con (Pro) Arguments]
Given an argumentation line $\lambda$, the \textbf{set of con} (resp., \textbf{pro}) \textbf{arguments} $\interf$ (resp., $\support$) of $\lambda$ is the set containing all the arguments placed on \textbf{even} (resp., \textbf{odd}) positions in $\lambda$.
\end{Definition}

To avoid undesirable sequences that may represent circular or fallacious reasoning chains, in \DLP\ an argumentation line has to be \emph{acceptable}: it has to be finite, an argument cannot appear twice, there cannot be two consecutive blocking defeaters, and the set of pro (resp., con) arguments has to be non-contradictory, \ie the set of defeasible rules from the union of arguments inside the set of pro (resp., con), along with $\Pi$, does not yield a contradiction. The domain of all acceptable argumentation lines in \PP\ is denoted as \Lines{\PP}. An acceptable argumentation line in a program \PP\ is called \emph{exhaustive} if no more arguments from \PP\ can be added to the sequence without compromising the acceptability of the line. 
For more details on acceptability of argumentation lines, refer to \cite{GS04delp}. The domain of all acceptable and exhaustive argumentation lines in \PP\ is denoted as \exLines{\PP}.

\begin{Remark}\label{remark.lines}
Given a \delp\ \PP, $\exLines{\PP}\subseteq\Lines{\PP}$ holds.
\end{Remark}

\begin{Example}\label{ex:defeaters}
Consider the arguments from Ex.~\ref{ex:arg}, and the following defeat relations: $\Barg_3$ is a proper defeater for $\Barg_2$, and $\Barg_4$ properly defeats $\Barg_3$. From these three arguments we can build the sequence $[\Barg_2, \Barg_3, \Barg_4]$, which is an acceptable and exhaustive argumentation line: it is non-circular, finite, it does not include blocking defeats, $\Barg_2 \cup\Barg_4\cup\Pi_{\ref{ex:delp}}$ is non-contradictory, and no more defeaters can be attached to it.
\end{Example}

Next we introduce the notion of \emph{upper segment}, which identifies subsequences of an argumentation line, from the root to a specific argument in it, determining new (non-exhaustive) argumentation lines. This notion will be central in the argumentative model of change presented in this article.


\begin{Definition}[Upper Segment]\label{def:upper:segment}
Given a \delp\ \PP, and an acceptable argumentation line $\lambda\in\Lines{\PP}$ such that $\lambda = [\Brg_1,\ldots,\Brg_n]$; the \textbf{upper segment} of $\lambda$ wrt. $\Brg_i$ $(1\leq i\leq n)$ is defined as $\upsegmeq{}{\Brg_{i}}=[\Brg_1,\ldots, \Brg_{i}]$, while the \textbf{proper upper segment} of $\lambda$ wrt. $\Brg_i$ is defined as $\upsegm{\lambda}{\Brg_{i}}=[\Brg_1,\ldots, \Brg_{i-1}]$. The proper upper segment of $\lambda$ wrt. $\Brg_1$ is undefined, noted as $\upsegm{\lambda}{\Brg_1}= \epsilon$.
\end{Definition}

In the sequel, we refer to both proper and non-proper upper segments simply as ``upper segment'' and either usage is distinguishable through its notation (round or square brackets, respectively). As stated next, the upper segment of an argument in a line constitutes a (possibly non-exhaustive) argumentation line by itself. 
\ojo{Besides, given a line $\lambda = [\Brg_1, \dots, \Brg_j, \dots, \Brg_n]$, we will say that an argument $\Brg_i$ is \emph{below} (respectively, \emph{above}) $\Brg_j$ iff $j > i$ (respectively, $i > j$).}

\begin{Proposition}\label{prop.uppersegments2lines}
For any $\lambda\in\Lines{\PP}$ and any $\Brg\in\lambda$, it holds $\upsegmeq{}{\Brg}\in\Lines{\PP}$.
\end{Proposition}
\begin{proof}
Straightforward from Definitions \ref{def:argum:line} and \ref{def:upper:segment}, and the notion of acceptable argumentation line. 
\end{proof}

Many argumentation lines could arise from one argument, leading to a tree structure. In a \textit{dialectical tree}, each node (except the root) represents a defeater of its parent, and leaves correspond to arguments with no defeaters in the line. 

\begin{Definition}[Dialectical Tree]\label{def.dTree}
Given a \delp\ \PP, a \textbf{dialectical tree} \dtree{\Arg}{\PP} rooted in \Arg is built by a set $X\subseteq\Lines{\PP}$ of lines rooted in an argument $\Arg\in\ARGS$, such that an argument \Crg in $\dtree{\Arg}{\PP}$ is: (1) a \textbf{node} \ifff $\Crg\in\lambda$,  for any $\lambda\in X$; (2) a \textbf{child} of a node \Brg in $\dtree{\Arg}{\PP}$ \ifff $\Crg\in\lambda$, $\Brg\in\lambda'$, for any $\{\lambda,\lambda'\}\subseteq X$, and $\upsegmeqP{}{\Brg}=\upsegm{\lambda}{\Crg}$. A leaf of any line in $X$ is a \textbf{leaf} in $\dtree{\Arg}{\PP}$. The domain of all dialectical trees from \PP\ is noted as \allTrees{\PP}.
\end{Definition}

%

The set containing all the acceptable and exhaustive argumentation lines rooted in a common argument \Arg will determine the \emph{bundle set} for \Arg.

\begin{Definition}[Bundle Set]\label{def.bundle}
Given a \delp\ \PP, a set $\bundle$ is the \textbf{bundle set} for \Arg from \PP\ \ifff $\bundle$ contains all the lines rooted in \Arg from $\exLines{\PP}$.
\end{Definition}

The objective of a dialectical tree is to evaluate all the information that could determine the warrant status of the root argument. In addition to this, the argumentation lines included in a dialectical tree should be acceptable in order to ensure the exchange of arguments is performed in a sensible manner. This gives place to a restricted version of dialectical tree, called \textit{acceptable}. An acceptable dialectical tree provides a structure for considering all the possible acceptable argumentation lines that can be generated for deciding whether its root argument is defeated.

\begin{Definition}[\ojo{Acceptable Dialectical Tree}]\label{def.tree}
Given a \delp\ \PP, an \textbf{acceptable dialectical tree} \MT{\PP}{\Arg} rooted in an argument $\Arg\in\ARGS$ is a dialectical tree whose argumentation lines belong to the bundle set $\bundle$. We will say that \MT{\PP}{\Arg} is \textbf{determined by} $\bundle$, and identify the \textbf{domain of all acceptable dialectical trees} from \PP\ as \accTrees{\PP}.
\end{Definition}

\begin{Proposition}
Given a \delp\ \PP, $\accTrees{\PP}\subseteq\allTrees{\PP}$ holds.
\end{Proposition}
\begin{proof}
Let $\dtree{\Arg}{\PP}\in\accTrees{\PP}$. From Def.~\ref{def.tree} and Def.~\ref{def.bundle}, we know that every $\lambda\in\dtree{\Arg}{\PP}$ is an acceptable and exhaustive line from \exLines{\PP}. Since $\exLines{\PP}\subseteq\Lines{\PP}$, every such $\lambda$ also belongs to \Lines{\PP} and therefore, from Def.~\ref{def.dTree}, $\dtree{\Arg}{\PP}\in\allTrees{\PP}$.
\end{proof}

From now on, every dialectical tree will be assumed acceptable unless stated otherwise.
Observe that arguments in dialectical trees will be depicted as triangles (labeled with their names) and edges will denote defeat relations. Defeated arguments will be painted in gray, whereas undefeated ones will be white.

\begin{Example}\label{ex.bundle}
The dialectical tree determined by the bundle set $\{[\Arg,\Brg_1,\Brg_2],$ $[\Arg,\Brg_1,\Brg_3],$ $[\Arg,\Brg_1,\Brg_5],$ $[\Arg,\Brg_4,\Brg_5,\Brg_1,\Brg_3]\}$ is depicted on the right. We assume that every defeat
\begin{window}[0,r,{\mbox{\epsfig{file=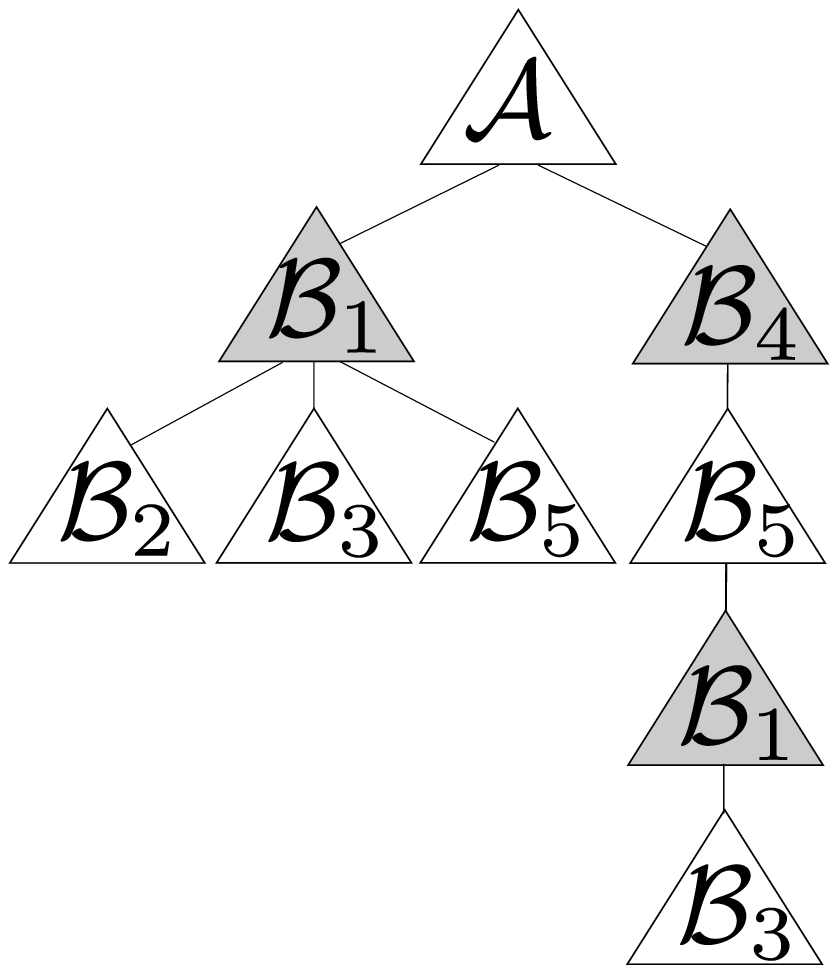,scale=.32}}},{}]
\noindent is proper excepting for $\Brg_5$ and $\Brg_1$ which are blocking defeaters, as well as $\Brg_1$ and $\Brg_2$. Observe that $[\Arg,\Brg_1,\Brg_5,\Brg_1]$ is not acceptable and thus, $\Brg_1$ cannot be child of $\Brg_5$ since it was already introduced in that line. However, nothing prevents $\Brg_1$ being child of $\Brg_5$ in the rightmost line. Note also that, since line $[\Arg,\Brg_4,\Brg_5,\Brg_1,\Brg_2]$ contains two consecutive blocking defeaters, it is not acceptable and therefore, it does not belong to the bundle set. 
\end{window}
\end{Example}


%
%

Given an argument structure \Ah, to decide whether $\alpha$ is warranted, \DLP\ follows a specific marking criterion applied over the dialectical tree  $\dtree{\Arg}{\PP}$. 
This criterion assigns a mark from the domain $\{\Dnode, \Unode\}$ to each node in the tree, where \Dnode\ stands for {\em defeated} and \Unode\ for {\em undefeated}.

\begin{enumerate}\label{marking}
\item all leaves in \MT{\PP}{\Arg} are marked as \Unode; and
\item every inner node \Barg\ of \MT{\PP}{\Arg} will be marked as \Unode\ \ifff every child of \Barg\ is marked as \Dnode; otherwise, \Barg\ is marked \Dnode.
\end{enumerate}


Thus, an argument \Barg\ will be marked as \Dnode\ \ifff it has at least one child marked as \Unode. Finally, if the root argument \Arg\ is marked as \Unode\ then we say that \MT{\PP}{\Arg} \emph{warrants} $\alpha$ and that  $\alpha$ is \emph{warranted} from \PP. When no confusion arises, we will refer to \Arg\ instead of $\alpha$, saying that \Arg\ is warranted. We call \MT{\PP}{\Arg} a \emph{warranting tree} if it warrants \Arg, otherwise we call it a \emph{non-warranting tree}. For instance, in Example~\ref{ex.bundle}, argument \Arg is warranted given that it ends up undefeated from its dialectical tree.

\begin{Example}\label{ex:tree}
From the \delp\ $\pair{\SSet_{\ref{ex:delp}}}{\DD_{\ref{ex:delp}}}$ of
Ex.~\ref{ex:delp} we can consider a new program $\PP_{\ref{ex:tree}} = \pair{\SSet_{\ref{ex:delp}}}{\DD_{\ref{ex:delp}} \cup \{\drule{a}{x}\}}$, from which we can build the arguments in Ex.~\ref{ex:arg} along with a new argument $\Ar{\Arg}{a} = \Ar{\{(\drule{a}{x}),(\drule{x}{z})\}}{a}$. The defeat relations are: $\Brg_1$, 
\begin{window}[0,r,{\mbox{\epsfig{file=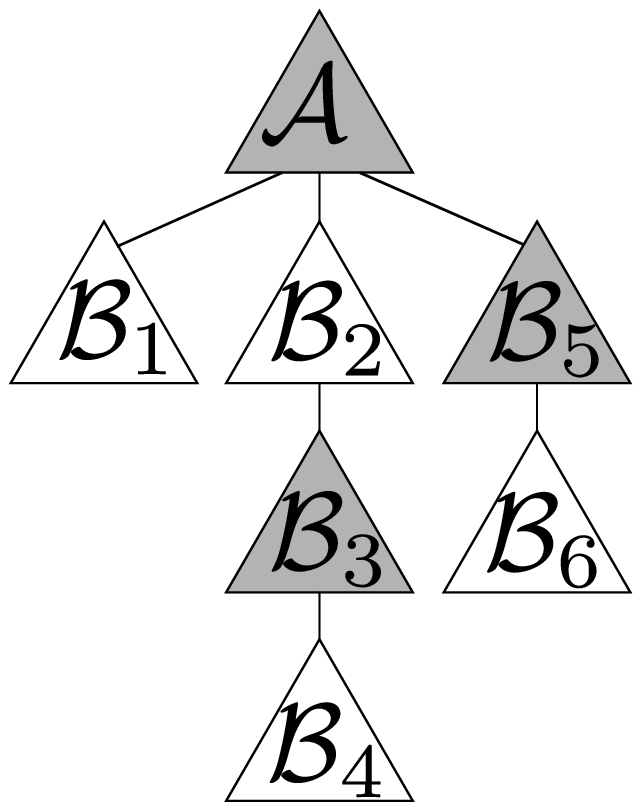,scale=.32}}},{}]
\noindent $\Barg_2$ and $\Barg_5$ defeat $\Arg$; $\Barg_3$ defeats $\Barg_2$; $\Barg_4$ defeats $\Barg_3$; and $\Barg_6$ defeats $\Barg_5$. The non-warranting tree \MT{\PP_{\ref{ex:tree}}}{\Arg} is depicted on the right.
For simplicity, those arguments that can be built from $\PP_{\ref{ex:tree}}$ but do not appear in the tree \MT{\PP_{\ref{ex:tree}}}{\Arg} are assumed to be defeated by the corresponding arguments that do appear. For instance, we assume arguments \Arg and $\Brg_3$ to be preferred over argument $\Ar{$\{\drule{\no a}{y},\drule{y}{x},\drule{x}{p}\}$}{\no a}$, which in consequence is not a defeater of \Arg\ nor $\Barg_3$.
\end{window}
\end{Example}

\section{About Argumentation Lines and Marking}\label{sec.marking}


Dialectical trees are composed by argumentation lines. In particular we are interested in distinguishing those lines that determine the defeat of the root argument, which we call \emph{attacking lines}. In \cite{ijcai09}, a characterization of the marking criteria has introduced the possibility of abstracting away from any specific marking. In that article, different specific marking criteria have been also studied, allowing the full characterization of argumentation lines, and specially \textit{attacking lines} have been defined on top of their \textit{marking sequence}. There, the specific case of the \DLP\ marking criterion, and the morphology of the \DLP\ attacking lines, were analyzed in detail. 

\begin{Definition}[Marking Sequence]\label{def:markingSeq}
Given a \delp\ \PP, and an argumentation line $\lambda = [\Barg_1,\dots,\Barg_n]$ belonging to the marked dialectical tree \dtree{\Barg_1}{\PP}; the \textbf{marking function} $\markline:\Lines{\PP}\times\allTrees{\PP}\longrightarrow\{U,D\}^n$ determines a sequence $\markline(\lambda,\dtree{\Barg_1}{\PP}) = [m_1,\dots,m_n]$ such that each $m_i$ is the mark of the corresponding argument $\Barg_i$ according to \dtree{\Barg_1}{\PP}.
\end{Definition}



Observe that the marking of a line is not considered individually, but in concordance with the context provided by the tree it belongs to. For instance, in Ex.~\ref{ex:tree} the line $[\Arg,\Barg_5,\Barg_6]$ does not have the marking sequence $[U,D,U]$ but the sequence $[D,D,U]$, since $\Barg_1$ (from line $[\Arg,\Barg_1]$) is an undefeated defeater for $\Arg$, which is thus defeated (marked as $D$). Observe that the marking sequence of a line can be simply a $U$ (never a single $D$), or begin with a $U$ or a $D$ and then include an arbitrarily long alternation of $D$s and $U$s, and even at some point it could repeat any number of $D$s (never a $U$). Finally, any marking sequence always ends with a $U$.

\begin{Proposition}\label{prop:U-rep}
Given a \delp\ \PP, and a dialectical tree $\dtree{\Barg_1}{\PP}\in\allTrees{\PP}$, the following conditions are met:
\begin{description}
\item 1. No argumentation line has a repetition of $Us$ in its marking sequence. That is, there is no $\lambda\in\dtree{\Barg_1}{\PP}$ such that $\markline(\lambda,\dtree{\Barg_1}{\PP})=[\ldots,\Unode,\Unode,\ldots]$.
\item 2. Argumentation lines may repeat $\Dnode$s in their marking sequence. That is, it may be the case that there is $\lambda\in\dtree{\Barg_1}{\PP}$ such that $\markline(\lambda,\dtree{\Barg_1}{\PP})=[\ldots,\Dnode,\Dnode,\ldots]$.
\item 3. The marking sequence of every argumentation line ends in a \Unode. That is, for every $\lambda\in\dtree{\Barg_1}{\PP}$ it holds $\markline(\lambda,\dtree{\Barg_1}{\PP})=[\ldots,\Unode]$.
\item 4. In the marking sequence of every argumentation line, a \Unode\ is followed by a \Dnode\ unless the \Unode\ stands for the leaf argument. That is, for every $\lambda\in\dtree{\Barg_1}{\PP}$ if $\markline(\lambda,\dtree{\Barg_1}{\PP})=[\dots, \Unode, m, \ldots]$ then $m=\Dnode$. 
\end{description}
\end{Proposition}
\begin{proof}
The corresponding proof for each item follows straightforwardly from the definition of the marking criterion given in Sect.~\ref{DeLP}. That is the case of items 3) and 4). For 1) and 2), the following sketches may clarify this assertion.

1) From the adopted marking criterion, an inner node \Brg is marked \Unode\ \ifff every child of \Brg is marked \Dnode. This condition makes impossible to have a mark \Unode\ for both a node and some of its children.

2) From the adopted marking criterion, an inner node \Brg is marked \Dnode\ \ifff there is at least one child marked \Unode. However, other children of \Brg may be marked as \Dnode, thus determining marking sequences with repetitions of \Dnode s.
\end{proof}

From the proposition above, argumentation lines may be classified through regular expressions, which typify the lines according to their marking sequence. Depending on the dialectical tree being warranting or not, we will distinguish two main different sorts of regular expressions.

\begin{Definition}[Warranting and Non-warranting Lines]\label{def.w-lines}
Given a \delp\ \PP and the tree $\dtree{\Arg}{\PP}\in\allTrees{\PP}$; for any $\lambda\in\dtree{\Arg}{\PP}$, if the marking sequence $\markline(\lambda,\dtree{\Arg}{\PP})$ conforms to the regular expression $U(D^+U)^*$ (resp., $(D^+U)^+$) then line $\lambda$ is referred to as \textbf{warranting} (resp., \textbf{non-warranting}).
\end{Definition}

\begin{Proposition}\label{prop.mark.argline}
Given the \delp\ \PP and $\dtree{\Arg}{\PP}\in\allTrees{\PP}$; for any $\lambda\in\dtree{\Arg}{\PP}$, $\lambda$ is warranting (according to $\markline(\lambda,\dtree{\Arg}{\PP})$) \ifff \dtree{\Arg}{\PP} is warranting.
\end{Proposition}
\begin{proof}
To prove that those regular expressions corresponding to warranting/non-warranting lines are obtained from the \DLP\ marking criterion, we should obtain a finite automaton equivalent to the regular expression at issue. From such automaton a regular grammar can be defined. Afterwards, by induction, this grammar has to be shown to conform a given regular language which should be proved to be obtained from the marking criterion specified on page~\pageref{marking}, Sect.~\ref{DeLP}. The complete formal proof was left out from this article due to space reasons.

The proof for $\lambda\in\dtree{\Arg}{\PP}$ being a warranting line \ifff \dtree{\Arg}{\PP} is warranting follows straightforwardly from the definition of warranting trees on page~\pageref{marking} and Def.~\ref{def.w-lines}.
\end{proof}

Two different sorts of regular expressions were given in the proposition above depending on whether the dialectical tree upon consideration is warranting or not. In addition, we can go further in each case, and analyze particular configurations to isolate the relevant situations for the specification of this model of change. Regarding a warranting tree, we will make no distinction among its argumentation lines, since the objective of the change method we propose is to achieve warrant, and in such a tree there is nothing to be done. On the other hand, when we consider a non-warranting tree, as said before, we are interested in the characterization of those lines that are somehow responsible for the root argument to be defeated, \ie \textit{attacking lines}. The following proposition distinguishes the two different types of non-warranting lines according to their marking sequence: those that have at least one repetition of \Dnode, which we call \emph{D-rep lines}, and those having a perfect alternation of \Dnode s and \Unode s, which we call \textit{alternating lines}. Moreover, this characterization will be shown to be complete afterwards in Prop.~\ref{prop.mark.non-warranting.argline}.

\begin{Definition}[D-rep and Alternating Lines]\label{def.d-rep.alternating.lines}
Given a \delp\ \PP and a non-warranting tree $\dtree{\Arg}{\PP}\in\allTrees{\PP}$; for any $\lambda\in\dtree{\Arg}{\PP}$, if the marking sequence $\markline(\lambda,\dtree{\Arg}{\PP})$ conforms to the regular expression $(DU)^*(D^+(DU)^+)^+$ (resp., $(DU)^+$) then line $\lambda$ is referred to as \textbf{D-rep} (resp., \textbf{alternating}).
\end{Definition}

\begin{Definition}[\ojo{D-rep Sequence}]\label{def.d-rep.sequence}
Given a \delp\ \PP and a non-warranting tree $\dtree{\Arg}{\PP}\in\allTrees{\PP}$; for any D-rep line $\lambda\in\dtree{\Arg}{\PP}$ where $\lambda=[\Arg,\ldots,\Brg_1,\ldots,\Brg_k,\ldots]$, any subsequence $\Brg_1,\ldots,\Brg_k$ of arguments whose marking is a consecutive subsequence of \Dnode\ nodes in $\markline(\lambda,\dtree{\Arg}{\PP})$ is referred to as \textbf{D-rep sequence} if it holds that either $\Brg_1$'s parent in $\lambda$ is marked as \Unode\ or $\Brg_1=\Arg$ and that $\Brg_k$'s defeater in $\lambda$ is marked as \Unode. Argument $\Brg_1$ is referred to as the \textbf{head} of the D-rep sequence. The D-rep sequence in $\lambda$ which is closer to the root argument \Arg is referred to as the \textbf{uppermost D-rep sequence}.
\end{Definition}


\begin{Proposition}\label{prop.mark.non-warranting.argline}
Given a \delp\ \PP, and the non-warranting tree $\dtree{\Arg}{\PP}\in\allTrees{\PP}$; for any $\lambda\in\dtree{\Arg}{\PP}$, $\markline(\lambda,\dtree{\Arg}{\PP})$ conforms either to a D-rep or an alternating line.
\end{Proposition}
\begin{proof}
The proof showing that the regular expressions corresponding to D-rep/ alternating lines are obtained from the \DLP\ marking criterion, follows similarly to the proof given for Prop.~\ref{prop.mark.argline}. 

On the other hand, for the proof showing that this characterization of non-warranting lines is complete, we will asssume a non-warranting dialectical tree rooted in an argument \Arg, thus we have that \Arg is marked as \Dnode. From Prop.~\ref{prop:U-rep}.3, leaves are marked as \Unode, hence, \Arg has at least one child. Any child of an argument marked as \Dnode\ is (*) either marked  (1) \Unode, or (2) \Dnode. In the latter case, the node is not a leaf (see Prop.~\ref{prop:U-rep}.3), 
and the line would be a D-rep. Regarding (1), if it is the case of a leaf, then we have an alternating line, whereas if it is an inner node, the only option for its child is to be marked as \Dnode\ (see Prop.~\ref{prop:U-rep} items 1 and 4). In this case, its child might be marked as either \Unode, or \Dnode. By recursively following this construction from (*), we have only alternating lines or D-rep lines.
\end{proof}

In order to finally identify attacking lines, it is necessary to study the relation between these two sorts of non-warranting lines. D-rep and alternating lines are interrelated in Lemma~\ref{lemma.D-rep.alternating}, as shown below. But firstly, let us introduce the notion of \emph{adjacency} among lines to provide appropriate theoretic terminology. Two (or more) argumentation lines are referred to as \textit{adjacent} if they share a common upper segment containing one or more arguments. Finally, we refer as \textit{adjacency point} to the last argument in the common upper segment of two (or more) adjacent lines.

\begin{Definition}[Adjacency]\label{def:adjacent}
Given a \delp\ \PP, two acceptable argumentation lines $\lambda_1\in\Lines{\PP}$ and $\lambda_2\in\Lines{\PP}$ are said to be \textbf{adjacent} at \Barg\ \ifff $\lambda_1^{\uparrow}(\Barg_1) = \lambda_2^{\uparrow}(\Barg_2) = [\Arg,\dots,\Barg]$; where $\Barg_1\in\lambda_1$, $\Barg_2\in\lambda_2$, and $\Barg_1\neq\Barg_2$. Argument \Barg\ is said to be the \textbf{adjacency point} between $\lambda_1$ and $\lambda_2$.
\end{Definition}

\begin{Example}\label{ex.upSegm.adjLines}
Consider the dialectical tree depicted on the right. Argumentation lines $\lambda_1$ and $\lambda_2$ are adjacent at argument $\Barg_1$ since both upper segments $\upsegm{\lambda_1}{\Barg_2}$ and $\upsegm{\lambda_2}{\Barg_4}$ 
\begin{window}[0,r,{\mbox{\epsfig{file=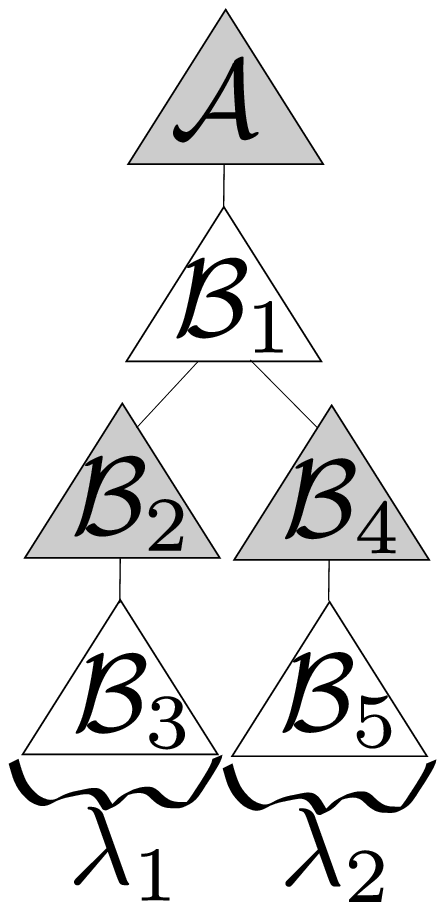,scale=.32}}},{}]
\noindent are equal $[\Arg,\Barg_1]$. Argument $\Barg_1$ is the adjacency point between $\lambda_1$ and $\lambda_2$.
Regarding Ex.~\ref{ex:tree}, the three Lines in the tree are adjacent at the root argument \Arg. 
From Ex.~\ref{ex.bundle}, four lines appear: $\lambda_1=[\Arg,\Brg_1,\Brg_2]$, $\lambda_2=[\Arg,\Brg_1,\Brg_3]$, $\lambda_3=[\Arg,\Brg_1,\Brg_5]$, and $\lambda_4=[\Arg,\Brg_4,\Brg_5,\Brg_1,\Brg_3]$. Note that line $\lambda_4$ is adjacent to lines $\lambda_1$, $\lambda_2$, and $\lambda_3$, at the root argument \Arg. However, lines $\lambda_1$, $\lambda_2$, and $\lambda_3$, are pairwise adjacent at argument $\Brg_1$.
\end{window}
\end{Example}

\begin{Lemma}\label{lemma.D-rep.alternating}
\ojo{Every D-rep line has an adjacent alternating line whose adjacency point is the head of its uppermost D-rep sequence.}
\end{Lemma}
\begin{proof}
Consider a \delp\ \PP, a non-warranting tree $\dtree{\Arg}{\PP}\in\allTrees{\PP}$, and a D-rep line $\lambda=[\Arg,\ldots,\Brg_1,\Brg_2,\ldots,\Brg_{k-1},\Brg_k,\ldots]$. For the marking sequence $\markline(\lambda,\dtree{\Arg}{\PP})$, let us assume, without loss of generality, that for every $1\leq i < k$, arguments $\Brg_i$ are marked as \Dnode, $\Brg_k$ as \Unode, and from the root argument \Arg down to $\Brg_1$ (both marked as \Dnode), we have a perfect alternation of \Dnode s and \Unode s (see Prop.~\ref{prop.mark.non-warranting.argline}). In the rest of the proof, we will refer to the sequences of \Dnode s such as $\Brg_1,\ldots,\Brg_{k-1}$ as a \textit{D-rep sequence} \ojo{(see Def.~\ref{def.d-rep.sequence})}. In a D-rep sequence, the last \Dnode\ argument ($\Brg_{k-1}$) is defeated because it has a child ($\Brg_k$) that is undefeated. On the other hand, for its parent ($\Brg_{k-2}$) the situation is different: since $\Brg_{k-2}$ is also marked as \Dnode\ the only option we have for $\Brg_{k-2}$ is to be the adjacency point with another line $\lambda'\in\dtree{\Arg}{\PP}$ such that $\upsegm{\lambda}{\Brg_{k-1}}=\upsegm{\lambda'}{\Brg'}$ and the mark of $\Brg'\in\lambda'$ is \Unode. Hence, since $\Brg'$ is a child of $\Brg_{k-2}$ in \dtree{\Arg}{\PP}, the adjacency point $\Brg_{k-2}$ ends up marked as \Dnode. 


The same reasoning follows for each one of the $\Brg_i$ arguments ($1\leq i < k$), \ie for each argument in the D-rep sequence. In particular, when considering the uppermost \Dnode\ in the uppermost D-rep sequence (\ie $\Brg_1$) since its child $\Brg_2$ is marked as \Dnode, we necessarily have an adjacent line $\lambda''\in\dtree{\Arg}{\PP}$ turning the adjacency point $\Brg_1$ to \Dnode. There is an argument $\Brg''\in\lambda''$ marked as \Unode\ which is a child of $\Brg_1$ in \dtree{\Arg}{\PP}. Note that the marking sequence of $\lambda''$ in its upper segment $\upsegmeqPP{\lambda''}{\Brg''}$ has a perfect alternation of \Dnode s and \Unode s. 

Below $\Brg''$, $\lambda''$'s marking sequence may contain a D-rep sequence, but in that case $\lambda''$ would be a D-rep and by following the same reasoning we necessarily have adjacent lines with adjacency points in each one of the \Dnode s of $\lambda''$'s D-rep sequence. Since dialectical trees have a finite number of argumentation lines, this process necessarily ends with a line that is not D-rep. Thus, from Def.~\ref{def.d-rep.alternating.lines} and Prop.~\ref{prop.mark.non-warranting.argline}, an alternating line $\lambda^n\in\dtree{\Arg}{\PP}$ appears,  adjacent to $\lambda''$ below $\Brg''$ with adjacency point at the first \Dnode\ of the D-rep sequence. But note that $\lambda^n$ is also adjacent to $\lambda$ at $\Brg_1$. 
Hence, for every D-rep line there is an adjacent line with adjacency point \Dnode\ (the first \Dnode\ corresponding to the head of the uppermost D-rep sequence) that is an alternating line.
\end{proof}

\ojo{Observe that reversing the Lemma~\ref{lemma.D-rep.alternating} does not hold since, for instance, an even-length single-line tree has one alternating line and no D-rep.}

\begin{Lemma}\label{lemma.adj.point.pro}
\ojo{The adjacency point between an alternating and a D-rep line is a pro argument.}
\end{Lemma}
\begin{proof}
From Lemma~\ref{lemma.D-rep.alternating} we know the adjacency point between an alternating and a D-rep line is the head of the D-rep sequence. This means that the adjacency point is marked as \Dnode. Afterwards, it is easy to see that the adjacency point is a pro argument given that according to Def.~\ref{def.d-rep.alternating.lines}, only pro arguments are marked as \Dnode\ in the alternating lines.
\end{proof}

\begin{Theorem}\label{theorem:warrant}
Given the dialectical tree $\dtree{\Arg}{\PP}\in\allTrees{\PP}$, \dtree{\Arg}{\PP} is warranting \ifff there is no alternating line $\lambda\in\dtree{\Arg}{\PP}$.
\end{Theorem}
\begin{proof}
\ojo{$\Rightarrow$) According to Prop.~\ref{prop.mark.argline}, a warranting tree has no alternating line, which is a kind of non-warranting line (see Def.~\ref{def.w-lines} and Def.~\ref{def.d-rep.alternating.lines}). On the other hand, if the tree is non-warranting, only two different kinds of line may appear (see Prop.~\ref{prop.mark.non-warranting.argline}): alternating and D-rep. If we have some alternating line then we are done. On the other hand, if we assume to have some D-rep, according to Lemma~\ref{lemma.D-rep.alternating}, we necessarily have an alternating line to which the D-rep is adjacent.}


\noindent$\Leftarrow$) Assuming we have no alternating lines, the root argument is marked either as \Dnode\ or \Unode. For the latter case, it is clear that the tree is warranting. For the former, the tree is non-warranting. Thus, from Def.~\ref{def.d-rep.alternating.lines} and Prop.~\ref{prop.mark.non-warranting.argline}, if no alternating lines exist, only D-rep lines could appear. However, this is not possible, from Lemma~\ref{lemma.D-rep.alternating}. Hence, we reach an absurdity by assuming possible to have a non-warranting tree without alternating lines. On the other hand, if we have at least one alternating line, then the root argument is marked as \Dnode, and the tree is non-warranting.
\end{proof}

From Theorem~\ref{theorem:warrant}, we can finally ensure that alternating lines are the ones threatening the warrant status of dialectical trees. Thus, from now on, alternating lines will be referred to as \emph{attacking lines}.

\begin{Definition}[Attacking Line]\label{def:attackingLine}
Given a \delp\ \PP and an argumentation line $\lambda$ in $\dtree{\Arg}{\PP}\in\allTrees{\PP}$, for any argument $\Arg\in\ARGS$; $\lambda$ is an \textbf{attacking line} in \dtree{\Arg}{\PP} \ifff $\markline(\lambda,\dtree{\Arg}{\PP})$ corresponds to an alternating line.
\end{Definition}

\begin{Corollary}\label{corollary.warrant}
Given the dialectical tree $\dtree{\Arg}{\PP}\in\allTrees{\PP}$, \dtree{\Arg}{\PP} is warranting \ifff there is no attacking line $\lambda\in\dtree{\Arg}{\PP}$.
\end{Corollary}

\begin{Example}
Consider the dialectical tree in Ex.~\ref{ex:tree} with lines $\lambda_1=[\Arg,\Barg_1]$, $\lambda_2=[\Arg,\Barg_2,\Barg_3,\Barg_4]$, and $\lambda_3=[\Arg,\Barg_5,\Barg_6]$. Two  attacking lines appear: $\lambda_1$ and $\lambda_2$. Note that the tree containing only line $\lambda_3$ warrants \Arg, since its marking sequence would be $[U,D,U]$. However, by considering trees with either $\lambda_1$ or $\lambda_2$ (or both), none would end up warranting.
\end{Example}

Clearly, this definition for an attacking line is totally dependent on the marking criterion adopted. An intuitive notion of attacking lines may be given through a minimal set $\Lambda\subseteq X$ of argumentation lines in a dialectical tree (built with lines from $X$) such that a new dialectical tree built with lines from $X\setminus\Lambda$ ends up warranting. However, when analyzing the warrant status of a tree, we necessarily need to apply the marking procedure, and thus the recognition of attacking lines would again depend on their marking sequences. 

Note that the specific definition of attacking lines in \DLP\ is quite natural: an argumentation line is attacking if its pro (resp., con) arguments are defeated (resp., undefeated). Recall that pro (resp., con) arguments are in favor (resp., against) of the main issue being disputed: the root argument. Thus, such a line provides a reasoning chain which is entirely against the root argument, and therefore, it is sensible to consider the dialectical tree to be non-warranting. 

We pursue a theory that recognizes the changes to be applied to a \delp\ \PP in order to turn a non-warranting tree \dtree{\Arg}{\PP} into warranting. The study of argumentation lines along with their marking sequences, and the notion of attacking lines, aids the definition of our argumentative model of change. These models handle the dynamics of argumentative knowledge through the variation of the set of available arguments. In ATC, dynamics in the argumentation theory is handled through the \textit{alteration} of some argumentation lines. Alterations could be carried out in a variety of ways: a kind of alteration of a line $\lambda$ is the removal (deactivation) of an argument from $\lambda$, whereas a more complex choice is to add (activate) a defeater to some argument in $\lambda$. In this article, we assume only the former alternative.

The main intuition behind the \emph{deactivating} method for argument revision is to alter each attacking line by deactivating an argument in it. When this alteration turns the line to non-attacking, we refer to it as an \textit{effective alteration}. As was shown in Corollary~\ref{corollary.warrant}, a dialectical tree free of attacking lines is a warranting tree. Hence, after the revision, no argumentation lines will threaten the warrant of the root argument since no attacking lines will be left unaltered.



\begin{Definition}[Line Alteration]\label{def.line.alt}
Given a \delp\ $\PP$, a set $\Gamma\subseteq\PP$, and a line $\lambda\in\exLines{\PP}$; the removal $\PP\setminus\Gamma$ provokes the \textbf{alteration} of $\lambda$ on \Brg\ \ifff there is some $\Gamma'\subseteq\Gamma$ such that $\Gamma'\subseteq\Brg$, with $\Brg\in\lambda$, and for every $\Crg\in\upsegm{\lambda}{\Brg}$, it holds $\Gamma\cap\Crg=\emptyset$.
\end{Definition}

\ojo{The definition for a line alteration individualizes the argument \Brg\ that will disappear once rules in $\Gamma$ are removed. This argument is such that no other argument in \Brg's upper segment disappears. Hence, the proper upper segment of \Brg\ would be the resulting argumentation line from such an alteration. The following example illustrates the notion of \textit{line alteration} presented in Def.~\ref{def.line.alt}.}

\begin{Example}\label{ex.line.alteration}\ 
\begin{window}[3,r,{\mbox{\epsfig{file=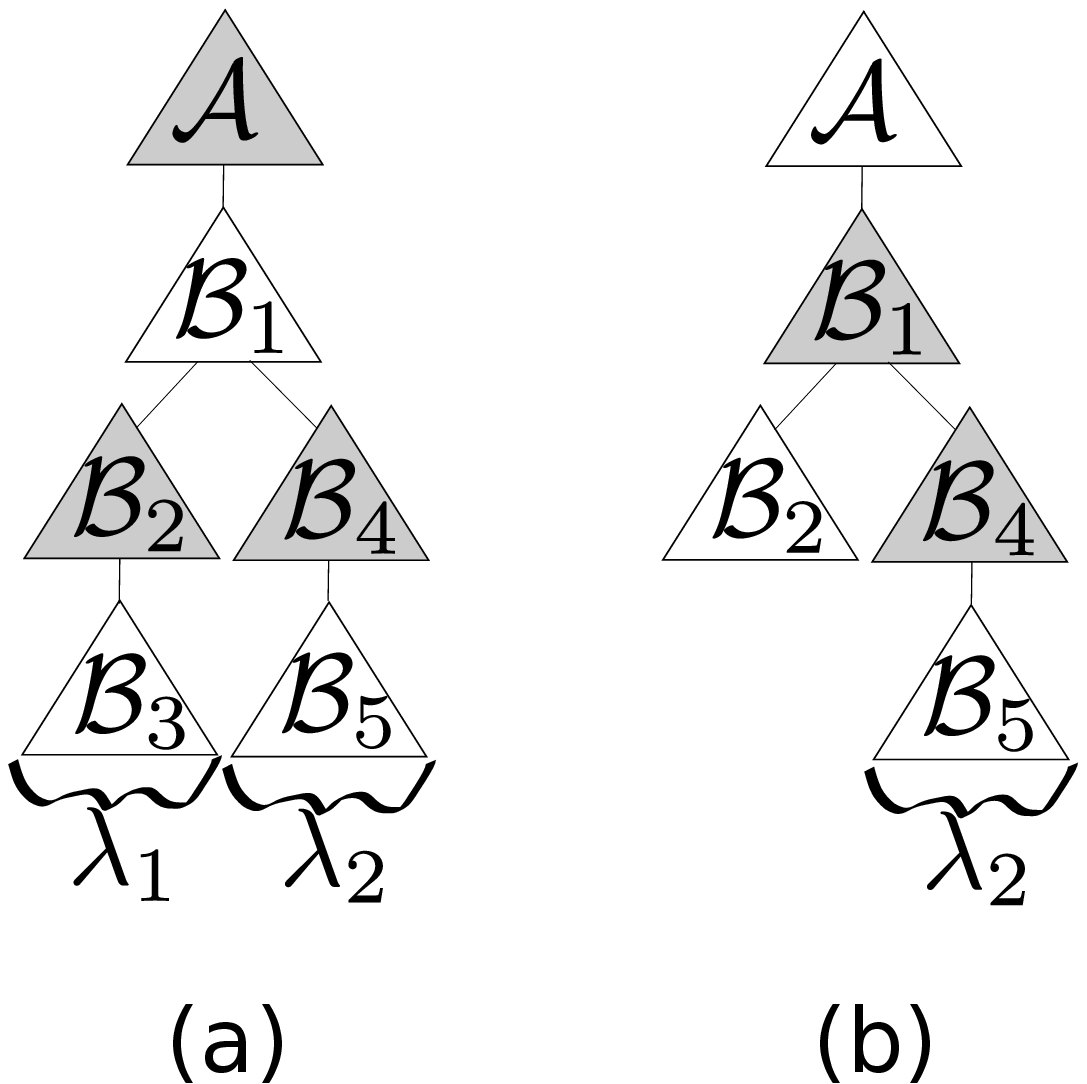,scale=.32}}},{}]
\ojo{Let us assume that the dialectical tree $\dtree{\Arg}{\PP}\in\accTrees{\PP}$, depicted on the right in fig.~(a), is built from a given \delp\ \PP. If we consider the \delp\ $\PP'$ resulting from the removal $\PP\setminus\Gamma$, for some non-empty $\Gamma\subseteq\Brg_3$; we obtain the dialectical tree $\dtree{\Arg}{\PP'}\in\accTrees{\PP'}$ depicted in fig.~(b). It is important to note that the removal $\PP\setminus\Gamma$ provokes the alteration of line $\lambda_1$ yielding the altered line $\upsegm{\lambda_1}{\Barg_3}=[\Arg,\Brg_1,\Brg_2]$ which is exhaustive in $\PP'$ but not in $\PP$.\\
\indent As a different example, if we consider $\PP'=\PP\setminus\Gamma$, where $\Gamma\subseteq\Brg_2$; the resulting dialectical tree will only contain a single argumentation line: $\lambda_2$. This is so, given that $\upsegm{\lambda_1}{\Barg_2}=[\Arg,\Brg_1]$ is not exhaustive in $\PP'$ since it is an upper segment of $\lambda_2$, \ie $\upsegm{\lambda_1}{\Barg_2}=\upsegm{\lambda_2}{\Barg_4}$.}
\end{window}
\end{Example}

Given a \delp\ $\PP$ and a line $\lambda\in\exLines{\PP}$, if $\Gamma=\emptyset$, then according to Def.~\ref{def.line.alt}, $\lambda$ might be considered \textit{altered} on any of its arguments from $\PP\setminus\Gamma$. Observe however that $\lambda$ does not modify its configuration. 
\ojo{Let us see the particular case, according to Def.~\ref{def.line.alt}, with more detail: if $\Gamma=\emptyset$ then $\PP\setminus\Gamma$ (\ie \PP) provokes the \textit{alteration} of $\lambda$ on \Brg\ \ifff there is some $\Gamma'\subseteq\Gamma$ (\ie $\Gamma'=\emptyset$) such that $\Gamma'\subseteq\Brg$, with $\Brg\in\lambda$, and for every $\Crg\in\upsegm{\lambda}{\Brg}$, it holds $\Gamma\cap\Crg=\emptyset$ (given that $\Gamma=\emptyset$). }
This kind of line alterations (in which $\Gamma = \emptyset$) will be allowed just for theoretical matters, and will be referred to as \textit{null line alterations}. 


\ojo{Recall that the objective of altering lines is to turn (potential) attacking lines into non-attacking. When this happens, the alteration is referred to as \emph{effective}. However, further considerations have to be taken into account for an alteration to be considered effective.}


\begin{Definition}[Effective Alteration]\label{def.eff.alt}
Given two \delp s \PP\ and $\PP'$, and two lines $\lambda\in\exLines{\PP}$ and $\lambda'\in\exLines{\PP'}$ such that $\lambda'$ is the line alteration of $\lambda$ on an argument \Brg; $\lambda$ is \textbf{effectively altered} on \Brg\ \ifff $\lambda'=\upsegm{\lambda}{\Brg}$ and if $\lambda'\in\dtree{\Arg}{\PP'}$, with $\dtree{\Arg}{\PP'}\in\accTrees{\PP'}$, then  $\lambda'$ is not attacking in $\dtree{\Arg}{\PP'}$.
\end{Definition}

\ojo{This definition takes a line alteration $\lambda'$ and provides the necessary conditions for it to be effective. Checking that $\lambda'$ is not an attacking line in the resulting \delp\ depends on $\lambda'$ being an argumentation line in said program. This ensures that $\lambda'$ is exhaustive, \ie that it is not totally included in another argumentation line. Whenever this happens, however, an alteration can still be effective, and ATC would take care of the argumentation line including $\lambda'$.}

An interesting question remains: given an attacking line, which is the right position in it to perform an effective alteration? Observe that every attacking line ends with a con argument. This implies that an argumentation line ending with a pro argument could never be an attacking line. However, not every line ending with a con argument is an attacking line. Finally, the removal of a con argument in an attacking line turns it into non-attacking, and the removal of a pro argument in such a line yields an upper segment that is an attacking line. That is, the only way to effectively alter an attacking line is by the removal of a con argument; removing a pro would potentially augment the threat to the root. These intuitions are formalized through the following propositions. Finally, Lemma~\ref{lemma.eff.alt} shows that any alteration necessarily needs to be applied over a con argument in order to be effective.

\begin{Proposition}\label{prop.attacking.length.even}
If $\lambda$ is an attacking line in $\dtree{\Arg}{\PP}\in\allTrees{\PP}$ then $\lambda$ has even length.
\end{Proposition}
\begin{proof}
By \textit{reductio ad absurdum}, if we assume an attacking line $\lambda$ to be of odd length, its leaf argument is a pro argument. From Prop.~\ref{prop:U-rep}.3, we know that the leaf is marked as \Unode. From Def.~\ref{def:attackingLine}, we know that an attacking line has its pro arguments marked as \Dnode. This means that $\lambda$ is not attacking, contrary to the hypothesis.
\end{proof}

\begin{Corollary}\label{coro.ends.in.pro}
A line ending with a pro argument can never be an attacking line.
\end{Corollary}

\ojo{For the following corollary, observe that the upper segment of a con argument in any line renders a line whose leaf is a pro argument. Thus, according to Corollary~\ref{coro.ends.in.pro} it cannot be attacking.}

\begin{Corollary}\label{corollary.upper.con}
\ojo{The upper segment of a con argument in any line is a non-attacking line. That is, for any $\lambda\in\dtree{\Arg}{\PP}$ and any $\Brg\in\lambda^-$, the upper segment $\upsegm{\lambda}{\Brg}$ is non-attacking.}
\end{Corollary}

\begin{Proposition}\label{prop.upper.pro}
The upper segment of a pro argument in an attacking line is also an attacking line. That is, if $\lambda\in\dtree{\Arg}{\PP}$ is attacking then for any $\Brg\in\lambda^+$, the upper segment $\upsegm{\lambda}{\Brg}$ keeps being attacking.
\end{Proposition}
\begin{proof}
Line $\lambda$ being attacking means that its marking sequence corresponds to the regular expression $(\Dnode\Unode)^+$ (see Def.~\ref{def.d-rep.alternating.lines} and Prop.~\ref{prop.mark.non-warranting.argline}). When cutting the line on a pro argument, the resulting upper segment $\upsegm{\lambda}{\Brg}$ still conforms to the regular expression $(\Dnode\Unode)^+$, since removing a \Dnode\ placed below a \Unode\ does not change the latter \Unode\ mark. 
(Note that since $\upsegm{\lambda}{\Brg}$ ends in a con argument, from Prop.~\ref{prop:U-rep}.3, the leaf is marked as \Unode.) Hence, $\upsegm{\lambda}{\Brg}$ keeps being attacking.
\end{proof}



\begin{Lemma}\label{lemma.eff.alt}
\ojo{Given a non-warranting tree $\dtree{\Arg}{\PP}\in\allTrees{\PP}$, a line $\lambda\in\dtree{\Arg}{\PP}$, and an argument $\Brg\in\lambda$; if $\lambda$ is altered on \Brg\ and $\Brg\in\lambda^-$ then $\lambda$ is effectively altered.}
\end{Lemma}
\begin{proof}
%
%
Assuming $\lambda$ is altered on \Brg\ and $\Brg\in\lambda^-$, we have to prove that $\lambda$ is effectively altered on \Brg. From Def.~\ref{def.line.alt}, we know that $\upsegm{\lambda}{\Brg}$ is the resulting altered line, and since we know that \Brg is a con argument in $\lambda$, we also know that $\upsegm{\lambda}{\Brg}$ cannot be attacking (see Corollary~\ref{corollary.upper.con}). Finally, since the conditions on Def.~\ref{def.eff.alt} are fulfilled, we know $\lambda$ is effectively altered on \Brg.
\end{proof}

\begin{Proposition}\label{prop.D-rep.odd.length}
\ojo{Given a non-warranting tree, if there is a D-rep line $\lambda$, then either $\lambda$'s length is odd or there is some D-rep $\lambda'$ such that $\lambda'$ is adjacent to $\lambda$ and $\lambda'$'s length is odd.}
\end{Proposition}
\begin{proof}
Since $\lambda$ is a D-rep we know it is adjacent to an attacking line $\lambda_a$ (see Lemma.~\ref{lemma.D-rep.alternating}). We also know $\lambda$ has some D-rep sequence (repeated sequence of \Dnode\ nodes) (see Def.~\ref{def.d-rep.alternating.lines} and Def.~\ref{def.d-rep.sequence}). Observe that the first \Dnode\ from such a sequence is the adjacency point, say \Brg, between $\lambda$ and $\lambda_a$ (see Def.~\ref{def:adjacent} and Lemma.~\ref{lemma.D-rep.alternating}). Since $\lambda_a$ is attacking we know that this node (\Brg) corresponds to a pro argument (see Def.~\ref{def:attackingLine}). Therefore, we know that the second \Dnode\ in the D-rep sequence in $\lambda$ corresponds to a con argument, say \Crg, and since the leaf of a line is always marked as \Unode\ (see Prop.~\ref{prop:U-rep}.3) we necessarily have that \Crg is not $\lambda$'s leaf. 

Afterwards, we know there is an extra argument defeating \Crg which is pro. 
If such an argument is $\lambda$'s leaf, then we have that $\lambda$ ends in a pro argument and therefore it is easy to see that $\lambda$'s length is odd. 

On the other hand, if it is not $\lambda$'s leaf, we necessarily have that either $\lambda$ eventually ends in a pro argument, being odd its length in such a case; or that there is some other adjacent D-rep $\lambda'$ whose adjacency point is \Crg and $\lambda'$'s length is odd. This is so, given that we know there is some \Unode\ node defeating \Crg (given that \Crg is marked as \Dnode), and since such a node corresponds to a pro argument, $\lambda'$ is necessarily an odd-length D-rep line. 
\end{proof}

\begin{Proposition}\label{prop.D-rep.upper.pro}
\ojo{The upper segment of a pro argument in a D-rep line is an attacking line if there is no other odd-length D-rep line adjacent to it. That is, if $\lambda\in\dtree{\Arg}{\PP}$ is a D-rep line then for any $\Brg\in\lambda^+$, the upper segment $\upsegm{\lambda}{\Brg}$ is an attacking line if it holds that for any $\lambda'\in\dtree{\Arg}{\PP}$, if $\lambda'$ is adjacent to $\upsegmeq{}{\Brg}$ then $\lambda'$ is not an odd-length D-rep.}
\end{Proposition}
\begin{proof}
Given a D-rep line $\lambda\in\dtree{\Arg}{\PP}$ and some $\Brg\in\lambda^+$, we will assume there is no other odd-length D-rep $\lambda'\in\dtree{\Arg}{\PP}$, such that $\lambda'$ is adjacent to $\upsegmeq{}{\Brg}$. By \textit{reductio ad absurdum}, we will assume that altering $\lambda$ over the pro argument \Brg ends in a non-attacking uppersegment $\upsegm{\lambda}{\Brg}$.

Let us assume argument \Crg to be the adjacency point between $\lambda$ and the adjacent attacking line, say $\lambda_a$ (in accordance to Lemma~\ref{lemma.D-rep.alternating}). From Lemma~\ref{lemma.adj.point.pro} we know \Crg is a pro argument, \ie $\Crg\in\lambda^+$ and $\Crg\in\lambda_a^+$. 
The following alternatives arise:

1) If $\Brg=\Crg$ or $\Brg\in\upsegm{\lambda}{\Crg}$ then the alteration affects not only to $\lambda$ but also to $\lambda_a$ (since in such a case, it holds $\Brg\in\lambda_a$). Afterwards, since \Brg is a pro argument and $\Brg\in\lambda_a$, from Prop.~\ref{prop.upper.pro} we know that $\upsegm{\lambda}{\Brg}$ ends up being attacking, reaching the absurdity.

2) On the other hand, if $\Crg\in\upsegm{\lambda}{\Brg}$ then the alteration affects only to $\lambda$ (since $\Brg\notin\lambda_a$). Afterwards, since \Crg is defeated in the adjacent attacking line $\lambda_a$ by an undefeated argument, the alteration of $\lambda$ over \Brg will not change $\lambda_a$'s attacking status. This means that the dialectical tree will keep being non-warranting, and therefore, the alteration of $\lambda$ will render an uppersegment which will be either D-rep or attacking. For the former case, we will assume $\upsegm{\lambda}{\Brg}$ as D-rep. By hypothesis, we also know that there is no odd-length D-rep adjacent to $\upsegmeq{}{\Brg}$, and afterwards, from Prop.~\ref{prop.D-rep.odd.length} we know $\upsegm{\lambda}{\Brg}$ necessarily ends being of odd length. Finally, we reach the absurdity given that \Brg is a pro argument and therefore, we know that $\upsegm{\lambda}{\Brg}$'s length is even.

Hence, the only option for $\upsegm{\lambda}{\Brg}$ is to be attacking.
\end{proof}

\ojo{The need to avoid alterating lines over pro arguments is highlighted through the following remark, which appears from the results shown by Prop.~\ref{prop.D-rep.upper.pro} and Prop.~\ref{prop.upper.pro}.}

\begin{Remark}
\ojo{The alteration of a non-warranting line over a pro argument can render an attacking upper segment.}
\end{Remark}

It is clear that any effective alteration of an attacking line needs to be performed over a \emph{con argument} in it. However, not necessarily every attacking line in a tree has to be altered in order to obtain a tree free of attacking lines. This situation is illustrated next and formalized afterwards by Lemma~\ref{lemma.adj.att}.

\begin{Example}\label{ex.adjPointU}\ 
\begin{window}[3,r,{\mbox{\epsfig{file=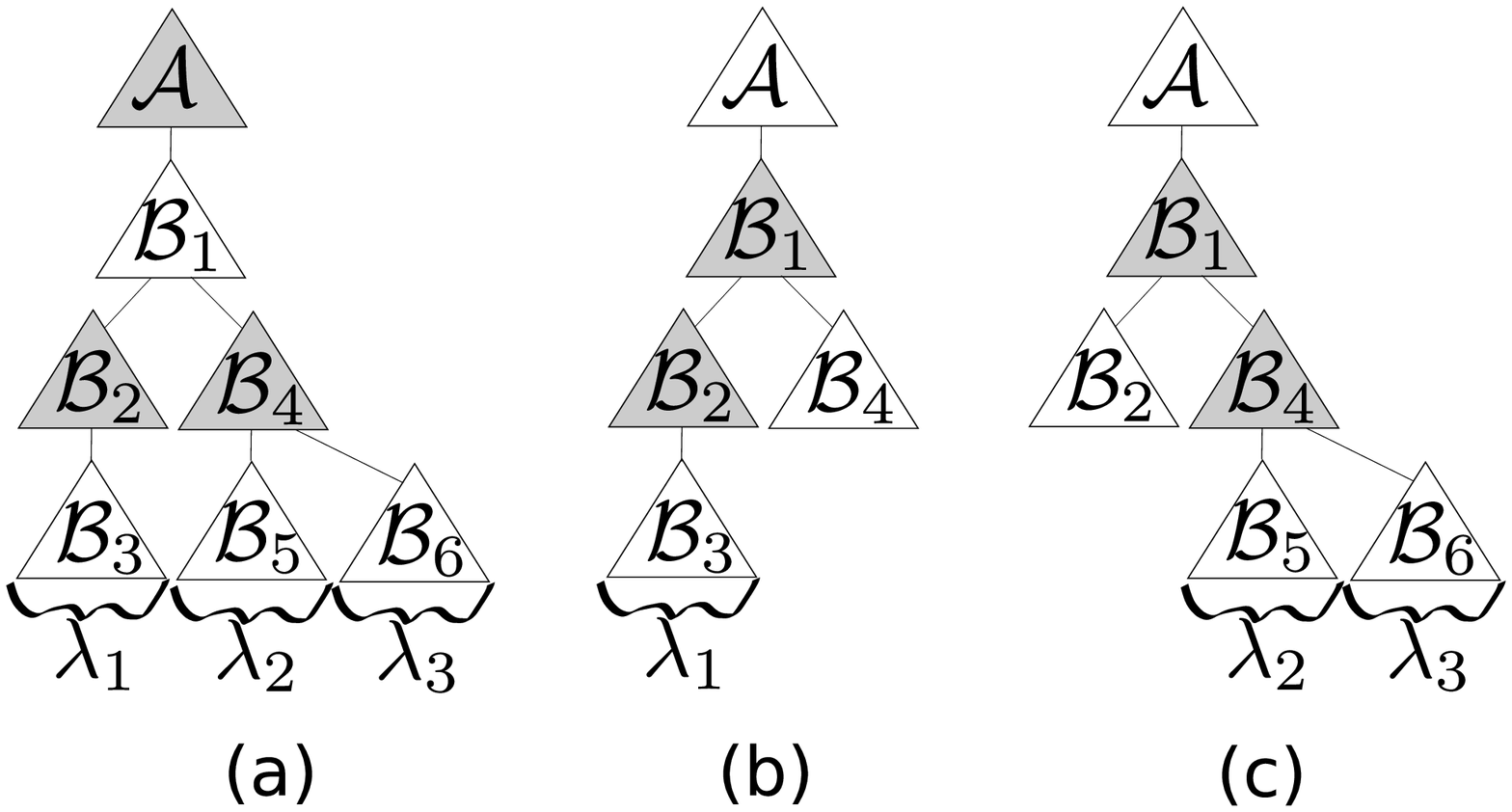,scale=.32}}},{}]
Consider the non-warranting tree depicted below in fig.~(a). The three argumentation lines $\lambda_1$, $\lambda_2$, and $\lambda_3$ are attacking. Note that $\lambda_1$ and $\lambda_2$ are adjacent at $\Brg_1$ which is marked as \Unode, and the same situation occurs regarding $\lambda_1$ and $\lambda_3$. However, $\lambda_2$ and $\lambda_3$ are adjacent at $\Brg_4$ which is placed below $\Brg_1$. To warrant the root argument, we need to alter every attacking line from the tree. Nonetheless, two different alternatives arise, either to alter both $\lambda_2$ and $\lambda_3$ and turning in consequence $\lambda_1$ to non-attacking (see fig.~(b) on the right), or to alter only $\lambda_1$ which ends up turning both $\lambda_2$ and $\lambda_3$ to non-attacking (see fig.~(c)). Note that in both cases the resulting dialectical trees end up as warranting.
\end{window}
\end{Example}

%

This example shows a particular configuration in which we can reduce the number of attacking lines to be altered: attacking lines that are adjacent at an argument marked as \Unode. This requirement comes from the following analysis: if the adjacency point is marked as $U$, the next argument in each of the (attacking) lines is necessarily marked as $D$, and an effective alteration on any of these lines would turn it into a $U$ argument, consequently changing the marking of the adjacency point.

\begin{Lemma}\label{lemma.adj.att}
Given $\dtree{\Arg}{\PP}\in\allTrees{\PP}$ and two attacking lines $\lambda_1\in\dtree{\Arg}{\PP}$ and $\lambda_2\in\dtree{\Arg}{\PP}$ adjacent at an argument \Brg marked as \Unode. Effectively altering $\lambda_1$ and every attacking line adjacent to it at any argument below \Brg, turns $\lambda_2$ to non-attacking.
\end{Lemma}
\begin{proof}
For the particular case in which \Brg is $\lambda_1$'s leaf this proof results trivial since $\lambda_1=\lambda_2$. Consequently, let us assume $\lambda_1=[\Arg,\ldots,\Brg,\Crg,\ldots]$. Since $\lambda_1$ is attacking and \Brg is marked as \Unode, we know \Crg is marked as \Dnode. To achieve the effective alteration of $\lambda_1$ and that of every attacking line adjacent to it at any argument below \Brg, two options arise: either changes are done (1) below \Crg or (2) above \Crg. \ojo{Note that \Crg should not be deactivated (removed) since it is a pro argument and an effective alteration is only ensured when applied over a con argument (see Lemma~\ref{lemma.eff.alt}), and moreover, since $\lambda_1$ is attacking, from Prop.~\ref{prop.upper.pro} we know that altering it over \Crg will result in a still-attacking uppersegment.} For (2) it is clear that not only $\lambda_1$ results effectively altered but also $\lambda_2$ since both share the common upper segment $[\Arg,\ldots,\Brg]$. On the other hand, for (1), the effective alteration of $\lambda_1$ and every attacking line adjacent to it at any argument below \Brg provokes \Crg to be marked as \Unode. This turns \Brg's mark to \Dnode\ independently from the existence of $\lambda_2$. Hence, $\lambda_2$ results unaltered in form but its marking sequence turns to D-rep, \ie non-attacking. 
\end{proof}

Now we know that for certain sets of attacking lines, just some of them have to be altered in order to render every line in the set into non-attacking. Consider a situation similar to the one expressed through Lemma~\ref{lemma.adj.att}, but with $\lambda_1$ and $\lambda_2$ adjacent at an argument \Brg marked as \Dnode. In this case, the effective alteration of $\lambda_1$ (along with the attacking lines adjacent to it below \Brg) would not turn $\lambda_2$ to non-attacking, since \Brg's mark \Dnode\ can only turn to \Unode\ by changing the mark of all of its children to \Dnode. Hence, $\lambda_2$ would need to be altered as well. 

\ojo{
Consequently, we should identify a subset of attacking lines, named \textit{attacking set}, in a tree such that altering only the lines it contains (no more, no less) would yield a tree free of attacking lines. For instance, in order to turn to warranting the dialectical tree depicted in figure (a) from Ex.~\ref{ex.adjPointU}, two possible subsets of attacking lines appear: $X_1=\{\lambda_1\}$ and $X_2=\{\lambda_2,\lambda_3\}$. Ideally, we would prefer the attacking set of smallest cardinality. However, this restriction is left as a minimal change criterion. That is, it is not ensured that altering only $\lambda_1$ instead of both $\lambda_2$ and $\lambda_3$, provokes less change in the \delp\ Thus, to introduce the formal definition of attacking set, we will rely upon an \textit{alteration criterion} identified as $\lessAlter$ to recognize among the subsets of lines from $\bundle$ --the bundle set determining $\dtree{\Arg}{\PP}$ (see Def.~\ref{def.bundle})-- that set whose effective alteration of every line contained in it would provoke as less change as possible in \PP with the objective to turn $\dtree{\Arg}{\PP}$ into warranting. 
}

\begin{Definition}[\ojo{Alteration Criterion}]\label{def.alteration.criterion}
Given a \delp\ \PP and an argument $\Arg\in\ARGS$, the \textbf{alteration criterion} $\lessAlter\subseteq(\Lines{\PP}\times\Lines{\PP})$ over the dialectical tree $\dtree{\Arg}{\PP}\in\allTrees{\PP}$, is the set of pairs $(X_1,X_2)$ stating that the effective alteration of every line in $X_1\subseteq\bundle$ is assumed to provoke less change than the effective alteration of every line in $X_2\subseteq\bundle$. The infix notation $X_1\lessAlter X_2$ will be used to refer to $(X_1,X_2)\in\lessAlter$.
\end{Definition}

\ojo{
Note that the alteration criterion defined above can be concretized by pursuing minimality according to set cardinality, such that for any pair of sets $X_1$ and $X_2$ of attacking lines from $\dtree{\Arg}{\PP}$, $X_1\lessAlter X_2$ holds \ifff $|X_1|< |X_2|$ holds.\label{alteration.criterion.cardinality} Nevertheless, we keep this criterion abstract in order to render a theory without unnecessary restrictions. This decision benefits the pursuit of a model of change that allows to guarantee different sorts of minimal change\footnote{Different perspectives on minimal change will be discussed later, in Section~\ref{sec.principles}.}. The notion of attacking set relying upon the alteration criterion, as is formalized next in Def.~\ref{def.alines.delp}, favors this objective by choosing minimal sets of lines to be altered from a dialectical tree (see Theorem~\ref{theorem.attset.minimal}) in order to render an altered dialectical tree which ends up warranting its root argument (see Theorem~\ref{theorem.attset.warrants}).
}

\begin{Definition}[Attacking Set]\label{def.alines.delp}
Given a \delp\ \PP, and the dialectical tree $\dtree{\Arg}{\PP}\in\allTrees{\PP}$; \ojo{the \textbf{attacking set} \ALINESPP is the set of lines} satisfying:
\vspace{-6mm}
\begin{description}
\item\indent\begin{enumerate}
\item $\ALINESPP \subseteq \{\lambda\in\dtree{\Arg}{\PP}\ |\ \lambda$ is an attacking line in $\dtree{\Arg}{\PP}\}$;
\item there is no pair of lines $\lambda$ and $\lambda'$ in \ALINESPP such that $\lambda$ is  adjacent to $\lambda'$ at an argument marked as \Unode;
\item there is no set $X$ satisfying $(1)$ and $(2)$ such that $\ALINESPP\subsetneq X$ or $X\lessAlter\ALINESPP$.
\end{enumerate}
\end{description}
\end{Definition}

\ojo{
The recognition of the attacking set \ALINESPP from a given dialectical tree \dtree{\Arg}{\PP} involves the verification of three conditions: firstly, \ALINESPP is a subset of the set of attacking lines from \dtree{\Arg}{\PP}; secondly, as stated in Lemma~\ref{lemma.adj.att}, if several adjacent attacking lines with adjacency point \Unode appear, to provoke the effective alteration of every one of them it is sufficient to alter only one of them; for the twofolded third condition we have that, the attacking set is maximal in the sense that every line included in it must be effectively altered in order to render a warranting tree (see also Theorem~\ref{theorem.attset.warrants} and Theorem~\ref{theorem.attset.minimal}) and minimal in the sense that no attacking line excluded from it needs to be explicitly altered (since it will be effectively altered as a result of the alteration of the lines in \ALINESPP) (see also Theorem~\ref{theorem.attset.warrants}). Finally, the attacking set will be that which provokes as less change as possible according to the alteration criterion \lessAlter (see Def.~\ref{def.alteration.criterion}).
}

\ojo{Clearly, the empty attacking set implies a warrating dialectical tree.}

\begin{Proposition}
If $\ALINESPP=\emptyset$ then $\dtree{\Arg}{\PP}$ is warranting.
\end{Proposition}
\begin{proof}
Straightforward from Def.~\ref{def.alines.delp} and Corollary~\ref{corollary.warrant}.
\end{proof}

\begin{Theorem}\label{theorem.attset.warrants}
Given the non-warranting tree $\dtree{\Arg}{\PP}\in\allTrees{\PP}$ built from a set $X\subseteq\Lines{\PP}$ of argumentation lines in the context of a \delp\ \PP; the dialectical tree resulting from the effective alteration of every line in $\ALINESPP$ ends up warranting.
\end{Theorem}
\begin{proof}
\ojo{Let us assume the existence of a \delp\ $\PP'\subseteq\PP$ which is obtained from $\PP$ by effectively altering every line $\lambda\in\ALINESPP$ according to Def.~\ref{def.eff.alt}.} From Corollary~\ref{corollary.warrant}, we know that any dialectical tree free of attacking lines is warranting. From Def.~\ref{def.alines.delp}, we know \ALINESPP contains all the attacking lines in \dtree{\Arg}{\PP} excepting the ones that have some attacking line within \ALINESPP adjacent at an argument marked as \Unode. Hence, we need to show that every attacking line $\lambda'\in\dtree{\Arg}{\PP}$ such that $\lambda'\not\in\ALINESPP$ is not an attacking line in $\dtree{\Arg}{\PP'}\in\allTrees{\PP'}$, which follows from Lemma~\ref{lemma.adj.att}.
\end{proof}


\begin{Theorem}\label{theorem.attset.minimal}
%
Given the non-warranting tree $\dtree{\Arg}{\PP}\in\allTrees{\PP}$ built from a subset of \Lines{\PP} of argumentation lines in the context of a \delp\ \PP; the dialectical tree resulting from the effective alteration of every line in any proper subset of \ALINESPP ends up non-warranting.
\end{Theorem}
\begin{proof}
From Def.~\ref{def.alines.delp}, every line in \ALINESPP is attacking, hence if a line $\lambda\in\ALINESPP$ is left unaltered, the resulting dialectical tree will contain an attacking line, and from Corollary~\ref{corollary.warrant}, it will be non-warranting.
\end{proof}

\section{Argument Theory Change}
\label{sec.atc}

We apply to \DLP\ the \textit{deactivating revision operator}~\cite{comma08}, that is part of ATC. In the \textit{dynamic abstract argumentation framework}~\cite{daf.comma10} only active arguments are considered by the argumentative reasoner. Thus by deactivation of an argument we refer to the reasoner no longer considering that argument. 
In its abstract form, the ATC argument revision operator revises an argumentation theory by an argument seeking for its warrant. In this article, we propose a concrete approach, from the abstract logic for arguments in past papers to the logics used for \delp s. These programs constitute the KBs from where the argumentation framework is built. The deactivating ATC approach reified to \DLP\ was preliminary introduced in~\cite{aaai08} and is extended here. 

For specifying the ATC argument revision upon \DLP, we firstly describe how to expand a \delp\ by an argument \Arg, to afterwards modify it (if necessary) by analyzing the dialectical tree rooted in \Arg, aiming at warranting \Arg. To this end, we follow the abstract deactivating approach to ATC which identifies the arguments to be \textit{deactivated} from the tree. \textit{Deactivation of arguments} in \DLP\ involves removal of rules from the worked \delp\ Conversely, the modification of the program aims at altering the tree, turning it to warranting. 
Doing this provokes change, not only regarding the \delp\ and the dialectical tree rooted in \Arg, but also regarding the set of warrants: some arguments could now be warranted, while some others could consequently lose such condition. Alternatives to control change according to different standpoints will be discussed in Section~\ref{sec.principles}.

Besides removing rules (arguments deactivation), alteration of trees could be performed through addition of rules (argument activation), in order to generate new arguments to be incorporated to the tree, in order to turn it to warranting. This approach, referred to as \textit{activating}, falls beyond the scope of this article and was treated in the context of an abstract argumentation framework in~\cite{atc.act.comma10}. Its reification to concrete logics like \DLP's is part of future work.

The main idea towards the alteration of the tree is to effectively alter each (attacking) line from the attacking set. This will be done through two main mechanisms: a \textit{selection function} which maps the appropriate argument to be deactivated from each line, and an \textit{incision function} which maps the appropriate defeasible rules inside the argument to be deleted from the \delp\ To decide which con argument is selected in a given line, we will assume a \textit{selection criterion} through which the set of con arguments --from each argumentation line-- could be ordered. A similar situation occurs among the defeasible rules inside arguments which will be addressed through a \textit{rule-based criterion}.

Deactivating a con argument from an attacking line $\lambda\in\dtree{\Arg}{\PP}$ always ends up yielding a non-attacking upper segment (see Corollary~\ref{corollary.upper.con}) and thus, the line ends up effectively altered (see Lemma~\ref{lemma.eff.alt}). However a major drawback appears: since deactivating an argument \Brg means the deletion of some defeasible rules from the \delp, other arguments \ojo{containing some of these rules would also disappear.} Particularly, a line $\lambda'\in\dtree{\Arg}{\PP}$ containing some of those disappearing arguments will be collaterally altered. This is referred to as a \textit{collateral incision} provoked by the original \textit{incision} over \Brg. Observe that a collaterally incised non-attacking line might be turned to attacking. Hence, the revision process should consider to alter not only attacking lines, but also other lines that may turn to attacking from a collateral incision. \ojo{The \textit{alteration set} is identified as the set of every line in \dtree{\Arg}{\PP} to be altered by the revision process.} Note that this new set would contain the attacking lines contained in the attacking set, while possibly adding more lines. The general outline of the revision process is given in Schema~\ref{schema}.

\floatname{algorithm}{Schema}

\begin{algorithm}
\caption{The Revision Process}\label{schema}
\begin{algorithmic}[1]
\REQUIRE A \delp\ $\PP=\SD$ and an argument \Arg
\ENSURE A revised \delp\ $\PP*\Arg$
\STATE Expand the program \PP\ by \Arg to obtain the program $\PP'=(\SSet,\DD\cup\Arg)$
\STATE Obtain the dialectical tree $\dtree{\Arg}{\PP'}\in\allTrees{\PP'}$
\STATE Define a \textit{selection function} $\deactsel: \Lines{\scriptsize\PP'} \rightarrow \ARGSP{\PP'}$ mapping each $\lambda\in\dtree{\Arg}{\PP'}$ to the con argument $\Brg\in\lambda^-$ whose deactivation would provoke less change according to a \textit{selection criterion} ``\selcrit''.
\STATE Define an \textit{incision function} $\incise: \ARGSP{\PP'} \rightarrow \powerSet{\Ld}$ to map the selected argument $\deactsel(\lambda)$ from each line $\lambda$ to some defeasible rules inside $\deactsel(\lambda)$ according to a \textit{rule-based ordering criterion} ``\lesschg''.
\STATE Define the \textit{alteration set} $\Lambda$ containing every line from the attacking set \ALINESPPprime along with those lines from $\dtree{\Arg}{\PP'}$ that would be turned into attacking by a \textit{collateral incision}.
\STATE $\PP*\Arg= (\SSet,\DD')$, where $\DD'= (\DD\cup\Arg)\setminus\bigcup_{\scriptsize\lambda\in\Lambda}\incise(\deactsel(\lambda))$
\end{algorithmic}
\end{algorithm}

\floatname{algorithm}{Algorithm}
\setcounter{algorithm}{0}



Regarding \textit{minimal change}, it is natural to look for changing a program by deleting \ojo{as few rules as possible} from it. However, each defeasible rule that is deleted has a direct impact in the resulting dialectical tree analyzed to give warrant to the new argument. Consequently, we can identify three different \textit{axes of change}: 

\begin{enumerate}\label{axes.of.change}
\item how to decide among the con arguments in each line to be altered, which will be controlled by the \textit{selection criterion};
\item how to decide among the defeasible rules inside each selected argument to be deactivated, that will be controlled through the \textit{rule-based criterion};
\item how to deal with the problem of \textit{collateral incisions}. 
\end{enumerate}

We will study in detail the first and third axes of change, whereas the second axis is abstracted away by assuming it to be given in advance. Regarding the latter axis, an unanswered question is left to be addressed throughout this section: is it necessarily mandatory to avoid collateral incisions, or is it possible to take advantage of a collateral incision to alter several lines at once? Moreover, for cases in which the latter question is true, how such pursuit would affect the first two axes? 
The reader should keep in mind that the need to manage different criteria of change is related to the inherently complex nature of the problem; furthermore, each criteria is meant to interact with the others in order to achieve an appropriate solution. 
Note that this solution would involve a compromise, since the main challenge in this model of change is to achieve a balance among the three axes of change.
In consequence, sometimes it will be necessary to update some of the initial orderings towards this balance. This discussion will be attended later in this section.




\subsection{Basic Elements of the Change Machinery}\label{sec.atc.basic.elements}

The argument revision operation will be performed over a \delp\ $\PP = \pair{\Pi}{\Delta}$ by a new argument \Aalpha. This argument will end up being warranted from the program resulting from the revision. However, \Aalpha\ is required to constitute a proper argument structure after the addition of \Arg to \PP. Thus, $\Arg\cup\Pi$ should have a defeasible derivation for $\alpha$. Since the set $\Pi$ of strict rules and facts represents (in a way) the current state of the world (which is indisputable), it is clear that \Arg does not stand by itself, but in conjunction with $\Pi$. Argument \Aalpha\ could be brought up, for instance, by an agent sensing the environment. In such a case, \Aalpha\ is going to be called \emph{\PP-external}, since it may contain defeasible rules external to the program \PP. However, if we consider the external information in \Arg along with the set $\Delta$ of defeasible rules in \PP, then \Arg should be an argument compliant with Def.~\ref{def.argument}.

\begin{Definition}[External Argument Structure] 
Let $\PP = \SD$ be a \delp, \Aalpha\ is a \textbf{\PP-external argument structure} (or simply, \textbf{external argument}) for a literal $\alpha$ from \PP\ \ifff $\Arg\not\subseteq\DD$ and \Aalpha\ is an argument structure from $(\SSet,\DD\cup\Arg)$. The domain of \PP-external arguments is identified through the set \EXT.
\end{Definition}

Once the external argument \Ah\ is added, the dialectical tree rooted in it has to be built in order to check its warrant status. The change machinery alters this tree only whenever it does not warrant \Ah. Therefore, such tree rooted in \Arg is referred to as \emph{temporary dialectical tree},\label{temporaryTree} since it will (in general) be an intermediate state during the revision process. 

Although it would be interesting to provoke \Ah\ to end up warranted only when $\alpha$ is not already warranted from another argument, it is desirable to always achieve warrant for \Ah\ in order to support the new external information brought by it. Therefore, since warrant for $\alpha$ could be easily checked beforehand, the stress is put on the complications arising of ensuring \Arg\ to end up undefeated.

Given the temporary dialectical tree, for every line $\lambda$ in it, a con argument is selected over $\interf$ on behalf of the \textit{selection criterion} ``\selcrit'', by means of an \emph{argument selection function} $\deactsel$. This criterion codifies one of the axes of change, setting an ordering among the con arguments in a given line. Afterwards, we present an example illustrating a reasonable \emph{initial setting} of the set ``\selcrit''. Later in this section it will be clear why the proposed ordering is just ``initial'', and not definitive.

\begin{Definition}[Selection Criterion (Set)]
Given a \delp\ \PP and the dialectical tree $\dtree{\Arg}{\PP}\in\allTrees{\PP}$, for any line $\lambda\in\dtree{\Arg}{\PP}$, the \textbf{selection criterion} $\selcrit\subseteq(\ARGS\times\ARGS)$ is the set of pairs $(\Brg_1,\Brg_2)$ stating that the deactivation of $\Brg_1\in\lambda^-$ is assumed to provoke less change than that of $\Brg_2\in\lambda^-$. The infix notation $\Brg_1\selcrit\Brg_2$ will be used to refer to $(\Brg_1,\Brg_2)\in\selcrit$. 
\end{Definition}

\begin{Example}\label{ex.selection.main.issue}
Assuming the root argument \Arg as the main issue being disputed, it is natural to think that the lower we go in an argumentation line in the tree, the most we move away from the main issue. Therefore, when looking for an argument in a line to be incised, the lowest the argument we deactivate is, the least change we provoke to the program in relation to the main issue in dispute. This intuition is used to initialize the selection criterion (different postures are discussed in Sect.~\ref{sec.principles}):
\begin{center}
\textit{Initial setting: }
Given a line $\lambda\in\bundle$ where \bundle determines \dtree{\Arg}{\PP},\\ 
$\selcrit=\{(\Brg_1,\Brg_2)\ |\ \Brg_1\in\interf, \Brg_2\in\interf,$ and $\Brg_2\in\upsegm{\lambda}{\Brg_1}\}$
\end{center}
\end{Example}

The selection criterion will allow us to univocally determine which argument is the right one to be deactivated in order to effectively alter any argumentation line. Recall that the choice of selecting just con arguments comes from Lemma~\ref{lemma.eff.alt}. In addition, we assume the existence of a special kind of argument, referred to as \textit{escape argument} and noted as $\epsilon$, such that $\epsilon\in\ARGS$ for any \delp\ \PP. The escape argument is used for theoretical matters only. Finally, no argument from $\ARGS$ defeats $\epsilon$ and $\epsilon$ does not defeat any argument from $\ARGS$.

\begin{Definition}[Argument Selection Function ``$\deactsel$'']\label{def:selection3}
Given a \delp\ \PP, the \textbf{argument selection function} $\deactsel: \Lines{\scriptsize\PP} \rightarrow \ARGS$ is defined as:
\[\deactsel(\lambda) = \begin{cases}
\epsilon & \text{if } \lambda^- = \emptyset\\
\Brg\in\lambda^- & \text{otherwise}
\end{cases}
\]
%
%
\ojo{such that if $\selcrit\neq\emptyset$ then there is some $\Crg\in\lambda^-$ where $(\Brg,\Crg)\in\selcrit$ and for any other $\Crg'\in\lambda^-$ such that $(\Crg',\Brg)\in\selcrit$ it holds $(\Brg,\Crg')\in\selcrit$.}
\end{Definition}




After an argument was selected for incision, a decision should be made according to which portion of the argument is going to be cut off. Since arguments are formed by defeasible rules, we provide a mechanism to set a preference over sets of defeasible rules: the \textit{rule-based criterion}. This criterion addresses the second axis of change stated before. \ojo{Rules in \delp s can be ordered in a wide variety of ways, \eg dynamically through a lexicographic ordering method. This discussion exceeds the scope of the article.} Although incisions are going to rely on the rule-based criterion, we will abstract away from it and will assume an order is given beforehand. 

\begin{Definition}[Rule-Based Criterion]
Given a \delp\ $\SD$, the \textbf{rule-based criterion} $\lesschg \subseteq \powerSet{\Ld} \times \powerSet{\Ld}$ defines a total order between pairs of subsets of $\Delta$.
\end{Definition}


\begin{Definition}[Argument Incision Function ``$\incise$'']\label{def:incision3}
Given a \delp\ \PP, a function $\incise: \ARGS \rightarrow \powerSet{\Ld}$ is an \textbf{argument incision function} such that:
\[\incise(\Brg) = \begin{cases}
\emptyset & \text{if } \Brg = \epsilon\\
\ojo{\Gamma} & \ojo{\text{otherwise}}
\end{cases}
\]
\ojo{where $\emptyset\subset\Gamma\subseteq\Brg $ such that $\Gamma\lesschg\Gamma'$ holds for any $\Gamma'\subseteq\Brg$, where $\Gamma\neq\Gamma'$.}
\end{Definition}


An \emph{argument incision function} $\incise$ is applied to the selected argument, identifying a non-empty subset of defeasible rules to be cut off from the \delp\ Once the incision over a con argument is performed, the line it belonged to ends up effectively altered.

\begin{Lemma}\label{lemma.incision.effectiveAlteration}
Given a \delp\ \PP and $\dtree{\Arg}{\PP}\in\accTrees{\PP}$, if ``\incise'' is an incision function then for any $\lambda\in\dtree{\Arg}{\PP}$,  $\PP\setminus\incise(\deactsel(\lambda))$ determines an effective alteration of $\lambda$. 
\end{Lemma}
\begin{proof}
Considering $\deactsel(\lambda)$, from Def.~\ref{def:selection3}, two options arise: either $\lambda^-=\emptyset$, in which case $\deactsel(\lambda)=\epsilon$, or otherwise we know $\deactsel(\lambda)\in\lambda^-$. Considering $\incise(\deactsel(\lambda))$, from Def.~\ref{def:incision3}, again we have two options: either (1) $\deactsel(\lambda)=\epsilon$, in which case $\incise(\deactsel(\lambda))=\emptyset$, or otherwise, (2) $\incise(\deactsel(\lambda))=\Gamma$, where $\Gamma\subseteq\deactsel(\lambda)$. For (1), $\PP\setminus\incise(\deactsel(\lambda))=\PP$ determines a null alteration of $\lambda$ (see Def.~\ref{def.line.alt}). However, since $\lambda^-=\emptyset$, we know $\lambda$ contains only the root argument which means that $\lambda$ has odd length, and from Prop.~\ref{prop.attacking.length.even} (contrapositive), we know $\lambda$ is not attacking. Hence, the alteration of $\lambda$ is effective (see Def.~\ref{def.eff.alt}). On the other hand, for (2), $\PP\setminus\incise(\deactsel(\lambda))$ alters $\lambda$ rendering a non-attacking upper segment $\upsegm{\lambda}{\Brg}$ of $\lambda$. 
Hence, $\lambda$ is altered by deactivating one of its con arguments. From Lemma~\ref{lemma.eff.alt}, we know this ends up in an effective alteration of $\lambda$ which means that it turns to non-attacking (independently of $\lambda$ being previously attacking). 
Finally, $\PP\setminus\incise(\deactsel(\lambda))$ determines an effective alteration of $\lambda$ (Def.~\ref{def.eff.alt}).
\end{proof}

To deactivate an argument, we need to delete defeasible rules from the \delp\ at issue. These rules are mapped by the incision function applied over that argument. Moreover, the incised defeasible rules are considered to provoke the least possible change in concordance with the minimal change principle. 
Regretfully, sometimes incisions will affect more arguments than the one being incised. In order to identify this situation, we introduce the notion of \emph{collateral incision}.

\begin{Definition}[Collateral Incision]\label{def.colinc}
Given a \delp\ \PP, an incision over an argument $\Drg$ provokes a \textbf{collateral incision} over an argument $\Brg$ iff $\incise(\Drg)\cap\Brg \neq \emptyset$. For any $\lambda\in\Lines{\PP}$, if $\Brg\in\lambda$ and $\incise(\Drg)\cap\Crg = \emptyset$ for every $\Crg\in\lambda^{\uparrow}(\Brg)$, we say $\colinc{\incise(\Drg)}{\Brg}$ is the \textbf{uppermost collateral incision over $\lambda$} where $\colinc{\incise(\Drg)}{\Brg} = \incise(\Drg)\cap\Brg$.
\end{Definition}

When a collateral incision occurs over more than one argument in the same line, we will be interested in the uppermost collaterally incised argument, since its deactivation will make the arguments below it in the line to disappear from the resulting tree. Hence, non-uppermost collateral incisions will not be affecting the status of the root argument in the temporary tree.

From now on, we will make reference only to uppermost collateral incisions (though sometimes we will omit the word ``uppermost'') through the notation introduced in the above definition. Collateral incisions represent the main difficulty to overcome: the involuntary deactivation of pro arguments in non-attacking lines might turn them into attacking lines. In the case of the pro argument belonging to an attacking line, as analyzed before, its deactivation would not change the line's status (see Prop.~\ref{prop.upper.pro}). Moreover, although a collaterally incised con argument does not turn lines into attacking, it would also be a source of possibly unnecessary change. Therefore, it is desirable to select arguments in which there is a possibility of incision that never results in a collateral incision to other arguments. This is captured by the \emph{cautiousness} principle.

\begin{center}
\textbf{(Cautiousness)} $\deactsel(\lambda) \setminus \bigcup\{\Brg$ in $\dtree{\Arg}{\PP} \ | \ \Brg\neq\deactsel(\lambda)\} \neq \emptyset$, for every $\lambda \in \dtree{\Arg}{\PP}$
\end{center}

\begin{Definition}[Cautious Selections]\label{def:cautious}
A selection function $\deactsel$ is identified as \textbf{cautious} \ifff it satisfies \textbf{cautiousness}.
\end{Definition}

A cautious selection function $\deactsel$ ensures that there is some incision function $\incise$ such that for any $\lambda\in\dtree{\Arg}{\PP}$, it follows $\incise(\deactsel(\lambda))$ does not collaterally incise any other argument \Brg in the tree; namely $\colinc{\incise(\deactsel(\lambda))}{\Brg} = \emptyset$. Nonetheless, it is important to remark that the possibility for an incision function to cut rules avoiding collateral incisions will be highly dependent on the rule-based criterion. 
That is, an incision may apply for the best option given by the criterion, but this might not be the best one to avoid a collateral incision. In such cases, a possible alternative is to relax the order given by the rule-based criterion. This matter speaks about the relation between the second and third axes of change. The appropriate analysis of this subject is similar to the selection criterion's which motivates the inclusion of an update rule. This will be made clear later in this section.


\begin{Example}\label{ex:cautious}
From the tree of Ex.~\ref{ex:tree}, the only possible selection in the attacking line $[\Arg, \Barg_1]$ is $\Barg_1$, whereas for the attacking line $[\Arg, \Barg_2, \Barg_3, \Barg_4]$, the selection function could return either $\Barg_2$ or $\Barg_4$, depending on the selection criterion. Regarding the selection of $\Barg_4$, it satisfies cautiousness because it has no intersection with any other argument. In contrast, the selection of $\Barg_2$ would be non-cautious, since its two rules \drule{\no a}{y} and \drule{y}{p} belong to $\Barg_1$ and $\Barg_3$, respectively. Finally, considering $\Barg_1$ in the other attacking line, we have that $\Barg_1 \cap \Arg = \{\drule{x}{z}\}$ and $\Barg_1 \cap \Barg_2 = \{\drule{\no a}{y}\}$. However, the remaining portion of $\Barg_1$ is non-empty: $\Barg_1 \setminus \ojo{\bigcup \{\Arg,\Barg_2, \Barg_3, \Barg_4, \Barg_5, \Barg_6\}} = \{\drule{y}{x}\}$; hence, the selection of $\Barg_1$ satisfies cautiousness.
\end{Example}

Sometimes cautiousness may not be satisfied. In such a case, when a non-cautious selection is unavoidable, the incision inevitably provokes collateral incisions. Sometimes collateral incisions could be harmful: assume a line $\lambda$ is collaterally incised over $\Brg\in\lambda$, thus provoking the involuntary deactivation of $\Brg$. If $\Brg$ is placed above the selection in that line, \ie $\Brg\in\upsegm{\lambda}{\deactsel(\lambda)}$, then the regular (non-collateral) alteration of $\lambda$ --performed through the deactivation of $\deactsel(\lambda)$-- is left without effect. That is, the effective alteration of $\lambda$, leaving a non-attacking line $\lambda'$, cannot be completely trusted since collateral incisions might turn $\lambda'$ to attacking afterwards.

These situations need to be appropriately addressed to ensure the correctness of the revision operation. Hence, an additional condition should be provided in order to preserve every effective alteration from being collaterally altered. Next we introduce the \emph{preservation} principle, which avoids collateral incisions over any line to occur over arguments placed above the selected argument in that line. In addition, preservation also ensures no collateral incision to occur over the root argument. This is necessary to keep \Arg active which is paramount to pursue its warrant. 
It is important to note that the preservation principle is given as a logical formula to restrict the respective images of the selection and incision functions. Note that this principle does not intend to provide any specific procedure nor algorithm.

\begin{center}\label{preservation}
\textbf{(Preservation)}
If $\colinc{\incise(\deactsel(\lambda'))}{\Brg}\neq\emptyset$ then $\Brg\neq\Arg$ and $\deactsel(\lambda)\in\lambda^{\uparrow}[\Brg]$,\\
for every $\lambda\in\dtree{\Arg}{\PP}$, $\lambda'\in\dtree{\Arg}{\PP}$, and $\Brg\in\lambda$
\end{center}

\begin{window}[0,r,{\mbox{\epsfig{file=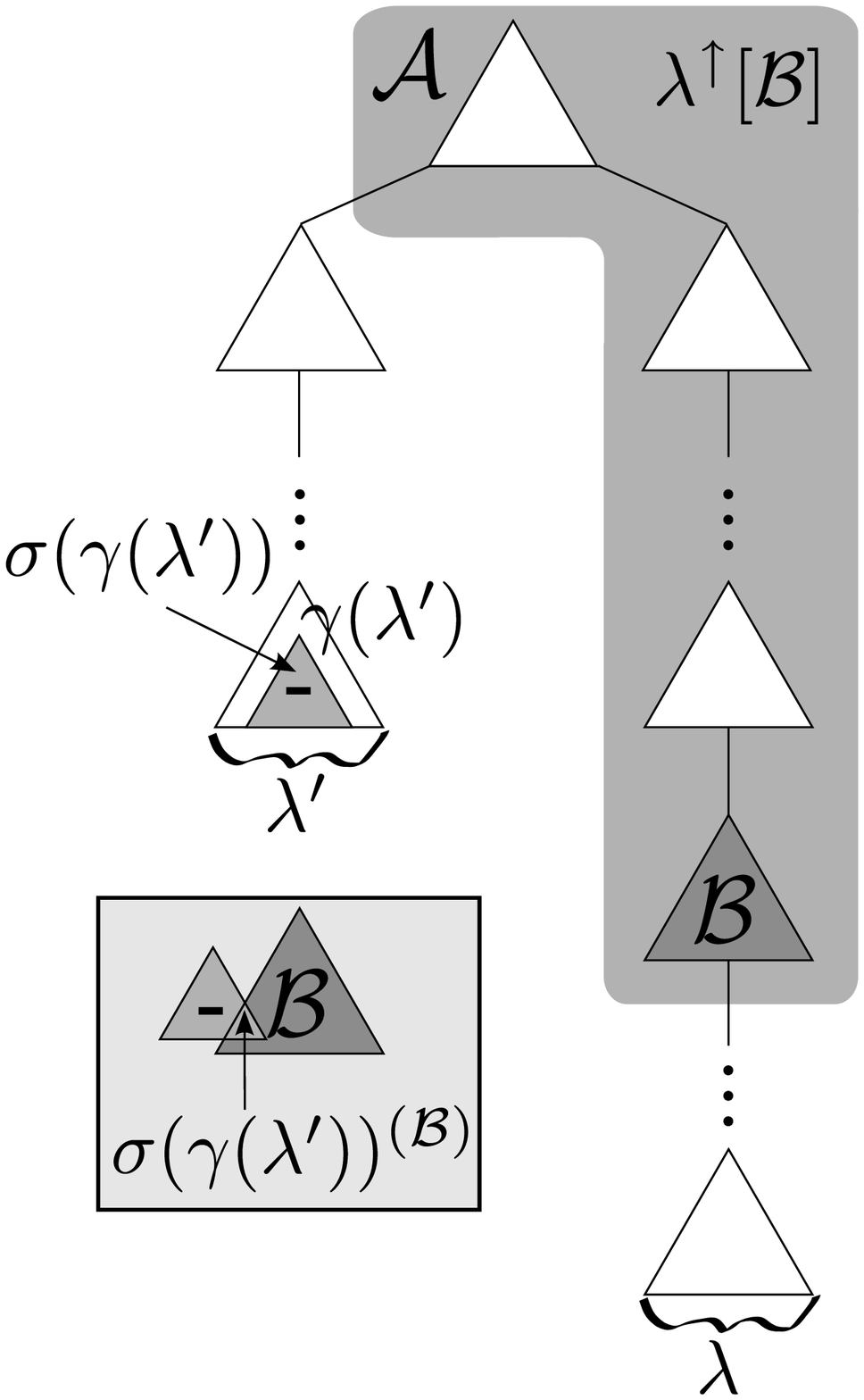,scale=.3}}},{}]
This principle is illustrated in the figure depicted on the right and exemplified in Ex.~\ref{ex.preservation}. When an incision
in line $\lambda'$ (the left branch in the figure) results in an uppermost collateral incision $\colinc{\incise(\deactsel(\lambda'))}{\Brg}$ over an argument \Brg in line $\lambda$ (right branch), it must be ensured that the selection $\deactsel(\lambda)$ in line $\lambda$ is performed on the upper segment of \Brg. 
Finally, note that if \Brg is the root node \Arg then the consequent of the preservation principle is false ($\Brg = \Arg$), which forces the antecedent to be false in order for the principle to hold. Hence, preservation requires $\colinc{\incise(\deactsel(\lambda'))}{\Arg}=\emptyset$ to be satisfied. 
We refer to this individual condition as \textit{root preservation}.
\end{window}

\begin{center}
	\textbf{(Root preservation)} $\colinc{\incise(\deactsel(\lambda))}{\Arg} = \emptyset$, for every $\lambda\in\dtree{\Arg}{\PP}$ 
\end{center}

%

\begin{Example}\label{ex.preservation}
Let us consider the dialectical tree on the right, upon which selections
\begin{window}[1,r,{\mbox{\epsfig{file=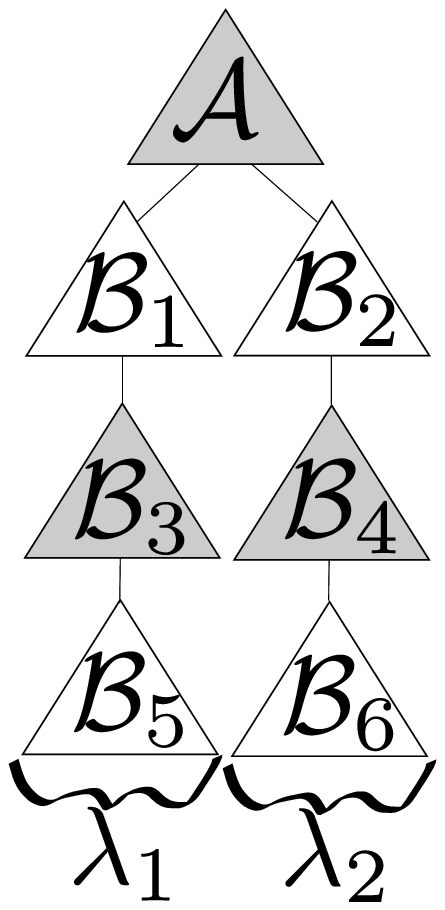,scale=.32}}},{}]
\noindent and incisions are applied. Assuming the selections $\deactsel(\lambda_1) = \Brg_5$ and $\deactsel(\lambda_2) = \Brg_6$, and that the incision over $\Brg_5$ provokes a collateral incision over $\Brg_4$, \ie $\colinc{\incise(\deactsel(\lambda_1))}{\Brg_4}\neq\emptyset$, we have that, in order for preservation to hold, the selection over $\lambda_2$ should be placed on or above $\Brg_4$, \ie $\deactsel(\lambda_2)\in\upsegmeq{2}{\Brg_4}$. The alternative would be to find a different selection function mapping to $\Brg_2$ from $\lambda_2$, and thus the alteration of $\lambda_2$ would finally result from the incision over $\Brg_2$. Note that, if the old selection were kept, the collateral incision over $\Brg_4$ would have yielded the attacking line $[\Arg,\Brg_2]$.
\end{window}
\end{Example}


\ojo{When preservation is satisfied given that every collateral incision $\colinc{\incise(\deactsel(\lambda'))}{\Brg}$ occurs solely over the selection $\Brg=\deactsel(\lambda')$ (and therefore $\lambda=\lambda'$), the principle holds for each line in the dialectical tree, since $\deactsel(\lambda)\in\upsegmeq{}{\Brg}$ is always satisfied. We refer to this particular case as \textit{strict preservation}.} 

\begin{center}
\textbf{(Strict preservation)} 
\ojo{If $\colinc{\incise(\deactsel(\lambda'))}{\Brg}\neq\emptyset$ then $\lambda=\lambda'$ and $\deactsel(\lambda)=\Brg$,\\
for every $\lambda\in\dtree{\Arg}{\PP}$, $\lambda'\in\dtree{\Arg}{\PP}$, and $\Brg\in\lambda$}
\end{center}

An incision function satisfying strict preservation ensures no argument is collaterally incised by any incision in the tree \ojo{excepting for the incision over the selected argument itself}. However, this principle is too restrictive and thus, it may sometimes be impossible to satisfy. Observe that there exists a relation between cautiousness and strict preservation: a selection function satisfying cautiousness ensures there exists some incision function free of collateral incisions, while an incision function satisfying strict preservation ensures that, with the actual configuration of incisions, no collateral incisions over an argument \ojo{different from the selected one} will occur. The following two propositions address such relation between cautiousness and strict preservation.

\begin{Proposition}\label{prop.cautious.to.strict}
Given a \delp\ \PP, a dialectical tree \dtree{\Arg}{\PP}, and a selection function ``$\deactsel$'', if cautiousness is satisfied then there is some rule-based criterion which leads to a strict-preserving incision function ``$\incise$''.
\end{Proposition}
\begin{proof}
If \deactsel is cautious, then for any line $\lambda\in\dtree{\Arg}{\PP}$ there exists a set of rules $X_{\lambda}\subseteq\deactsel(\lambda)$ that does not overlap with any other argument in $\dtree{\Arg}{\PP}$, \ie $X_{\lambda}\cap\Brg=\emptyset$, for every $\lambda\in \dtree{\Arg}{\PP}$, every $\Brg\in\lambda'$ and every $\lambda'\in \dtree{\Arg}{\PP}$, such that $\deactsel(\lambda)\neq\Brg$. Let us assume a rule-based criterion such that for any $\Gamma\subseteq\bigcup_{\scriptsize\lambda\in\dtree{\Arg}{\tiny\PP}}X_{\lambda}$ and any $\Gamma'\subseteq\DD$ satisfying $\Gamma\cap \Gamma'=\emptyset$, it holds $\Gamma\lesschg \Gamma'$. 
Such rule-based criterion leads to an incision function satisfying $\incise(\deactsel(\lambda))\subseteq X_{\lambda}$, and then the deactivation of each selected argument $\deactsel(\lambda)$, \ojo{would not provoke any other argument to be deactivated. That is, $\colinc{\incise(\deactsel(\lambda))}{\Brg}=\emptyset$, for every $\lambda\in\dtree{\Arg}{\PP}$ and any $\Brg\neq\deactsel(\lambda)$. Observe that this condition leads to the verification of strict preservation.} Therefore, there exists a rule-based criterion leading to a strict-preserving incision function.
\end{proof}

\begin{Proposition}\label{prop.strict.to.cautious}
Given a \delp\ \PP, a dialectical tree \dtree{\Arg}{\PP}, a selection function ``$\deactsel$'', and an incision function ``$\incise$'', if strict preservation is satisfied then cautiousness is also satisfied. 
\end{Proposition}
\begin{proof}
If \incise is strict-preserving then no incision provokes collateralities. That is, \ojo{if $\colinc{\incise(\deactsel(\lambda'))}{\Brg}\neq\emptyset$ then $\lambda=\lambda'$ and $\deactsel(\lambda)=\Brg$, for every $\lambda\in\dtree{\Arg}{\PP}$, $\lambda'\in\dtree{\Arg}{\PP}$, and $\Brg\in\lambda$. Thus,  $\colinc{\incise(\deactsel(\lambda'))}{\Brg} = \emptyset$ holds always that $\Brg\neq\deactsel(\lambda)$ is ensured.} This means that, in each $\lambda$, there is a subset of $\deactsel(\lambda)$ that does not belong to any argument \ojo{$\Brg\neq\deactsel(\lambda)$} in \dtree{\Arg}{\PP}. Therefore, \deactsel is cautious.
\end{proof}

Proposition~\ref{prop.cautious.to.strict} states that, when a selection is cautious, even though it overlaps with some argument, the incision over that selection might be performed outside this overlapping. In this case, there is no collateral incision and strict preservation holds. 
Finally, by Proposition~\ref{prop.strict.to.cautious}, it is clear that a strict-preserving incision function may be achieved only through cautious selections.


The three preservation principles are interrelated through the following proposition. Afterwards, given a dialectical tree, it is shown the utmost importance of the preservation principle regarding the controlled alteration of lines and any arising collateralities towards achieving a warranting condition for the root argument.

\begin{Proposition}\label{prop.preservation}
Given a \delp\ \PP, a dialectical tree $\dtree{\Arg}{\PP}\in\accTrees{\PP}$, and an argument incision function ``$\incise$'',
\vspace{-6mm}\begin{description}
\item\indent\begin{enumerate}
\item if preservation is satisfied then root preservation is also satisfied.
\item if strict preservation is satisfied then preservation is also satisfied.
\end{enumerate}
\end{description}
\end{Proposition}
\begin{proof}
1. Assume preservation is satisfied and consider $\Brg=\Arg$, if $\colinc{\incise(\deactsel(\lambda'))}{\Arg} \neq \emptyset$ then we know $\Arg\neq\Arg$, which is absurd. Hence, $\colinc{\incise(\deactsel(\lambda'))}{\Arg} = \emptyset$, and therefore, root-preservation holds.
%


\ojo{2. Assuming strict preservation is satisfied, we know that if $\colinc{\incise(\deactsel(\lambda'))}{\Brg}\neq\emptyset$ then $\lambda=\lambda'$ and $\deactsel(\lambda)=\Brg$, for every $\lambda\in\dtree{\Arg}{\PP}$, $\lambda'\in\dtree{\Arg}{\PP}$, and $\Brg\in\lambda$. In particular this means that $\Brg\neq\Arg$, given that from Def.~\ref{def:selection3}, a selection is necessarily a con argument, and thus it cannot be the root argument which is pro. Besides, since $\deactsel(\lambda)=\Brg$ it also holds $\deactsel(\lambda)\in\upsegmeq{}{\Brg}$. Finally, preservation is satisfied.}
\end{proof}

\begin{Definition}[Warranting Incision Function]\label{def.warrantingIncision}
An argument incision function ``$\incise$'' is said to be \textbf{warranting} iff it satisfies preservation.
\end{Definition}

The fact that an incision function satisfies the preservation principle ensures that it will handle collateral incisions in a proper manner. Then, any arising attacking line will be correspondingly effectively altered. \ojo{For the following definition recall that \bundle refers to the bundle set of argumentation lines in the dialectical tree rooted in \Arg from the \delp\ \PP (see Def.~\ref{def.bundle}).}


\begin{Lemma}\label{lemma.warrantingIncision.effectiveAlteration}
Given a \delp\ \PP and a dialectical tree $\dtree{\Arg}{\PP}\in\accTrees{\PP}$, if ``\incise'' is a warranting incision function then for any set $X\subseteq\bundle$, $\PP\setminus\bigcup_{\lambda\in X}\incise(\deactsel(\lambda))$ determines the effective alteration of each $\lambda\in X$.
\end{Lemma}
\begin{proof}
If $|X|=1$, let $X = \{\lambda\}$ be effectively altered from $\PP\setminus\bigcup_{\lambda\in X}\incise(\deactsel(\lambda))$ (see Lemma~\ref{lemma.incision.effectiveAlteration}). We need to show that this property also holds when $|X|>1$. Suppose now $X$ also contains a line $\lambda'\in X$ and assume its effective alteration provokes a collateral alteration over argument $\Brg \in\lambda$. By \textit{reductio ad absurdum}, assume the effective alteration of $\lambda$ is affected by the collateral alteration provoked by $\lambda'$, this means that the upper segment $\lambda^{\uparrow}(\deactsel(\lambda))$ turns to attacking from the collateral incision over \Brg. This implies that $\deactsel(\lambda)\notin\lambda^{\uparrow}[\Brg]$. However, since $\incise$ is known to be warranted, from Def.~\ref{def.warrantingIncision}, it satisfies preservation, hence since  $\colinc{\incise(\deactsel(\lambda'))}{\Brg}\neq\emptyset$, we know that $\deactsel(\lambda)\in\lambda^{\uparrow}[\Brg]$, which is absurd. Afterwards, no line $\lambda'\in X$ affects the effective alteration of any $\lambda\in X$. 
Observe that this also holds for any $X\subseteq\bundle$. Finally, $\PP\setminus\bigcup_{\lambda\in X}\incise(\deactsel(\lambda))$ determines the effective alteration of each $\lambda\in X$.
\end{proof}

\ojo{Sometimes, the ordering among con arguments established by the selection criterion could make the incision function fail to be warranting. That is, since incisions are determined by selections, the only solution to this issue is for the selection criterion to propose another candidate. This involves an update of the initial order. Such update will provoke the selections to reassign some of the original mappings. Thereafter, the incisions over the new mappings will also be reassigned, and finally preservation would be satisfied, thus obtaining a warranting incision. For instance, in order to guarantee preservation, in Ex.~\ref{ex.preservation} it is suggested the selection over $\lambda_2$ to be reassigned from $\Brg_6$ to $\Brg_2$. Since Def.~\ref{def:selection3} determines the mapping $\deactsel(\lambda_2)$ to the ``best'' argument (\ie $\Brg_6$) according to the selection criterion, the alternative to provoke the selection to determine a different mapping (\ie $\Brg_2$) is to update (or change) the ordering among con arguments.}  

Note that the selection order is being re-accommodated to the detriment of the first axis of change but favoring the other two, bringing balance among the three axes. This is quite natural since the initial order is just a general posture but the updated order would finally suit the particular domain in which the line is immersed: the dialectical tree. The set determined by the selection criterion is updated through:

\begin{center}
\textbf{(Update rule)} For every $\lambda\in\dtree{\Arg}{\PP}$, $\lambda'\in\dtree{\Arg}{\PP}$, and $\Brg\in\lambda$ where \textit{preservation} does not hold for $\colinc{\incise(\deactsel(\lambda'))}{\Brg}$, then the new selection order is $\selcrit\setminus\{(\Brg_1,\Brg_2)\in\selcrit \ | \ \Brg_1\notin\upsegmeq{}{\Brg}$ or $\Brg_2\notin\upsegmeq{}{\Brg}\}$.
\end{center}

\ojo{Observe that the update rule forces a set \selcrit to be replaced by $\selcritP{\upsegmeq{}{\Brg}}$. This will prevent an argument placed below \Brg to be mapped by the selection function $\deactsel(\lambda)$, and therefore, preservation will now be satisfied for the case $\colinc{\incise(\deactsel(\lambda'))}{\Brg}$, which failed beforehand.}

\begin{Theorem}\label{theo.warranting}
There is always a selection criterion leading to a warranting incision.
\end{Theorem}
\begin{proof}
To this end, it is sufficient to show one selection criterion that always allows for a warranting incision function: the selection of the root's direct defeaters. Hence, given a \delp\ $\PP = \SD$ and a dialectical tree $\dtree{\Arg}{\PP}\in\accTrees{\PP}$, for every $\lambda\in\dtree{\Arg}{\PP}$, it follows $\deactsel(\lambda)$ defeats \Arg, thus $\lambda=[\Arg,\deactsel(\lambda),\ldots]$. 
By \textit{reductio ad absurdum}, let us assume that such selection criterion does not lead to a warranting incision function. This means that at least one incision is not compliant with the preservation principle, \ie an incision of a selected argument over $\lambda'\in\dtree{\Arg}{\PP}$ triggers a collateral incision over an argument $\Brg\in\lambda$ where $\lambda\in\dtree{\Arg}{\PP}$, namely $\colinc{\incise(\deactsel(\lambda'))}{\Brg}\neq\emptyset$, in a way that (a) $\deactsel(\lambda)\notin\lambda^{\uparrow}[\Brg]$, and/or (b) $\Brg = \Arg$.

For a), the only option we have is $\Arg=\Brg$, hence case a) is resolved in b). Afterwards, for b), we would necessarily have a collateral incision over the root argument. Thus, for some $\lambda'\in\dtree{\Arg}{\PP}$, it follows  $\colinc{\incise(\deactsel(\lambda'))}{\Arg}\neq\emptyset$. 
Note that there cannot be an argument \Crg in \dtree{\Arg}{\PP} such that $\Crg\subseteq\Arg$ and \Crg is a direct defeater for \Arg. For this to take place, $\Arg\cup\Crg\cup\Pi = \Arg\cup\Pi$ would have to be contradictory. Therefore, for any $\lambda\in\dtree{\Arg}{\PP}$, it is never the case that $\deactsel(\lambda)\subseteq\Arg$, and thus there is some set $X_{\lambda}\subseteq\deactsel(\lambda)$ that could be incised without provoking a collateral incision over \Arg, \ie $X_{\lambda}\cap\Arg=\emptyset$. Besides, since we want to warrant \Arg, it is natural to assume that any rule-based criterion should preserve \Arg from being incised, and therefore, for any $\Gamma'\subseteq(\DD\setminus\Arg)$ and any $\Gamma\subseteq\Arg$, it holds $\Gamma'\lesschg\Gamma$. Hence, for any $\lambda\in\dtree{\Arg}{\PP}$, $\incise(\deactsel(\lambda))\subseteq X_{\lambda}$ and therefore $\colinc{\incise(\deactsel(\lambda))}{\Arg}=\emptyset$ hold, which is absurd.

Finally, it is absurd to assume that selecting the root's direct defeaters does not lead to an incision function satisfying preservation, and therefore by effect of the \textit{update rule}, there is always a selection criterion leading to a warranting incision.
\end{proof}

The alteration of every line in a dialectical tree through a given warranting incision function, renders a warranting tree as is shown next. 

\begin{Theorem}
Given a \delp\ \PP and a dialectical tree $\dtree{\Arg}{\PP}\in\accTrees{\PP}$, for any warranting incision function ``\incise'', \Arg ends up warranted from $\PP\setminus\bigcup_{\scriptsize\lambda\in\bundleSet{\Arg}{\tiny\PP}}\incise(\deactsel(\lambda))$.
\end{Theorem}
\begin{proof}
Straightforward from Theorem~\ref{theo.warranting},  Lemma~\ref{lemma.warrantingIncision.effectiveAlteration}, Def.~\ref{def.eff.alt}, and Corollary~\ref{corollary.warrant}.
\end{proof}

Regarding the amount of change provoked to a \delp, following the theorem above, no minimality is pursued so far. In the rest of the article we will provide additional theoretical elements towards minimal change. To this end, we will study how to bring balance among the three axes of change introduced on page~\pageref{axes.of.change}.

\ojo{As suggested before, when collateral incisions are unavoidable, we could still take advantage of them. That is, a collateral incision could be ``forced'' to provoke an attacking line to turn into non-attacking by collaterally incising its selected argument. Such a side-effect would be profitable, and it is described by the following principle. (Example~\ref{ex.profitability} illustrates the verification of this principle.)}

\begin{center}
\textbf{(Profitability)}
If $\colinc{\incise(\deactsel(\lambda'))}{\Brg}\neq\emptyset$ then 
$\lambda\in\ALINESPP$ and $\deactsel(\lambda) = \Brg$,\\
for every $\lambda\in\dtree{\Arg}{\PP}$, $\lambda'\in\dtree{\Arg}{\PP}$, and $\Brg\in\lambda$ 
\end{center}

The \emph{profitability} principle validates only those cases in which every collateral incision occurs over an attacking line and coincides with the selection in that line. Therefore, by updating the selection criterion (applying the update rule) towards profitability verification, we have the chance to take advantage of collateral incisions and to effectively alter several attacking lines at once, \ojo{thus reducing the amount of deleted rules in the \delp\ } However, this principle is not always possible to satisfy. For such cases, a \textit{weak profitability} principle is proposed.

\begin{center}
\textbf{(Weak Profitability)} If $\colinc{\incise(\deactsel(\lambda'))}{\Brg}\neq\emptyset$ then $\deactsel(\lambda) = \Brg$,\\
for every $\lambda\in\dtree{\Arg}{\PP}$, $\lambda'\in\dtree{\Arg}{\PP}$, and $\Brg\in\lambda$ 
\end{center}

Weak profitability skips checking whether the collaterally altered line is contained by the attacking set or not. 
\ojo{However, it still satisfies that the collaterally incised argument is the selection in that line, \ie satisfying this principle would still help to reduce the deletion of rules from the \delp}

\begin{Example}\label{ex.profitability}\ 
\begin{window}[0,r,{\mbox{\epsfig{file=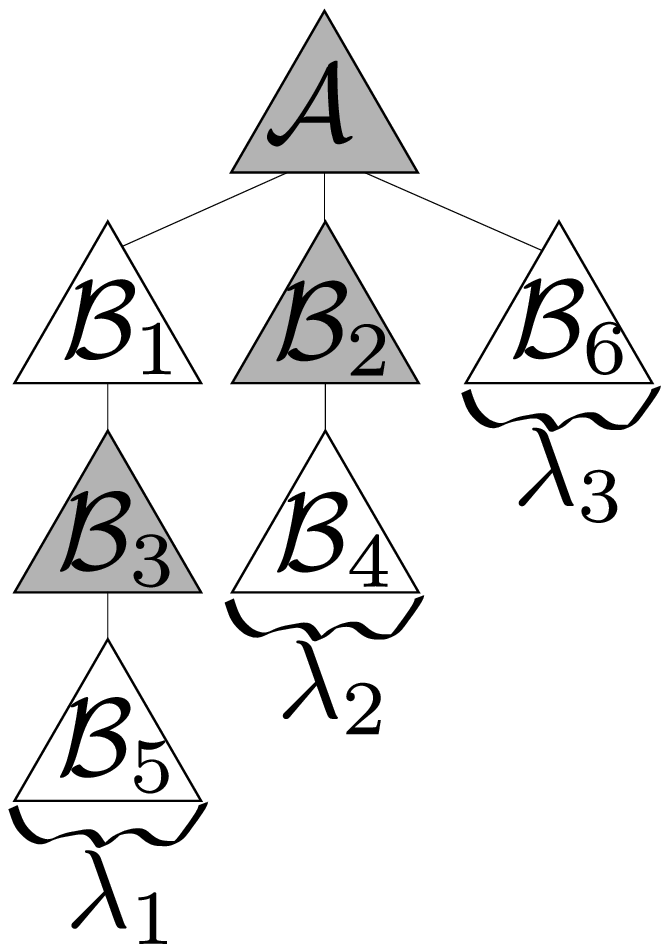,scale=.32}}},{}]
Consider the tree $\dtree{\Arg}{\PP}$ on the right, with three argumentation lines: $\lambda_1 = [\Arg,\Brg_1,\Brg_3,\Brg_5]$, $\lambda_2 = [\Arg,\Brg_2,\Brg_4]$, and $\lambda_3 = [\Arg,\Brg_6]$, of which $\lambda_1$ and $\lambda_3$ are attacking lines. Assume the selections $\deactsel(\lambda_1) = \Brg_5$, $\deactsel(\lambda_2) = \Brg_2$, and $\deactsel(\lambda_3) = \Brg_6$. The following table shows potential configurations of incisions and collateral incisions, along with their compliance with the principles of profitability and weak profitability:
\end{window}
\begin{center}
\begin{tabular}{cccc}
\hline
Incision & Coll. Inc. & Profitability & Weak Profitability\\
\hline
$\Brg_5$ & $\Brg_4$ & No, $\deactsel(\lambda_2) \neq \Brg_4$ and $\lambda_2\notin\ALINESPP$ & No, $\deactsel(\lambda_2) \neq \Brg_4$ \\
$\Brg_5$ & $\Brg_2$ & No, $\lambda_2\notin\ALINESPP$ & Yes \\
$\Brg_5$ & $\Brg_6$ & Yes & Yes\\
\end{tabular}
\end{center}
\end{Example}

\begin{Proposition}\label{prop.profitability}
Given a \delp\ \PP, a dialectical tree $\dtree{\Arg}{\PP}\in\accTrees{\PP}$, and an argument incision function ``$\incise$'':
\vspace{-7mm}\begin{description}
\item\indent
\begin{enumerate}
\item if profitability is satisfied then weak profitability is satisfied;
\item if weak profitability is satisfied then preservation is satisfied.
\end{enumerate}
\end{description}
\end{Proposition}
\begin{proof}
1. Straightforward from profitability and weak-profitability. 

2. For any $\lambda\in\dtree{\Arg}{\PP}$ and any $\Brg\in\lambda$, $\Brg=\deactsel(\lambda)$ holds. Afterwards, $\Brg\in\lambda^{\uparrow}[\Brg]$ and $\Arg\neq\Brg$ hold given that $\Brg\in\lambda^-$ (see Def.~\ref{def:incision3}). Finally, preservation is satisfied. 
\end{proof}

\subsection{Alteration Set Recognition}\label{sec.alteration}

Given $\dtree{\Arg}{\PP}\in\accTrees{\PP}$, the warrant for \Arg may be obtained by effectively altering a subset $X\subseteq\bundle$ of its lines via incisions. This would render a (altered) warranting tree. Collateral incisions may appear as the main drawback, hence the analysis will require to identify \textit{hypothetical trees} $\htree(\Arg,\Psi)$: dialectical trees that would result by removing the defeasible rules contained in a given set $\Psi\subseteq\PP$ from the \delp\ \PP. 

\begin{Definition}[Hypothetical Tree]\label{def.hyp}
Given a \delp\ $\PP = \SD$, an argument $\Arg\in\ARGS$, the tree $\dtree{\Arg}{\PP}\in\accTrees{\PP}$, and a set $\Psi\subseteq\Delta$ of defeasible rules; the \textbf{hypothetical tree} $\htree(\Arg,\Psi)$ is the dialectical tree built from the set $X_1\cup X_2$ of lines, where $X_1$ and $X_2$ are defined as follows:
\begin{description}
\item $X_1=\{\lambda\in\dtree{\Arg}{\PP}\ |\ \forall\Brg\in\lambda: \Psi\cap \Brg=\emptyset\}$
\item $X_2=\{\upsegm{\lambda}{\Brg}\ |\ \lambda\in\dtree{\Arg}{\PP}$ such that \ojo{$\exists\Brg\in\lambda,\forall\Brg'\in\upsegm{\lambda}{\Brg}:$\\ $\Psi\cap\Brg\neq\emptyset$ and $\Psi\cap\Brg'=\emptyset\}$}
\end{description}
\end{Definition}

Observe that hypothetical trees are built from sets of lines which may consider non-exhaustive lines, and therefore, $\htree(\Arg,\Psi)$ for any $\Psi$, is contained in the set $\allTrees{\PP}$ of (non-acceptable) dialectical trees from \PP. 

\begin{Proposition}\label{prop.htree.allTrees}
Given a \delp\ $\PP = \SD$ and an argument $\Arg\in\ARGS$, for any set $\Psi\subseteq\Delta$ of defeasible rules, it holds $\htree(\Arg,\Psi)\in\allTrees{\PP}$.
\end{Proposition}
\begin{proof}
From Def.~\ref{def.hyp}, a hypothetical tree $\htree(\Arg,\Psi)$ is built by a set $X_1\cup X_2$ of lines. If $\Psi=\emptyset$ it is easy to see that $X_1=\bundle$, where $\bundle$ is the bundle set of the tree $\dtree{\Arg}{\PP}\in\accTrees{\PP}$, and $X_2=\emptyset$. Hence, $\htree(\Arg,\Psi)=\dtree{\Arg}{\PP}$ and therefore $\htree(\Arg,\Psi)\in\allTrees{\PP}$. The same situation occurs when $\Psi$ is such that there is no defeasible rule $\beta\in\Psi$ such that $\beta\in\Brg$, where $\Brg\in\lambda$ for any $\lambda\in\dtree{\Arg}{\PP}$. When any of these alternatives hold, we know $X_2\neq\emptyset$ and therefore $\htree(\Arg,\Psi)$ is known to consider some non-exhaustive lines. This is so, given that $X_2$ will contain upper segments $\upsegm{\lambda}{\Brg}$ of lines $\lambda\in\dtree{\Arg}{\PP}$ for some $\Brg\in\lambda$. Since $\dtree{\Arg}{\PP}\in\accTrees{\PP}$, we know it is composed by acceptable and exhaustive lines from $\exLines{\PP}$. From Remark~\ref{remark.lines}, we have that $\lambda\in\Lines{\PP}$, and from Proposition~\ref{prop.uppersegments2lines} and Def.~\ref{def:upper:segment}, we have that $\upsegm{\lambda}{\Brg}\in\Lines{\PP}$ for any $\Brg\in\lambda$. Finally, $\htree(\Arg,\Psi)\in\allTrees{\PP}$ holds for any $\Psi$.
\end{proof}

As being (informally) introduced before, the \textit{alteration set} of a dialectical tree \dtree{\Arg}{\PP}, is a subset of lines from \bundle such that their effective alteration will determine a resulting warranting tree for \Arg. Ideally, the attacking set \ALINESPP will be included within this set, however, other lines could arise to be altered: those that may collaterally turn to attacking, introducing a new source of threat for the root's status of warrant. Based on the notion of hypothetical tree, \textit{collaterality functions} bring a tool for accounting on collateralities which are determined by the effective alteration of lines. The formalization of alteration set (given in Def.~\ref{def.attset}) will rely on the notion of collaterality functions introduced next.

\begin{Definition}[Collaterality Functions]\label{def.collaterality.functions}
Given a \delp\ \PP, an argument $\Arg\in\ARGS$, and a warranting incision function ``\incision''. Functions $\open:\exLines{\PP}\rightarrow 2^{\scriptsize\exLines{\PP}}$ and $\closed:\exLines{\PP}\rightarrow 2^{\scriptsize\exLines{\PP}}$, are referred to as \textbf{collaterality functions} \ifff for any $\lambda\in\exLines{\PP}$, if $\lambda\notin\dtree{\Arg}{\PP}$ then both functions map to $\emptyset$, otherwise:

$\closed(\lambda)=\{\lambda'\ |\ \lambda'\in\dtree{\Arg}{\PP}, \incise(\deactsel(\lambda))\cap\deactsel(\lambda')\neq\emptyset\}$

$\open(\lambda)=\{\lambda'\ |\ \lambda'\in\dtree{\Arg}{\PP}, \upsegm{\lambda'}{\Brg}\in\ALINES{\htree(\Arg,\incise(\deactsel(\lambda)))}$ for some $\Brg\in\lambda'\}$

\ojo{We refer to \open as \textbf{open}, and to \closed as \textbf{closed}.}
\end{Definition}


\ojo{The effective alteration of a line $\lambda\in\dtree{\Arg}{\PP}$ renders three different types of collateral alterations: those that turn effective the alteration of a line $\lambda'\in\dtree{\Arg}{\PP}$ being collaterally incised over its selected argument --included in $\closed$--, those that turn to attacking --included in $\open$--, and those that are effective, but collaterally incised over an argument placed below its selection. Collateral incisions over an argument placed above the selection of a line cannot occur given that ``\incision'' is ensured to be warranting, thus satisfying preservation. We will refer to lines in \open as \textit{open} given that they represent a still open problem: they are attacking or may collaterally turn to attacking. These lines need to be incised in order to be effectively altered. On the other hand, we refer to lines in \closed as \textit{closed} since they were already (collaterally) effectively altered: they no longer threaten the root's warrant status.}

\begin{Example}\label{ex.alteration.set}\
\begin{window}[4,r,{\mbox{\epsfig{file=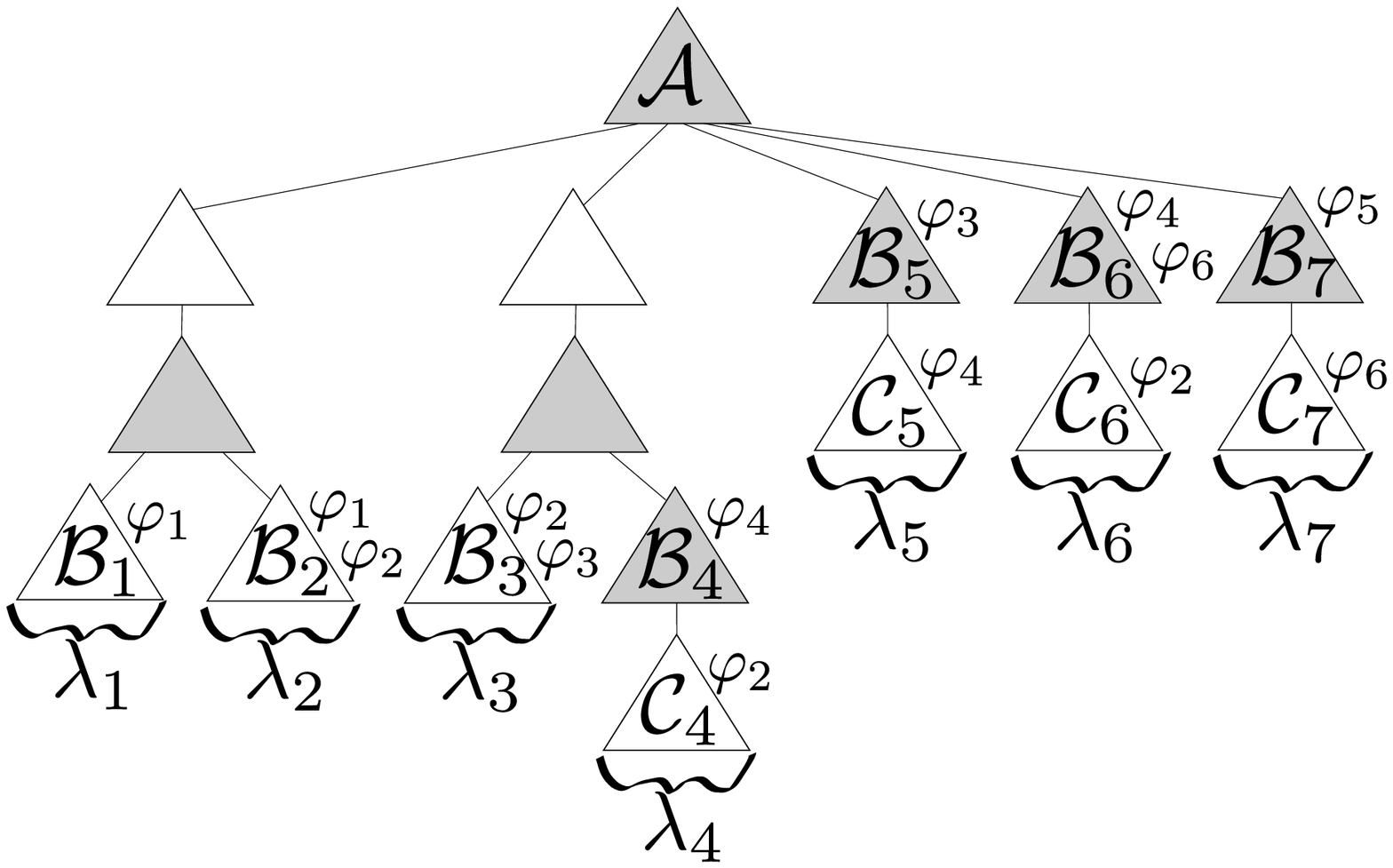,scale=.32}}},{}]
Consider the dialectical tree $\dtree{\Arg}{\PP}\in\accTrees{\PP}$ depicted below, on the right, with an attacking set $\ALINESPP=$ $\{\lambda_1, \lambda_2, \lambda_3\}$. Let us assume the selection criterion to determine the following mappings: $\deactsel(\lambda_1)=\Brg_1$, $\deactsel(\lambda_2)=\Brg_2$, $\deactsel(\lambda_3)=\Brg_3$, $\deactsel(\lambda_4)=\Brg_4$, $\deactsel(\lambda_5)=\Brg_5$, $\deactsel(\lambda_6)=\Brg_6$, and $\deactsel(\lambda_7)=\Brg_7$. Let $\varphi_1,$ $\varphi_2,$ $\varphi_3,$ $\varphi_4,$ $\varphi_5,$ and $\varphi_6$ be defeasible rules in the \delp\ \PP such that $\{\varphi_1\}\subseteq\Brg_1$, $\{\varphi_1,\varphi_2\}\subseteq\Brg_2$, $\{\varphi_2,\varphi_3\}\subseteq\Brg_3$, $\{\varphi_2\}\subseteq\Crg_4$, $\{\varphi_4\}\subseteq\Brg_4$, $\{\varphi_4\}\subseteq\Crg_5$, $\{\varphi_3\}\subseteq\Brg_5$, $\{\varphi_2\}\subseteq\Crg_6$, $\{\varphi_4,\varphi_6\}\subseteq\Brg_6$, $\{\varphi_6\}\subseteq\Crg_7$, $\{\varphi_5\}\subseteq\Brg_7$. Assume also the rule-based criterion:\\ 
$\{\varphi_6\}\lesschg\{\varphi_5\}\lesschg$ $\{\varphi_4\}\lesschg\{\varphi_3\}\lesschg\{\varphi_2\}\lesschg\{\varphi_1\}$, \\ and the incision function mappings:\\ 
$\incise(\deactsel(\lambda_1))=\{\varphi_1\}$, $\incise(\deactsel(\lambda_2))=\{\varphi_2\}$, $\incise(\deactsel(\lambda_3))=\{\varphi_3\}$, $\incise(\deactsel(\lambda_4))=\{\varphi_4\}$,  $\incise(\deactsel(\lambda_5))=\{\varphi_3\}$,  $\incise(\deactsel(\lambda_6))=\{\varphi_6\}$, and $\incise(\deactsel(\lambda_7))=\{\varphi_5\}$. \\ Observe that ``\incision'' is a warranting incision function since it satisfies preservation. The collaterality functions are defined as follows:
\end{window}

\begin{tabular}{l l}
$\open(\lambda_1)=\{\}$&$\closed(\lambda_1)=\{\lambda_1,\lambda_2\}$\\
$\open(\lambda_2)=\{\lambda_4,\lambda_6\}$&$\closed(\lambda_2)=\{\lambda_2,\lambda_3\}$\\
$\open(\lambda_3)=\{\}$&$\closed(\lambda_3)=\{\lambda_3,\lambda_5\}$\\
$\open(\lambda_4)=\{\lambda_5\}$&$\closed(\lambda_4)=\{\lambda_4,\lambda_6\}$\\
$\open(\lambda_5)=\{\}$&$\closed(\lambda_5)=\{\lambda_3,\lambda_5\}$\\
$\open(\lambda_6)=\{\lambda_7\}$&$\closed(\lambda_6)=\{\lambda_6\}$\\
$\open(\lambda_7)=\{\}$&$\closed(\lambda_7)=\{\lambda_7\}$
\end{tabular}

\ojo{The open set for $\lambda_1$ is empty because the incision over the selection in $\lambda_1$ does not ``open'' any line, \ie it does not collaterally turn any line into attacking in the context of the hypothetical tree $\htree(\Arg,\incise(\deactsel(\lambda_1)))$. The closed set for $\lambda_1$ includes lines $\lambda_1$ and $\lambda_2$, as the incision over the selection in $\lambda_1$ ``closes'' both lines, \ie turns them into non-attacking. On the other hand, if we look at $\lambda_2$, its open set is $\{\lambda_4,\lambda_6\}$, as the incision over the selection in $\lambda_2$ is $\varphi_2$ which collaterally incises $\Crg_4$ and $\Crg_6$, thus turning both $\lambda_4$ and $\lambda_6$ into attacking.  According to Def.~\ref{def.collaterality.functions}, this turns out from the analysis of the hypothetical tree $\htree(\Arg,\incise(\deactsel(\lambda_2)))=\htree(\Arg,\{\varphi_2\})$. That is, since $\upsegm{\lambda_4}{\Crg_4}\in\ALINES{\htree(\Arg,\{\varphi_2\})}$ and $\upsegm{\lambda_6}{\Crg_6}\in\ALINES{\htree(\Arg,\{\varphi_2\})}$ hold, we obtain $\open(\lambda_2)=$ $\{\lambda_4,\lambda_6\}$. All collaterality functions are obtained in a similar way.
}
\end{Example}

\ojo{The following properties for open and closed sets interrelate the collaterality functions presented in Def.~\ref{def.collaterality.functions}.}

\begin{Proposition}\label{prop.identity.closed}
Given a \delp\ \PP, a dialectical tree $\dtree{\Arg}{\PP}\in\accTrees{\PP}$, and a warranting incision function ``\incision'', for any $\lambda\in\dtree{\Arg}{\PP}$, it holds $\lambda\in\closed(\lambda)$.
\end{Proposition}
\begin{proof}
From Def.~\ref{def.collaterality.functions}, the set $\closed(\lambda)$ contains every line $\lambda'\in\dtree{\Arg}{\PP}$ whose selected argument $\deactsel(\lambda')$ contains at least a rule $\varphi$ such that $\varphi\in\incise(\deactsel(\lambda))$. This holds in particular when $\lambda=\lambda'$.
\end{proof}

\begin{Proposition}\label{prop.non.identity.open}
Given a \delp\ \PP, a dialectical tree $\dtree{\Arg}{\PP}\in\accTrees{\PP}$, and a warranting incision function ``\incision'', for any $\lambda\in\dtree{\Arg}{\PP}$, it holds $\lambda\notin\open(\lambda)$.
\end{Proposition}
\begin{proof}
By \textit{reductio ad absurdum}, assuming $\lambda\in\open(\lambda)$, from Def.~\ref{def.collaterality.functions}, we have $\upsegm{\lambda}{\Brg}\in\ALINES{\htree(\Arg,\incise(\deactsel(\lambda)))}$, for some $\Brg\in\lambda$. 
\ojo{However, since $\incise$ is warranting, we know preservation is satisfied which means that any collateral incision in a line will occur over some argument placed below the selected argument in that line (or over the selection itself). In this proof we are interested in the case in which such a line is $\lambda$ itself (given that we assumed $\lambda\in\open(\lambda)$). Hence, for any $\Brg\in\lambda$, if $\colinc{\incise(\deactsel(\lambda))}{\Brg}\neq\emptyset$ then $\deactsel(\lambda)\in\upsegmeq{}{\Brg}$ holds. Afterwards, since $\upsegm{\lambda}{\Brg}\in\ALINES{\htree(\Arg,\incise(\deactsel(\lambda)))}$, for some $\Brg\in\lambda$; we know that $\colinc{\incise(\deactsel(\lambda))}{\Brg}\neq\emptyset$ and also $\deactsel(\lambda)\in\upsegmeq{}{\Brg}$ hold. Thus, the only alternative is $\Brg=\deactsel(\lambda)$ to be held (recall that the colateral incision over $\Brg$ is the uppermost one in $\lambda$). 
}

Finally, since $\lambda$ is effectively altered through $\incise(\deactsel(\lambda))$ (see Lemma~\ref{lemma.warrantingIncision.effectiveAlteration}), we know that $\upsegm{\lambda}{\Brg}\notin\ALINES{\htree(\Arg,\incise(\deactsel(\lambda)))}$ holds reaching the absurdity.
\end{proof}

\begin{Corollary}
Given a \delp\ \PP, a dialectical tree $\dtree{\Arg}{\PP}\in\accTrees{\PP}$, and a warranting incision function ``\incision'', for any $\lambda\in\dtree{\Arg}{\PP}$, it holds $\open(\lambda)\cap\closed(\lambda)=\emptyset$.
\end{Corollary}

\ojo{Now we are able to formalize the definition of \textit{alteration set} by relying upon the open collaterality function as mentioned before.}

\begin{Definition}[Alteration Set]\label{def.attset}
Given a \delp\ \PP, a dialectical tree \linebreak$\dtree{\Arg}{\PP}\in\accTrees{\PP}$, and a warranting incision function ``\incision''; the \textbf{alteration set} \alteration{\Arg}{\PP} of \dtree{\Arg}{\PP} is the least fixed point of the operator $\attFunct$ defined as follows:
\begin{description}
\item $\attFunct^0\hspace*{11pt}=\ALINESPP$, and
\item $\attFunct^{k+1}=\attFunct^k\cup\bigcup_{\scriptsize\lambda\in\attFunct^k}\open(\lambda)$
\end{description}
\end{Definition}

\begin{Example}[Continues from Ex.~\ref{ex.alteration.set}]\label{ex.alteration.set2}
The alteration set $\alteration{\Arg}{\PP}$ is constructed as follows: 
$\attFunct^0=\{\lambda_1, \lambda_2, \lambda_3\}$, conciding with \ALINESPP; 
$\attFunct^1=\{\lambda_1, \lambda_2,$ $\lambda_3,$ $\lambda_4, \lambda_6\}$, given that both $\lambda_4$ and $\lambda_6$ are collaterally open (turned to attacking) from the effective alteration of $\lambda_2$, that is $\upsegm{\lambda_4}{\Crg_4}$ and $\upsegm{\lambda_6}{\Crg_6}$ are contained in $\ALINES{\htree(\Arg,\{\varphi_2\})}$; and analogously $\attFunct^2=\{\lambda_1, \lambda_2, \lambda_3, \lambda_4, \lambda_5, \lambda_6,\lambda_7\}$ is calculated. 
Observe that $\attFunct^2=\attFunct^3$ and hence $\attFunct^2=\alteration{\Arg}{\PP}$, determining the least fixed point of the operator \attFunct. Finally, $\alteration{\Arg}{\PP}=\{\lambda_1, \lambda_2, \lambda_3, \lambda_4, \lambda_5, \lambda_6, \lambda_7\}$. 
\end{Example}

%


\begin{Remark}\label{remark.alteration.contains.attset}
Given a \delp\ \PP and a dialectical tree $\dtree{\Arg}{\PP}\in\accTrees{\PP}$, the following conditions for an alteration set \alteration{\Arg}{\PP} with a warranting incision ``\incision'', are met:
\vspace{-6mm}\begin{description}
\item\indent\begin{enumerate}
\item $\ALINESPP\subseteq\alteration{\Arg}{\PP}$, and\label{item.remark.alines.in.alt}
\item $\bigcup_{\scriptsize\lambda\in\alteration{\Arg}{\PP}}\open(\lambda)\subseteq\alteration{\Arg}{\PP}$.\label{item.remark.open.in.alt}
\end{enumerate}
\end{description}
\end{Remark}

From now on, just for simplicity, we will rely on the operator $\Incisions:2^{\scriptsize\exLines{\tiny\PP}}\rightarrow 2^{\scriptsize\Ld}$ such that  $\Incisions(X)=\bigcup_{\scriptsize\lambda\in X}\incise(\deactsel(\lambda))$ for any $X\subseteq\bundle$ and any $\Arg\in\ARGS$, to refer to the composition of selections and incisions over lines included in the set $X$.

\begin{Lemma}\label{lemma.altset.warranting}
Given the alteration set \alteration{\Arg}{\PP} and a warranting incision function ``\incise''; if $\Psi = \Incisions(\alteration{\Arg}{\PP})$ then $\htree(\Arg, \Psi)$ is warranting.
\end{Lemma}
\begin{proof}
Since \incise is a warranting incision function and considering the dialectical tree $\dtree{\Arg}{\PP}\in\accTrees{\PP}$, from Lemma~\ref{lemma.warrantingIncision.effectiveAlteration}, we know that $\PP\setminus\incise(\deactsel(\lambda))$ effectively alters any $\lambda\in\dtree{\Arg}{\PP}$. The set $\Psi$ contains every incision $\incise(\deactsel(\lambda))$ for every $\lambda\in\alteration{\Arg}{\PP}$. Then the following properties arise:
\vspace{-6mm}\begin{description}
\item\indent\begin{enumerate}
\item for any $\lambda\in\dtree{\Arg}{\PP}$ such that $\lambda\in\ALINESPP$, $\lambda\in\alteration{\Arg}{\PP}$,\label{item.lemma.alines.in.alt} (this follows from \ref{item.remark.alines.in.alt} in Remark~\ref{remark.alteration.contains.attset})
\item for any $\lambda\in\dtree{\Arg}{\PP}$ such that $\lambda\in\alteration{\Arg}{\PP}$ and $\lambda\notin\ALINESPP$, $\lambda$ is collaterally turned to attacking by an incision $\incise(\deactsel(\lambda'))$, of some $\lambda'\in\alteration{\Arg}{\PP}$,\label{item.lemma.turned.to.att.in.alt} (this follows from \ref{item.remark.open.in.alt} in Remark~\ref{remark.alteration.contains.attset} and from Def.~\ref{def.collaterality.functions})
\item for any $\lambda\in\dtree{\Arg}{\PP}$ such that $\lambda\notin\alteration{\Arg}{\PP}$, if $\lambda$ is attacking then $\lambda\notin\ALINESPP$. That is, there is $\lambda'\in\ALINESPP$ such that $\lambda$ and $\lambda'$ are adjacent at an argument marked as \Unode\ (see Def.~\ref{def.alines.delp} and Def.~\ref{def.attset}), and\label{item.lemma.att.adj.unode.out.alt}
\item for any $\lambda\in\dtree{\Arg}{\PP}$ such that $\lambda\notin\alteration{\Arg}{\PP}$, $\lambda$ is not collaterally turned to attacking (see Def.~\ref{def.attset} and Def.~\ref{def.collaterality.functions}).\label{item.lemma.lines.out.alt}
\end{enumerate}
\end{description}

Consequently, we have that \alteration{\Arg}{\PP} contains every line in \ALINESPP (property \ref{item.lemma.alines.in.alt}) along with every line that collaterally turns to attacking (property \ref{item.lemma.turned.to.att.in.alt}). Only from property \ref{item.lemma.att.adj.unode.out.alt}, a line $\lambda$ can be attacking but not contained in \alteration{\Arg}{\PP}. In this case, $\lambda$ will be (collaterally) effectively altered by the alteration of its adjacent line $\lambda'$ (see Lemma~\ref{lemma.adj.att}). From property \ref{item.lemma.lines.out.alt}, we know that any other line outside \alteration{\Arg}{\PP} does not threaten the warrant status of the root argument. Afterwards, it is easy to see that $\htree(\Arg, \Psi)$ contains no attacking lines. Finally, from Theorem~\ref{theorem:warrant}, $\htree(\Arg, \Psi)$ is known to be a warranting tree.
\end{proof}


By effectively altering every line in the alteration set, through Lemma~\ref{lemma.altset.warranting} we ensure that the resulting hypothetical tree $\htree(\Arg,\Psi)$ ends up warranting. However, if we consider minimal change to force the set $\Psi$ of incisions to be minimal, an additional condition is required to restrict the notion of alteration set given so far. 

\begin{Example}[Continues from Ex.~\ref{ex.alteration.set2}]\label{ex.incision-aware}\ 
\begin{window}[0,r,{\mbox{\epsfig{file=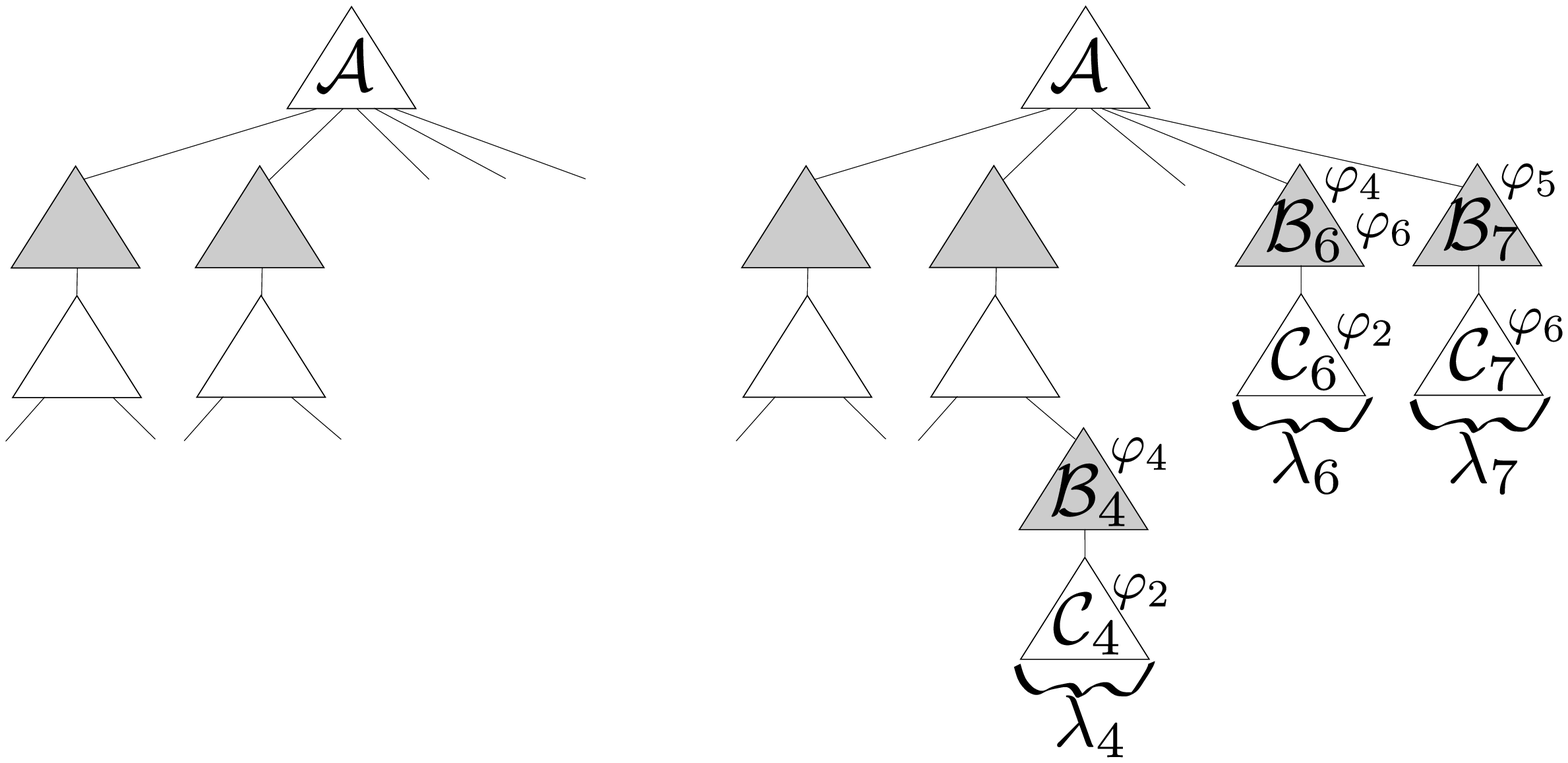,scale=.32}}},{}]
According to Lemma~\ref{lemma.altset.warranting}, the hypothetical tree $\htree(\Arg,\Psi)$ (leftmost tree depicted on the right), where $\Psi=\{\varphi_1,$ $\varphi_2,$ $\varphi_3,$ $\varphi_4,$ $\varphi_5,$ $\varphi_6\}$, ends up warranting \Arg. However, removing its subset $\Psi'=\{\varphi_1, \varphi_3\}$ from \PP\ is enough to alter every line in \ALINESPP, rendering the warranting hypothetical tree $\htree(\Arg,\Psi')$ (rightmost tree depicted on the right).
\end{window}
\end{Example}

Collaterality functions are not aware of the \textit{context} in which incisions are applied. 
That is, when a line $\lambda'$ is closed from $\closed(\lambda)$, we know it is effectively altered in a collateral way through the incision of $\lambda$. That means that the collateral incision of $\lambda'$ occurs over its selected argument, say \Brg. However, imagine that the collaterality occurs over a con argument \Crg placed below \Brg. In this case, although $\lambda'$ is not altered over its selected argument, it should be anyway considered closed (see Ex.~\ref{ex.alterationContext-sensitive}), \ojo{if it could be ensured that no other line will collaterally alter $\lambda'$ between \Brg and \Crg (we know there is no collateral alteration over an argument placed above \Brg given that we only consider a warranting incision function). In order to ensure this latter condition, we need keep track of all those lines that are going to be altered by inlcuding them in a set $X\subseteq\bundle$, for which set $X$ is identified as the \emph{context} at issue. In Ex.~\ref{ex.alterationContext-sensitive}, we illustrate a case in which a line ($\lambda_3$ in the example) could be ensured to be closed without deactivating its selected argument, if the rest of the lines to be altered are taken in consideration, \ie if we consider the context.}


\begin{Example}\label{ex.alterationContext-sensitive}
Consider the dialectical tree $\dtree{\Arg}{\PP}\in\accTrees{\PP}$ depicted on the right, with an attacking set $\ALINESPP=$ $\{\lambda_1, \lambda_2\}$. Let us assume the selection criterion to determine the following mappings: $\deactsel(\lambda_1)=\Brg_1$, $\deactsel(\lambda_2)=\Brg_2$, and $\deactsel(\lambda_3)=\Brg_3$. 
\begin{window}[0,r,{\mbox{\epsfig{file=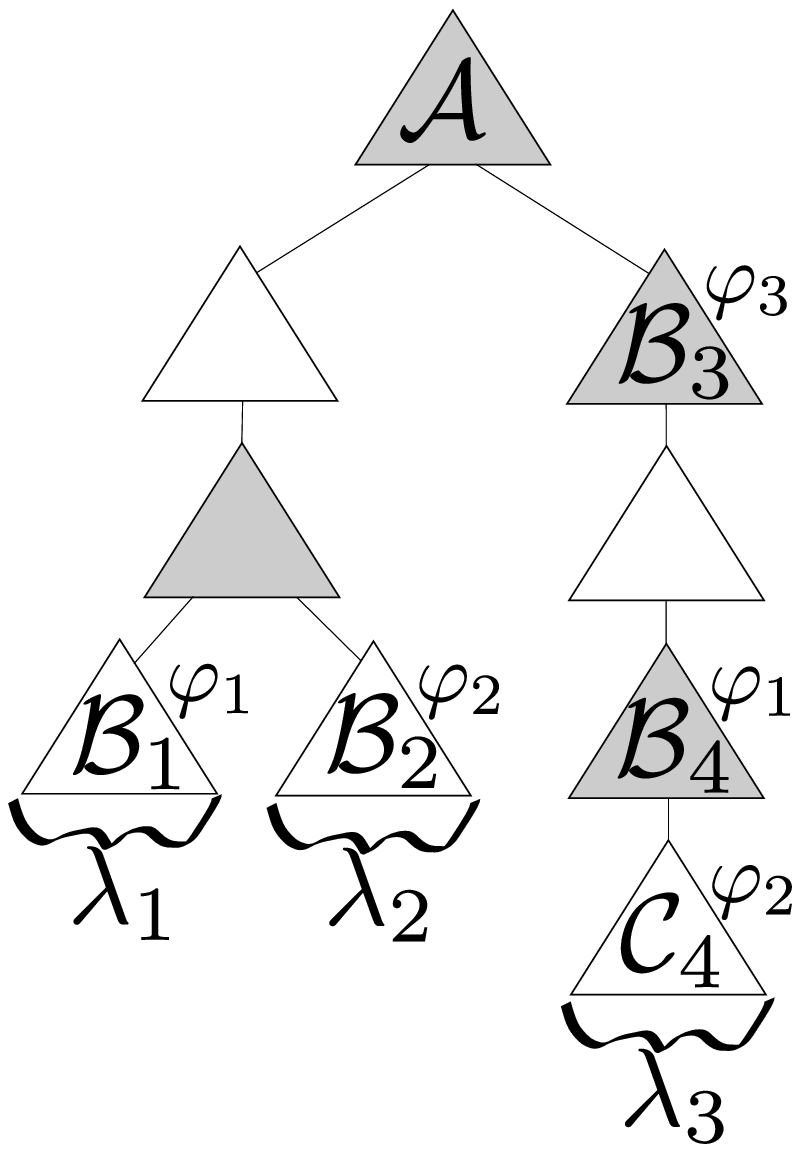,scale=.32}}},{}]
\noindent Let $\varphi_1,$ $\varphi_2,$ and $\varphi_3$ be defeasible rules in the \delp\ \PP such that $\{\varphi_1\}\subseteq\Brg_1$, $\{\varphi_2\}\subseteq\Brg_2$,  $\{\varphi_3\}\subseteq\Brg_3$, $\{\varphi_2\}\subseteq\Crg_4$, and  $\{\varphi_1\}\subseteq\Brg_4$. Assuming the rule-based criterion $\{\varphi_3\}\lesschg\{\varphi_2\}\lesschg\{\varphi_1\}$, the incision function would map as follows: $\incise(\deactsel(\lambda_1))=\{\varphi_1\}$, $\incise(\deactsel(\lambda_2))=\{\varphi_2\}$, and  $\incise(\deactsel(\lambda_3))=\{\varphi_3\}$. Observe that ``\incision'' is a warranting incision function since it satisfies preservation. The collaterality functions are defined as follows:
\end{window}
\begin{tabular}{l l}
$\open(\lambda_1)=\{\}$&$\closed(\lambda_1)=\{\lambda_1\}$\\
$\open(\lambda_2)=\{\lambda_3\}$&$\closed(\lambda_2)=\{\lambda_2\}$\\
$\open(\lambda_3)=\{\}$&$\closed(\lambda_3)=\{\lambda_3\}$\\
\end{tabular}
\end{Example}

Regarding open lines, when a line $\lambda'$ is included in $\open(\lambda)$, \ojo{we know it collaterally ends up being attacking from the incision over $\lambda$}. However, attacking lines are not considered open so far. \ojo{Next we define the \textit{context-sensitive collaterality functions} over a \textit{context set} $X\subseteq\bundle$ of lines.} As for its general version, the context-sensitive collaterality functions will include two inner functions: one for open lines and another one for closed lines. 
The context-sensitive open version will include the lines that are open from lines in the context $X$ (following Def.~\ref{def.collaterality.functions}), along with lines in the attacking set. On the other hand, the context-sensitive closed version will include every line being effectively altered as detailed before (in Ex.~\ref{ex.alterationContext-sensitive}, $\lambda_3$'s collateral alteration from the alteration of $\lambda_1$ should close $\lambda_3$), but taking into account the context. 
That is, lines which are closed by other lines contained in the context set $X$, along with those lines $\lambda'$ that are closed by being collaterally altered over a con argument $\Brg\in\lambda'^-$ placed below the selected argument in $\lambda'$, if it is the case that no other line in $X$ collaterally alters $\lambda'$ over an argument placed above \Brg.

\begin{Definition}[Context-sensitive Collaterality Functions]\label{def.context-sensitive.collaterality.functions}
Given a \delp\ \linebreak \PP, a tree $\dtree{\Arg}{\PP}\in\accTrees{\PP}$, and a warranting incision ``\incision''. Functions $\openContext:2^{\scriptsize\exLines{\PP}}\rightarrow 2^{\scriptsize\exLines{\PP}}$ and $\closedContext:2^{\scriptsize\exLines{\PP}}\rightarrow 2^{\scriptsize\exLines{\PP}}$, are referred to as \textbf{context-sensitive collaterality functions} \ifff for any set of lines $X\subseteq\exLines{\PP}$, if $X\not\subseteq\bundle$ (where $\bundle$ is the bundle set) then both functions map to $\emptyset$, otherwise:

\ojo{$\closedContext(X)=\{\lambda'\ |\ \lambda'\in\dtree{\Arg}{\PP} \wedge \exists\Crg\in\lambda'^-:(\exists\lambda\in X: \incise(\deactsel(\lambda))\cap\Crg\neq\emptyset)\wedge\linebreak\hspace*{3.2cm}(\forall \lambda''\in X:$ if $\incise(\deactsel(\lambda''))\cap\Brg\neq\emptyset$ where $\Brg\in\lambda'$ then $\Crg\in\upsegmeqP{}{\Brg})\}$}

$\openContext(X)=\ALINESPP\cup\bigcup_{\scriptsize\lambda\in X}\open(\lambda)$

We call \textbf{context-sensitive open} to \openContext, and \textbf{closed} to \closedContext.
\end{Definition}

\begin{Example}[Continues from Ex.~\ref{ex.alterationContext-sensitive}]
Note that if we consider a set of lines $X=\{\lambda_1,\lambda_2\}$ to be altered, the collaterality functions would determine  $\open(\lambda_1)\cup\open(\lambda_2)=\{\lambda_3\}$ and $\closed(\lambda_1)\cup\closed(\lambda_2)=\{\lambda_1,\lambda_2\}$, meaning that $\lambda_3$ is left unaltered. However, by following the notion of context-sensitive collaterality functions, we have 
$\openContext(X)=\{\lambda_1,\lambda_2,\lambda_3\}$, and $\closedContext(X)=\{\lambda_1,\lambda_2,\lambda_3\}$.
\end{Example}

\begin{Proposition}\label{prop.closed.in.context-sensitive.closed}
Given a \delp\ \PP, a dialectical tree $\dtree{\Arg}{\PP}\in\accTrees{\PP}$, and a warranting incision function ``\incision''; for any $\lambda\in\dtree{\Arg}{\PP}$, and any $X\subseteq\bundle$ (where \bundle is the bundle set of \dtree{\Arg}{\PP}), if $\lambda\in X$ then $\closed(\lambda)\subseteq\closedContext(X)$.
\end{Proposition}
\begin{proof}
Assuming $\lambda\in X$ we need to show that for any $\lambda'\in\closed(\lambda)$ it holds that $\lambda'\in\closedContext(X)$. From Def.~\ref{def.collaterality.functions}, we know that any $\lambda'\in\closed(\lambda)$ is such that $\incise(\deactsel(\lambda))\cap\deactsel(\lambda')\neq\emptyset$. Assume $\deactsel(\lambda')=\Crg$. Observe that $\Crg\in\lambda'^-$ (see Def.~\ref{def:selection3}). Finally, $\forall \lambda''\in X:$ if $\incise(\deactsel(\lambda''))\cap\Brg\neq\emptyset$ where $\Brg\in\lambda'$ then $\Crg\in\upsegmeqP{}{\Brg}$, is trivially satisfied given that \incise is a warranting incision function and $\Crg$ is the selected argument in $\lambda'$ (see \textit{preservation} on page~\pageref{preservation}). Finally, since all the conditions in Def.~\ref{def.context-sensitive.collaterality.functions} for a context-sensitive closed function are satisfied, it holds $\lambda'\in\closedContext(X)$.
\end{proof}

\begin{Proposition}\label{prop.identity.context-sensitive.closed}
Given a \delp\ \PP, a dialectical tree $\dtree{\Arg}{\PP}\in\accTrees{\PP}$, and a warranting incision function ``\incision''; for any $X\subseteq\bundle$ (where \bundle is the bundle set of \dtree{\Arg}{\PP}), it holds $X \subseteq\closedContext(X)$.
\end{Proposition}
\begin{proof}
Directly from Prop.~\ref{prop.identity.closed} and Prop.~\ref{prop.closed.in.context-sensitive.closed}.
\end{proof}

Note that, contrary to $\closed$, the operation $\closedContext$ is non-monotonic: given two sets $X$ and $Y$, of lines, if $X\subseteq Y$ then $\closedContext(X)\subseteq\closedContext(Y)$ is in general not satisfied. For space reasons we will not go further into this subject.

\begin{Lemma}\label{lemma.eff.alt.context-sensitive}
Given a \delp\ \PP, a dialectical tree $\dtree{\Arg}{\PP}\in\accTrees{\PP}$, and a warranting incision function ``\incision''; for any $\lambda\in\dtree{\Arg}{\PP}$, and any $X\subseteq\bundle$ (where \bundle is the bundle set of \dtree{\Arg}{\PP}), if every $\lambda'\in X$ is altered through $\incise(\deactsel(\lambda'))$ and $\lambda\in \closedContext(X)$ then $\lambda$ is effectively altered.
\end{Lemma}
\begin{proof}
From Prop.~\ref{prop.identity.context-sensitive.closed}, we know that $X\subseteq\closedContext(X)$. 
From Lemma~\ref{lemma.warrantingIncision.effectiveAlteration}, the alteration of every $\lambda'\in X$ rendering a new \delp\ $\PP'=\PP\setminus\Incisions(X)$, effectively alters every $\lambda'$ conforming to Def.~\ref{def.eff.alt}, \ie every $\lambda'\in X$ turns to non-attacking in $\PP'$. For the rest of the lines in $\closedContext(X)$, from Def.~\ref{def.context-sensitive.collaterality.functions}, if $\lambda\in\closedContext(X)$ then we know that there is some $\lambda'\in X$ such that 
$\incise(\deactsel(\lambda'))\cap\Crg\neq\emptyset$ where $\Crg\in\lambda^-$, and for any $\lambda''\in X$, if $\incise(\deactsel(\lambda''))\cap\Brg\neq\emptyset$ where $\Brg\in\lambda$ then $\Crg\in\upsegmeq{}{\Brg}$. This means that there is no line in $X$ whose incision could collaterally alter $\lambda$ in an argument placed above \Crg. Hence, \Crg is the uppermost collateral incision over $\lambda$ taking into account only the lines included in $X$. Finally, from Lemma~\ref{lemma.eff.alt}, since \Crg is a con argument in $\lambda$, we know $\lambda$ is effectively altered on \Crg.
\end{proof}

%

We finally restrict the definition of alteration set into the notion of \textit{incision-aware alteration set} to deal with situations as the ones described above. For this reduced alteration set, we will rely upon context-sensitive collaterality functions. Afterwards in Section~\ref{sec.principles}, an algorithm will be studied for future implementations.

\begin{Definition}[Incision-aware Alteration Set]\label{def.att2}
Given a dialectical tree \dtree{\Arg}{\PP} and a warranting incision function ``\incise'', the \textbf{incision-aware alteration set} of \dtree{\Arg}{\PP} is the set \aware{\Arg}{\PP} simultaneously satisfying:
\vspace{-6mm}\begin{description}
\item\indent\begin{enumerate}
\item $\aware{\Arg}{\PP}\subseteq\alteration{\Arg}{\PP}$,\label{item.aware.1}
\item $\openContext(\aware{\Arg}{\PP})\subseteq\closedContext(\aware{\Arg}{\PP})$,\label{item.aware.2}
\item there is no proper subset of $\aware{\Arg}{\PP}$ satisfying~\ref{item.aware.1} and~\ref{item.aware.2}, and\label{item.aware.4}
\item for any $X\neq\aware{\Arg}{\PP}$ satisfying conditions~\ref{item.aware.1} to \ref{item.aware.4}, it is not the case that $\Psi'\lesschg\Psi$, where $\Psi=\Incisions(\aware{\Arg}{\PP})$ and $\Psi' = \Incisions(X)$.\label{item.aware.5}
\end{enumerate}
\end{description}
\end{Definition}

\ojo{Following the four conditions of Def.~\ref{def.att2}, the incision-aware alteration set \aware{\Arg}{\PP} is (1) a subset of the alteration set $\alteration{\Arg}{\PP}$ such that (2) every line that was open was finally closed under the same context, which is \aware{\Arg}{\PP} itself, and (3) is the minimal set of lines --with regards to set inclusion-- that (4) provokes the least amount of change --according to the rule-based criterion $\lesschg$.}

\begin{Corollary}\label{corollary.inc-aware.altset.eff.alt}
Given the incision-aware alteration set \aware{\Arg}{\PP} and a warranting incision function ``\incise''; if every $\lambda\in\aware{\Arg}{\PP}$ is effectively altered through $\incise(\deactsel(\lambda))$ then every line in $\closedContext(\aware{\Arg}{\PP})$ is effectively altered.
\end{Corollary}

\begin{Proposition}\label{prop.alines.in.closedContext}
$\ALINESPP\subseteq\closedContext(\aware{\Arg}{\PP})$.
\end{Proposition}
\begin{proof}
From Def.~\ref{def.context-sensitive.collaterality.functions}, $\ALINESPP\subseteq\openContext(\aware{\Arg}{\PP})$ and from condition~\ref{item.aware.2} in Def.~\ref{def.att2}, $\openContext(\aware{\Arg}{\PP})\subseteq\closedContext(\aware{\Arg}{\PP})$ holds. Thus, $\ALINESPP\subseteq\closedContext(\aware{\Arg}{\PP})$.
\end{proof}


\begin{Theorem}\label{theo.inc-aware.altset.warranting}
Given the incision-aware alteration set \aware{\Arg}{\PP} and a warranting incision function ``\incise''; if $\Psi = \Incisions(\aware{\Arg}{\PP})$ then $\htree(\Arg, \Psi)$ is warranting.
\end{Theorem}
\begin{proof}
Since \incise is a warranting incision function and considering the dialectical tree $\dtree{\Arg}{\PP}\in\accTrees{\PP}$, from Corollary~\ref{corollary.inc-aware.altset.eff.alt}, we know that the effective alteration of every line in $\aware{\Arg}{\PP}$ determines the effective alteration of every line in $\closedContext(\aware{\Arg}{\PP})$. Since $\ALINESPP\subseteq\closedContext(\aware{\Arg}{\PP})$ (see Prop.~\ref{prop.alines.in.closedContext}) hold, we know that every line in $\ALINESPP$ is effectively altered. This means that no attacking lines appear in $\htree(\Arg,\Psi)$. Finally, from Corollary~\ref{corollary.warrant}, the hypothetical tree $\htree(\Arg,\Psi)$ is warranting.
\end{proof}
%
%

\ojo{While the regular alteration set looks for the ``minimal'' set of lines that has to be altered (without accounting on the incisions needed), the incision-aware alteration set pursues the same objective while looking for a minimum amount of incisions. 
For instance, in Ex.~\ref{ex.incision-aware}, the incision-aware alteration set ends up as $\aware{\Arg}{\PP}=\{\lambda_1,\lambda_3\}$. Observe that the set $\{\lambda_1,\lambda_2,\lambda_4,\lambda_5\}$ satisfies conditions \ref{item.aware.1} to \ref{item.aware.4} from Def.~\ref{def.att2}, but not condition \ref{item.aware.5}: the effective alteration of its lines determines a set of defeasible rules $\{\varphi_1,\varphi_2,\varphi_3,\varphi_4\}$ which naturally provokes more change (according to the rule-based criterion) than the set $\{\varphi_1,\varphi_3\}$ determined by the set $\aware{\Arg}{\PP}$.}

\ojo{However, being aware of the incisions to be made, determines a set of lines to be altered that could actually end up being smaller than the one determined by the regular alteration set. 
Such a situation ocurs, given that the incision-aware alteration set considers contextual information, and thus, the alteration of some lines included in the regular alteration set (and excluded from the incision-aware one) are achieved by taking into account advantageous collateralities, \ie those collateral incisions that end up in effective alterations. Nevertheless, although the incision-aware alteration set can be smaller, it ends up altering each of the lines contained in the regular alteration set, either in a direct way --through the application of the incision function in a line-- or by effect of a collateral incision. For instance, in Ex.~\ref{ex.ia-alteration.smaller.regular-alteration}, although the incision-aware alteration set (which ends up as $\aware{\Arg}{\PP}=\{\lambda_1,\lambda_2\}$) does not contain $\lambda_3$ (which is contained in the regular alteration set $\alteration{\Arg}{\PP}=\{\lambda_1,\lambda_2,\lambda_3\}$), it ends up effectively altering $\lambda_3$ by effect of the collaterality produced by the incision of $\lambda_1\in\aware{\Arg}{\PP}$. Thus, the incision-aware alteration set ends up altering (directly or through collateralities) every line included in the regular alteration set, however, the incision-aware alteration set does it by removing (occasionally) less rules from the \delp}


\begin{Example}[Continues from Ex.~\ref{ex.alterationContext-sensitive}]\label{ex.ia-alteration.smaller.regular-alteration}
The alteration set would be $\alteration{\Arg}{\PP}=$ 
\begin{window}[0,r,{\mbox{\epsfig{file=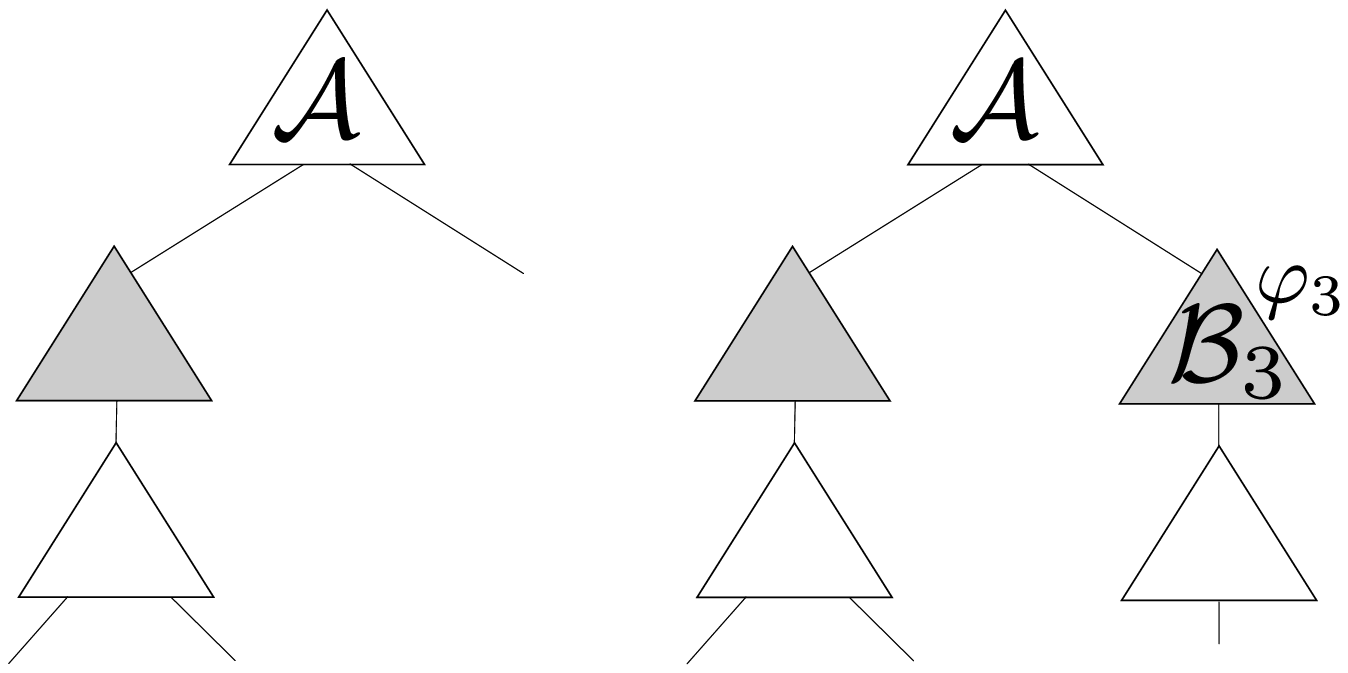,scale=.32}}},{}]
\noindent $\{\lambda_1,\lambda_2,\lambda_3\}$ determining a warranting tree $\htree(\Arg,\Psi)$ (leftmost tree depicted on the right), where $\Psi=\{\varphi_1,\varphi_2,\varphi_3\}$. However, the incision-aware alteration set would be $\aware{\Arg}{\PP}=\{\lambda_1,\lambda_2\}$ determining a warranting tree $\htree(\Arg,\Psi')$ (rightmost tree depicted on the right), where $\Psi'=\{\varphi_1,\varphi_2\}$.
\end{window}
\end{Example}


Assuming a warranting incision function ``\incision'' whose image maps only to sets of singletons, we can preserve a kind of minimality: any rule to be removed from a \delp\ is individually necessary to provide warrant to the root argument. This assertion aims at avoiding unnecessary removals from a \delp, and is required by the postulate of core-retainment, on page~\pageref{postulates}. Such a function is referred to as a \emph{minimally-warranting incision function}, and is defined next.

\begin{Definition}[Minimally-warranting Incision Function]\label{def.mwIncision}
An argument incision function ``$\incise$'' is said to be \textbf{minimally-warranting} iff it is a warranting incision function and \ojo{for all $\Brg\in\ARGS$ such that $\incision(\Brg)\neq\emptyset$ it holds $|\incision(\Brg)|=1$}.
\end{Definition}

\begin{Theorem}\label{theo.minimally-warranting.incision}
Given a \delp\ \PP\ and a warranting incision function ``\incision'', \ojo{if \incise is minimally-warranting} then $\htree(\Arg, \Psi\setminus\{\varphi\})$ is non-warranting, for any $\varphi\in\Psi$, where $\Psi= \Incisions(\aware{\Arg}{\PP})$.
\end{Theorem}
\begin{proof}
From the hypothesis we know that, $\varphi\in\incise(\deactsel(\lambda))$, for some $\lambda\in\aware{\Arg}{\PP}$; and since $|\incise(\deactsel(\lambda))|=1$ (see Def.~\ref{def.mwIncision}), it holds $\incise(\deactsel(\lambda))=\{\varphi\}$. 
Since $\aware{\Arg}{\PP}$ is minimal (see cond.~\ref{item.aware.4} in Def.~\ref{def.att2}), it includes $\lambda$ necessarily because it is the unique line which through $\incise(\deactsel(\lambda))$ closes another line $\lambda'\in\openContext(\aware{\Arg}{\PP})$. Thus, $\lambda'\in\closed(\lambda)$ and hence, $\lambda'\notin\closedContext(X)$, where $X=\aware{\Arg}{\PP}\setminus\{\lambda\}$. (If another line $\lambda''\in X$ were closing $\lambda'$, this would determine $\lambda'\in\closedContext(X)$, and thus minimality of $\aware{\Arg}{\PP}$ would be violated.) It is clear that $\lambda'\in\openContext(X)$, since $\lambda$ closes $\lambda'$, thus $\lambda'$ cannot be simultaneously open by $\lambda$. Afterwards, $\openContext(X)\not\subseteq\closedContext(X)$. This means that $\lambda'$ has not been effectively altered in $\htree(\Arg, \Psi\setminus\{\varphi\})$. And thus, since it contains an attacking line, $\htree(\Arg, \Psi\setminus\{\varphi\})$ is non-warranting.
\end{proof}

\begin{Corollary}
Given a \delp\ \PP\ and a warranting incision function ``\incision'', \ojo{if \incise is minimally-warranting} then $\htree(\Arg, \Psi)$ is non-warranting, where $\Psi= \Incisions(X)$, for any $X\subset\aware{\Arg}{\PP}$.
\end{Corollary}

The theorem above ensures ``some kind of'' minimality regarding the set of rules to be removed from a \delp\ However, real minimality regarding removals from a \delp\ cannot be achieved without compromising the first axis of change (selection criterion) as being described on page~\pageref{axes.of.change}. 
\ojo{Moreover, by concretizing the alteration criterion \lessAlter (see Def.~\ref{def.alteration.criterion}) considering that minimality is determined according to set cardinality (see example given on page~\pageref{alteration.criterion.cardinality}), the attacking set \ALINESPP ends up being the smallest possible set satisfying Def.~\ref{def.alines.delp}, \ie there is no set $X\neq\ALINESPP$ satisfying Def.~\ref{def.alines.delp} such that $|X|<|\ALINESPP|$. By ensuring such a concretization for the alteration criterion we still cannot ensure that there is no set $\Psi'\subseteq\PP$ such that $\htree(\Arg,\Psi')$ is warranting and $|\Psi'|<|\Psi|$, where $\Psi= \Incisions(\aware{\Arg}{\PP})$.} A set like $\Psi'$ can be obtained by deactivating only direct defeaters of the root argument (this option will be discussed in Section~\ref{sec.principles}), excluding any possible selection criterion from consideration and thus compromising the first axis of change. 
This is illustrated in Ex.~\ref{ex.minimality.DeLP}. A fourth axis of change could take into account the amount of rules to be removed from a \delp\ Nonetheless, this decision would complicate even more our theory. In this paper we abstracted away from such consideration, and concentrated only on the analysis of the impact of change over the original dialectical tree by proposing three axes of change as aforementioned.

\begin{Example}\label{ex.minimality.DeLP}
Consider the tree $\dtree{\Arg}{\PP}\in\accTrees{\PP}$ depicted on the right, where \PP
\begin{window}[0,r,{\mbox{\epsfig{file=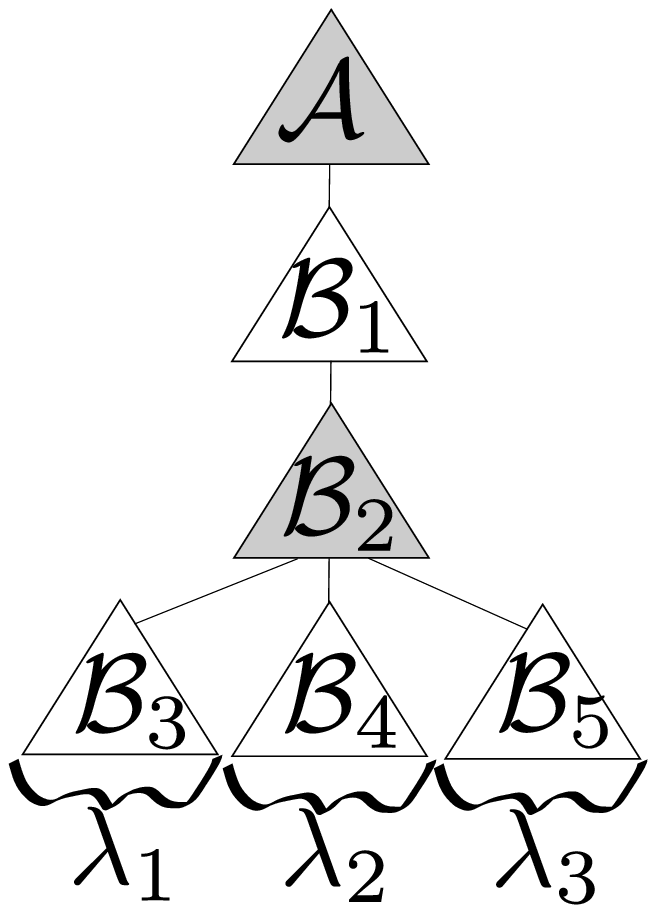,scale=.32}}},{}]
\noindent is a \delp\ Let us assume a selection criterion such that $\Brg_3\selcritsub{1}\Brg_1$, $\Brg_4\selcritsub{2}\Brg_1$, and $\Brg_5\selcritsub{3}\Brg_1$; and the support of arguments $\Brg_1$, $\Brg_3$, $\Brg_4$, and $\Brg_5$ to be singletons where $\varphi_1\in\Brg_1$, $\varphi_3\in\Brg_3$, $\varphi_4\in\Brg_4$, and $\varphi_5\in\Brg_5$. Clearly, $\aware{\Arg}{\PP}=\{\lambda_1,\lambda_2,\lambda_3\}$, and $\incise(\lambda_1)=\varphi_3$, $\incise(\lambda_2)=\varphi_4$, and $\incise(\lambda_3)=\varphi_5$. Hence, the hypothetical tree $\htree(\Arg,\Psi)$, with $\Psi=\{\varphi_3,\varphi_4,\varphi_5\}$, is warranting. Note however that a model of change disregarding not only the selection criterion, but the three axes of change, could choose a set $\Psi'=\{\varphi_1\}$ rendering a warranting hypothetical tree $\htree(\Arg,\Psi')$, and moreover, satisfying $|\Psi'|<|\Psi|$. By taking into account $\Psi'$, the resulting \delp\ would only lose a single defeasible rule. Nonetheless, the resulting dialectical tree would be smaller than the one resulting from our model of change. This is important when considering dialectical trees as explanations for the status of the root argument, as analyzed in~\cite{explanations07}.
\end{window}
\end{Example}

\subsection{Interrelating Attacking and Alteration Sets}\label{sec.atc.set.interrelation}


In the rest of this section we will study the relation among the attacking set, and the regular and incision-aware versions of the alteration set. In general, the studied properties pose restrictions upon the worked incision function such that, when satisfied they ensure $\ALINESPP=\alteration{\Arg}{\PP}$ (Theorem~\ref{theo.alines.equals.alteration.1}, Corollary~\ref{corollary.profit.then.alines.equals.alteration}, and Theorem~\ref{theo.weak.profitability}), $\alteration{\Arg}{\PP}=\aware{\Arg}{\PP}$ (Theorem~\ref{theo.aware.equals.alteration}), or even $\Incisions(\ALINESPP)=\Incisions(\alteration{\Arg}{\PP})=\Incisions(\aware{\Arg}{\PP})$ (Theorem~\ref{theo.equal.incisions}). These properties are important in the development of the optimized algorithms given in Section~\ref{sec.principles} for constructing the proposed argumentative model of change.

\ojo{We firstly state under which conditions the attacking set and the regular alteration set, coincide. Theorem~\ref{theo.alines.equals.alteration.1} shows that this happens when the lines in \ALINESPP only provoke collateral alterations (if any) over lines within \ALINESPP.} Since the same condition is part of the requirements of the profitability principle, Corollary~\ref{corollary.profit.then.alines.equals.alteration} follows afterwards. On the other hand, if every collaterality incises only selected arguments from other lines, then no new lines will be open. 
This property fulfills the requirements of an incision function satisfying weak profitability. Thus, Theorem~\ref{theo.weak.profitability}, shows $\ALINESPP=\alteration{\Arg}{\PP}$ holds if weak profitability is satisfied.

\begin{Theorem}\label{theo.alines.equals.alteration.1}
Given a warranting incision function ``$\incise$'', for any $\lambda'\in\ALINESPP$, any $\lambda\in\dtree{\Arg}{\PP}$, and any $\Brg\in\lambda$; $[$if $\colinc{\incise(\deactsel(\lambda'))}{\Brg}\neq\emptyset$ then $\lambda\in\ALINESPP]$ \ifff $\ALINESPP=\alteration{\Arg}{\PP}$.
\end{Theorem}
\begin{proof}
$\Rightarrow)$ From Remark~\ref{remark.alteration.contains.attset}, we know that $\ALINESPP\subseteq\alteration{\Arg}{\PP}$, hence we need to show that $\alteration{\Arg}{\PP}\subseteq\ALINESPP$. This is equivalent to show $\attFunct^0=\attFunct^1$ from Def.~\ref{def.attset}, which is equivalent to show $\open(\lambda')\subseteq\ALINESPP$, for any $\lambda'\in\ALINESPP$. Observe that this follows straightforwardly from the hypothesis. Hence, $\ALINESPP=\alteration{\Arg}{\PP}$ holds.

\ojo{$\Leftarrow)$ 
Analogously, assuming $\ALINESPP=\alteration{\Arg}{\PP}$ holds, implies $\attFunct^0=\attFunct^1$ from Def.~\ref{def.attset}, and thus also $\open(\lambda')\subseteq\ALINESPP$, for any $\lambda'\in\ALINESPP$. This means that the effective alteration of any attacking line will not open a non-attacking line, \ie no non-attacking line is turned into attacking. Finally, it is easy to see that the condition if $\colinc{\incise(\deactsel(\lambda'))}{\Brg}\neq\emptyset$ then $\lambda\in\ALINESPP$ holds.}
\end{proof}

\begin{Corollary}\label{corollary.profit.then.alines.equals.alteration}
If profitability is satisfied then $\ALINESPP=\alteration{\Arg}{\PP}$.
\end{Corollary}

\begin{Theorem}\label{theo.weak.profitability}
If weak-profitability is satisfied then $\ALINESPP=\alteration{\Arg}{\PP}$.
\end{Theorem}
\begin{proof}
Since weak-profitability holds, we have that, if $\colinc{\incise(\deactsel(\lambda'))}{\Brg}\neq\emptyset$ then $\deactsel(\lambda) = \Brg$, for every $\lambda\in\dtree{\Arg}{\PP}$, $\lambda'\in\dtree{\Arg}{\PP}$, and $\Brg\in\lambda$. This means that no collaterality will open any line, \ie for every $\lambda'\in\dtree{\Arg}{\PP}$, it follows $\open(\lambda')=\emptyset$ holds \ojo{--meaning that no collaterality will turn a non-attacking line into attacking.} Hence, it also holds $\attFunct^0=\attFunct^1$, from Def.~\ref{def.attset}. Finally, $\ALINESPP=\alteration{\Arg}{\PP}$.
\end{proof}

\begin{Remark}
If $\ALINESPP=\alteration{\Arg}{\PP}$ then $\aware{\Arg}{\PP}\subseteq\ALINESPP$.
\end{Remark}

\begin{Proposition}
Given a \delp\ \PP, a warranting incision function ``\incision'', and the sets of defeasible rules $\Psi_1=\Incisions(\alteration{\Arg}{\PP})$ and $\Psi_2 = \Incisions(\aware{\Arg}{\PP})$, it holds $\Psi_2\subseteq \Psi_1$
\end{Proposition}
\begin{proof}
Straightforward from condition~\ref{item.aware.1} in Def.~\ref{def.att2}.
\end{proof}

The construction of the incision-aware alteration set $\aware{\Arg}{\PP}$ may be skipped if the preconditions of Theorem~\ref{theo.aware.equals.alteration} are satisfied in the construction of the regular alteration set $\alteration{\Arg}{\PP}$. That is, $\alteration{\Arg}{\PP}=\aware{\Arg}{\PP}$ holds whenever no line in $\alteration{\Arg}{\PP}$ is closed by another line in the set.

\begin{Theorem}\label{theo.aware.equals.alteration}
$\alteration{\Arg}{\PP}=\aware{\Arg}{\PP}$ \ifff $\nexists\lambda\in \alteration{\Arg}{\PP}$ such that $\lambda\in\closedContext(\alteration{\Arg}{\PP}\setminus\{\lambda\})$.
\end{Theorem}
\begin{proof}
$\Rightarrow)$ By \textit{reductio ad absurdum} assume there is some $\lambda\in \alteration{\Arg}{\PP}$ such that $\lambda\in\closedContext(\alteration{\Arg}{\PP}\setminus\{\lambda\})$. 
Since $\alteration{\Arg}{\PP}=\aware{\Arg}{\PP}$, \ojo{from Def.~\ref{def.att2}} it is clear that $\lambda\in\openContext(\aware{\Arg}{\PP})$. From Prop.~\ref{prop.non.identity.open} and Def.~\ref{def.context-sensitive.collaterality.functions}, we know that $\lambda\in\openContext(\aware{\Arg}{\PP}\setminus\{\lambda\})$. Afterwards, $\openContext(\aware{\Arg}{\PP}\setminus\{\lambda\})\subseteq\closedContext(\aware{\Arg}{\PP}\setminus\{\lambda\})$ holds, satisfying condition~\ref{item.aware.2} from Def.~\ref{def.att2}. Thus, since condition~\ref{item.aware.4} ends up violated, the set $\aware{\Arg}{\PP}$ does not conform Def.~\ref{def.att2}, reaching the absurdity.

$\Leftarrow)$ We have that for every $\lambda\in \alteration{\Arg}{\PP}$, it holds $\lambda\notin\closedContext(\alteration{\Arg}{\PP}\setminus\{\lambda\})$. From Def.~\ref{def.att2}, we know $\aware{\Arg}{\PP}\subseteq\alteration{\Arg}{\PP}$, thus we need to show that $\alteration{\Arg}{\PP}\subseteq\aware{\Arg}{\PP}$. By \textit{reductio ad absurdum}, we assume there is some $\lambda\in\alteration{\Arg}{\PP}$ such that $\lambda\notin\aware{\Arg}{\PP}$. But then, from Def.~\ref{def.att2}, we know $\lambda\in\closedContext(\alteration{\Arg}{\PP}\setminus\{\lambda\})$, which is absurd.
\end{proof}

\begin{Lemma}\label{lemma.singleton.closed}
For any $\lambda\in\alteration{\Arg}{\PP}$ and any $\lambda'\in\alteration{\Arg}{\PP}$, if $[\lambda'\in\closed(\lambda)]\rightarrow[\incise(\deactsel(\lambda))=\incise(\deactsel(\lambda'))]$ and $\lambda\in\aware{\Arg}{\PP}$ then $\closed(\lambda)\cap\aware{\Arg}{\PP}=\{\lambda\}$. 
\end{Lemma}
\begin{proof}
From Prop.~\ref{prop.identity.closed}, we know $\lambda\in\closed(\lambda)$, thus $\lambda\in(\closed(\lambda)\cap\aware{\Arg}{\PP})$. By \textit{reductio ad absurdum}, let us assume there is some line $\lambda'\in(\closed(\lambda)\cap\aware{\Arg}{\PP})$, such that $\lambda\neq\lambda'$. Since $\lambda\in\aware{\Arg}{\PP}$, from Def.~\ref{def.att2}, $\lambda\in\openContext(\aware{\Arg}{\PP})$ and $\lambda\in\closedContext(\aware{\Arg}{\PP})$. We also have that $\lambda'\in\closedContext(\aware{\Arg}{\PP})$. From hypothesis we know that $\incise(\deactsel(\lambda))=\incise(\deactsel(\lambda'))$ holds. It is clear that every line that is open (resp., closed) by $\lambda'$ is also open (resp., closed) by $\lambda$. 
Hence, $\closedContext(\aware{\Arg}{\PP})=\closedContext(\aware{\Arg}{\PP}\setminus\{\lambda'\})$, and if $\openContext(\aware{\Arg}{\PP})\subseteq\closedContext(\aware{\Arg}{\PP})$ holds so it does $\openContext(\aware{\Arg}{\PP}\setminus\{\lambda'\})\subseteq\closedContext(\aware{\Arg}{\PP}\setminus\{\lambda'\})$. Since this is contrary to condition~\ref{item.aware.4} from Def.~\ref{def.att2}, we reach an absurdity. Finally $\closed(\lambda)\cap\aware{\Arg}{\PP}=\{\lambda\}$.
\end{proof}

The following theorem states under which conditions the attacking set, alteration set, and incision-aware alteration set, determine the same sets of rules to be removed. That is, $\Incisions(\ALINESPP)=\Incisions(\alteration{\Arg}{\PP})=\Incisions(\aware{\Arg}{\PP})$ holds when both conditions \ref{theo.cond.1} and \ref{theo.cond.2} from Theorem~\ref{theo.equal.incisions} are satisfied. Condition~\ref{theo.cond.1} states that if there is a line in $\alteration{\Arg}{\PP}$ closed through a regular collaterality function (Def.~\ref{def.collaterality.functions}) by another line in $\alteration{\Arg}{\PP}$, then their incisions coincide. Observe that this is part of the preconditions required in Lemma~\ref{lemma.singleton.closed}. On the other hand, condition~\ref{theo.cond.2}, states that if a line in $\alteration{\Arg}{\PP}$ is closed through a context-sensitive collaterality function (Def.~\ref{def.context-sensitive.collaterality.functions}) by another line in $\alteration{\Arg}{\PP}$, then it is necessarily closed by a regular collaterality function. Thus, the collaterality occurs over the selected argument in that line.

\begin{Theorem}\label{theo.equal.incisions}
Given the following two conditions:
\vspace{-7mm}\begin{description}
\item\indent
\begin{enumerate}
\item $\forall\lambda'\in\alteration{\Arg}{\PP},\forall\lambda\in\alteration{\Arg}{\PP}$; if $\lambda'\in\closed(\lambda)$ then $\incise(\deactsel(\lambda'))=\incise(\deactsel(\lambda))$\label{theo.cond.1}
\item $\forall\lambda'\in\alteration{\Arg}{\PP},\exists\lambda\in\alteration{\Arg}{\PP}$; if $\lambda'\in\closedContext(\alteration{\Arg}{\PP}\setminus\{\lambda'\})$ then $\lambda'\in\closed(\lambda)$ and $\lambda\neq\lambda'$\label{theo.cond.2}
\end{enumerate}
\end{description}
If both \ref{theo.cond.1} and \ref{theo.cond.2} hold then $\Incisions(\ALINESPP)=\Incisions(\alteration{\Arg}{\PP})=\Incisions(\aware{\Arg}{\PP})$.
\end{Theorem}
\begin{proof}
From Def.~\ref{def.att2}, we know $\aware{\Arg}{\PP}\subseteq\alteration{\Arg}{\PP}$, and hence, it is easy to see that $\Incisions(\aware{\Arg}{\PP})\subseteq\Incisions(\alteration{\Arg}{\PP})$ holds. Thus, we need to show that $\Incisions(\alteration{\Arg}{\PP})\subseteq\Incisions(\aware{\Arg}{\PP})$ also holds. We will assume (a) there is some $\lambda\in\alteration{\Arg}{\PP}$ and some $\lambda'\in\alteration{\Arg}{\PP}$ such that $\lambda'\in\closed(\lambda)$ and $\lambda\neq\lambda'$. 
From Lemma~\ref{lemma.singleton.closed}, we know that if $\lambda\in\aware{\Arg}{\PP}$ then $\lambda'\notin\aware{\Arg}{\PP}$ holds. However, since $\incise(\deactsel(\lambda'))=\incise(\deactsel(\lambda))$, we know that  $\incise(\deactsel(\lambda'))\subseteq\Incisions(\aware{\Arg}{\PP})$. On the other hand, if $\lambda\notin\aware{\Arg}{\PP}$ and $\lambda'\notin\aware{\Arg}{\PP}$ then from cond.~\ref{theo.cond.2} we know for every $\lambda''\in\alteration{\Arg}{\PP}$ there is some $\lambda'''\in\alteration{\Arg}{\PP}$ such that 
if $\lambda''\in\closedContext(\alteration{\Arg}{\PP}\setminus\{\lambda''\})$ then $\lambda''\in\closed(\lambda''')$ and $\lambda''\neq\lambda'''$. But then, it follows that either (1) $\lambda''\notin\closedContext(\alteration{\Arg}{\PP}\setminus\{\lambda''\})$ or (2) $\lambda''\notin\aware{\Arg}{\PP}$ and $\lambda'''\notin\aware{\Arg}{\PP}$. 
For the latter case, observe that for every $\lambda\in\aware{\Arg}{\PP}$ it follows that  $\lambda\notin\closedContext(\aware{\Arg}{\PP}\setminus\{\lambda\})$ and hence $\aware{\Arg}{\PP}=\alteration{\Arg}{\PP}$ (see Theorem~\ref{theo.aware.equals.alteration}). The former case is similar, and in case (a) is not satisfied, then we will also be satisfying $\aware{\Arg}{\PP}=\alteration{\Arg}{\PP}$. Finally, $\Incisions(\alteration{\Arg}{\PP})=\Incisions(\aware{\Arg}{\PP})$ holds.

Besides, since cond.~\ref{theo.cond.1} conforms the preconditions of Theorem~\ref{theo.weak.profitability}, $\ALINESPP=\alteration{\Arg}{\PP}$ holds, and hence, $\Incisions(\ALINESPP)=\Incisions(\alteration{\Arg}{\PP})$ is also satisfied.
\end{proof}

\begin{Example}
Considering Ex.~\ref{ex.minimality.DeLP}, and assuming $\varphi_3=\varphi_4=\varphi_5$, the conditions of Theorem~\ref{theo.equal.incisions} are satisfied. Observe that $\incise(\deactsel(\lambda_1))=\incise(\deactsel(\lambda_2))=\incise(\deactsel(\lambda_3))$. It is clear that $\ALINESPP=\alteration{\Arg}{\PP}$, and that $\aware{\Arg}{\PP}$ may be any singleton containing either $\lambda_1$, $\lambda_2$, or $\lambda_3$. As stated by Theorem~\ref{theo.equal.incisions}, the set of rules to remove from the \delp\ would be the same for any of the sets, $\ALINESPP$, $\alteration{\Arg}{\PP}$, or $\aware{\Arg}{\PP}$.
\end{Example}


As a consequence of the properties shown, from now on we will only rely on the incision-aware alteration set to formalize the upcoming change operations.

\subsection{Argument Change Operators}\label{sec.operators}

The argument expansion can be defined in a simple manner by just adding the necessary rules to activate the desired argument; formally:

\begin{Definition}[Argument Expansion]
An \textbf{argument expansion operation} $\PP\ \argExp \Arg$ over a \delp\ $\PP = \SD$ by an argument \Arg from either \ARGS or \EXT, is defined as follows:
\begin{center}
$\PP\ \argExp\Arg = (\Pi, \DD\cup\Arg)$
\end{center}
\end{Definition}

Note that not only argument $\Arg$ is activated, but the addition of \Arg's rules to $\Delta$ could cause the automatic activation of many other arguments. This is part of the dynamism of the theory. Moreover, the definition of the argument expansion has the inherent implications to expansions within any non-monotonic formalism: despite of the set of arguments $\ARGSP{\scriptsize(\PP\argExp\Arg)}$ being increased, the amount of warranted consequences from $\PP\ \argExp\Arg$ could be diminished.

Regarding contractions, we are looking for an operator that provides warrant for an argument $\Arg\in\ARGS$ by turning every attacking line in \dtree{\Arg}{\PP} to a non-attacking line through an argument incision function $\incise$. That is, we are going to drop arguments towards \Arg's warrant. This is the reason why we call it \emph{argument defeating contraction}. Considering that the notion of consequence is warrant, we are taking advantage of the non-monotonic nature of argumentation.

\begin{Definition}[Argument Defeating Contraction]
An \textbf{argument defeating contraction operation} $\PP\ \OPnwaccc \Arg$ of a \delp\ $\PP = \SD$ by an argument $\Arg\in\ARGS$, is defined by means of a minimally-warranting incision function ``$\incise$'' applied over selections
$\deactsel(\lambda)$ for each $\lambda\in\aware{\Arg}{\PP}$ in the incision-aware alteration set of $\dtree{\Arg}{\PP}$, as follows:
\[\PP\OPnwaccc\Arg = (\Pi,\DD\setminus\Incisions(\aware{\Arg}{\PP}))\]
\end{Definition}




An argument revision operator should firstly add to the program the new argument for which warrant is to be achieved. Afterwards, a warrant contraction should be applied. Note that in case the argument was already warranted, the contraction would produce no change since the alteration set would be empty. The operation is called \emph{argument warranting revision}.

\begin{Definition}[Argument Warranting Revision]\label{def.rev}
Given a DeLP program $\PP = (\Pi,\Delta)$, and an argument \Arg from either \ARGS or \EXT, an operator ``\OPwparp'' is an \textbf{argument warranting revision} \ifff
\[\PP\OPwparp\Arg = (\Pi,\Delta'\setminus\Incisions(\aware{\Arg}{\PP'})),\]
where ``\incise'' is a minimally-warranting incision, $\aware{\Arg}{\PP'}$ is the incision-aware alteration set of $\dtree{\Arg}{\PP'}$, and $\PP' = (\Pi,\Delta')$ with  $\Delta' = \Delta\cup\Arg$.
\end{Definition}

In belief revision, revisions and contractions may be defined one in terms of the other by means of the \textit{Levy identity}~\cite{Lev77}. In this model of change, Definition~\ref{def.rev} can be rewritten in terms of an argument expansion and a defeating contraction as an analogy of the \textit{reversed Levi identity}~\cite{HAN05}, which we have called the \textit{argument change identity}.

\begin{center}
\textbf{(Argument Change Identity)} \ $\PP\ \OPwparp \Arg = (\PP\ \argExp \Arg)\ \OPnwaccc \Arg$
\end{center}

In this revision the expansion has to be performed first because otherwise there would be no argument to warrant. Besides, inconsistent intermediate states are not an issue in this formalism, since it is based on argumentation.

Given a knowledge base \PP and an argument \Arg, the next postulates stand for the principles of inclusion, success, and minimal change, for an argument revision operator ``$*$'' based on alteration of dialectical trees such as ATC. 
A complete study about postulates in ATC can be referred to \cite{jigpal}.

\begin{description}\label{postulates}
\item[(inclusion)] $\PP*\Arg\subseteq\PP\cup\Arg$.
\item[(success)] \Arg is warranted from $\PP*\Arg$.
\item[(core-retainment)] If $\varphi\in(\PP\setminus \PP*\Arg)$ then there is some $\PP'\subseteq\PP$ such that \Arg is warranted from $\PP'\cup\Arg$ but not from $\PP'\cup\Arg\cup\{\varphi\}$.\label{core-ret}
\end{description}

Inclusion aims at guaranteeing that no other new information beyond the one conforming argument \Arg will be included to the \delp\ 
Success states that the new information to be incorporated should be accepted by the worked argumentation semantics, \ie the new argument should end up warranted. 
Core-retainment was originally introduced in \cite{HanssonRecovery} and then it was adapted for revision in \cite{HAN97,LocalChangeHW02}. 
Through this postulate, the amount of change is controlled by avoiding removals that are not related to the revision, \ie every rule $\varphi$ lost serves to the acceptation of the new argument. This means that $\varphi$ is removed in order to achieve an effective alteration.

%

Assuming the knowledge base \PP as a \delp, the argument \Arg corresponding to either \ARGS or \EXT, and associating the abstract argument revision operator ``$*$'' as the one given in Definition~\ref{def.rev}; the proposed argument revision operator ``$\OPwparp$'' upon \delp s is shown to satisfy the given postulates.


\begin{Theorem}\label{theo.construction.to.post}
Given a \delp\ $\PP = \SD$ and the external argument $\Arg\in\EXT$, if ``\OPwparp'' is an argument revision operator then $\PP\OPwparp\Arg$ satisfies inclusion, success, and core-retainment.
\end{Theorem}
\begin{proof}
We know ``\incise'' is minimally-warranting (Def.~\ref{def.rev}) thus satisfying preservation (Def.~\ref{def.mwIncision} and  Def.~\ref{def.warrantingIncision}) and thus strict preservation (see Prop.~\ref{prop.preservation}). Hence, inclusion is satisfied. 
Success and core-retainment follow from Theorem~\ref{theo.inc-aware.altset.warranting} and Theorem~\ref{theo.minimally-warranting.incision}.
\end{proof}


%
%
%
%
%

\section{Towards an Implementation for ATC}
\label{sec.principles}

In this section we present several examples of minimal change criteria, and afterwards introduce a \prolog-like algorithm that illustrates an implementation of the argument revision operator, as defined in Section~\ref{sec.atc}.

\subsection{Minimal Change Criteria Exemplified}

We have defined how ATC relies on the minimal change principle, which specifies the way change is evaluated. Following this principle, particular minimal change criteria can be developed in order to establish a specific way of measuring change. In this section we propose some of these criteria. Additionally, we will address the third axis of change mentioned in Section~\ref{sec.atc}, by considering restrictions over the relation between selections and incisions for each proposed criterion. These restrictions make use of the properties previously defined in this article, like cautiousness, (weak) profitability and strict preservation. The attachment of a restriction to each criterion represents just an example, and is not intended to be formal nor definitive. That is, additional restrictions could be posed to achieve the desired behavior for these criteria, as well as none.



\subsubsection*{Preserving Program Rules}

In general, when looking to remove as few rules as possible, selecting direct defeaters of the root argument ensures a minimal deletion of defeasible rules from the \delp\ This is so because the deletion of a root's defeater eliminates a whole subtree. Trying to achieve the same result by deleting rules from ``lower'' arguments in the tree would affect a greater amount of arguments, due to possible branching. We will make incisions only over those direct defeaters that are undefeated, \ie those belonging to attacking lines, since the ones that are defeated do not compromise the warrant of the root.

\emph{\textbf{Rules-Preserving Selection Criterion.} Given a line $\lambda\in\bundle$ where \bundle determines \dtree{\Arg}{\PP}, $\selcrit=\{(\Brg_1,\Brg_2)\ |\ \Brg_1\in\interf, \Brg_2\in\interf,$ and $\Brg_1\in\upsegm{\lambda}{\Brg_2}\}$}

An interesting \emph{restriction} is to seek for profitable incisions, \ie those that have a collateral incision with a selection in another line. Such an incision is desirable, since it would not only save a future incision, but would also collaborate with the criterion by preventing the deletion of extra rules.

Note that, in pursuit of profitability, this criterion could be relaxed to allow selections to be mapped to arguments placed at lower positions in the line. This is performed by updating the order, as shown before. However, the question remains about how much effort should be put on this re-ordering. Should we go for the best combination (select as high as possible capturing shared incisions) risking to end up deleting more rules? Or should there be specific boundaries beyond which dropping the search ends up being worthier? Since computational tractability is also at stake, a definite answer remains a matter of implementation.

\subsubsection*{Preserving the Dialectical Tree Structure}

When trees are treated as an explanation for the answer given to a query \cite{explanations07}, they are of utmost importance, since their structure turns out to be the main source of information. Provided that dialectical trees are the most suitable tool to trust and understand the interrelation among arguments and their influence to the final answer, we will define a selection criterion that determines a revision operation making minimal changes in the structure of the temporary tree \ojo{(recall this notion was introduced on page~\pageref{temporaryTree})} in order to render its root undefeated. Therefore, like in Ex.~\ref{ex.selection.main.issue}, the selection criterion will be determined by the level of the argument in the argumentation line; the lower an argument is, the less is its impact in the structure of the tree, making the argument more suitable for selection. Hence, this criterion specifies the opposite order than the rules-preserving one. 
In this case, an interesting restriction would be to identify those strict-preserving incisions, that is, incisions that do not collide with any other argument in the tree. This would collaborate with the preservation of the tree structure. Again, strict preservation should not be sought blindly, at the expense of the original ordering specified by the selection criterion, but some sort of balance must be pursued instead.

\subsubsection*{Preserving Rules Without Compromising Tree Structure}

Following the two principles given above, a combined approach can be studied in order to preserve the rules of the program while minimizing the pruning of the dialectical tree. This approach takes advantage of adjacency among attacking lines, incising one of the arguments acting as a ``common factor'' for them, \ie belonging to the shared upper segment. In this way, the deactivation of fewer arguments is encouraged. The principle also attempts to go as deep as possible when selecting arguments. Thus, the amount of arguments disappearing from prunes in the tree is diminished. Ex.~\ref{ex:principles2} illustrates the usage of this criterion.

\emph{\textbf{Tree-and-Rules-Preserving Selection Criterion.} Given the dialectical tree \dtree{\Arg}{\PP} and two lines $\lambda_1\in\ALINESPP$ and $\lambda_2\in\ALINESPP$ belonging to its attacking set; if $\lambda_1$ and $\lambda_2$ are adjacent at an argument $\Brg$ such that there is no adjacency point in $\upsegm{\lambda_1}{\Brg}$ shared with a line in \ALINESPP, then \ojo{$\selcritsub{1}=\selcritsub{2}=\{(\Brg_1,\Brg_2)\ |\ \Brg_1\in\interf_1, \Brg_2\in\interf_1, \Brg_1\in\upsegmeq{1}{\Brg},$ and $\Brg_2\in\upsegm{\lambda_1}{\Brg_1}\}$}}

\subsubsection*{Preserving Semantics}

In addition to the principles given above, we could consider to produce the least possible modifications to the semantics of the defeasible logic program. The set of warranted arguments would be preserved at the highest possible degree, while satisfying some minimal change criterion. A way to implement this would be not consider warranted arguments as candidates for deactivation. It might be the case that the deactivation of some arguments could be unavoidable. In such a case, there would be a compromise between the chosen criterion and the preservation of semantics, leading to an update of the selection criterion.
For instance, if we attempt to preserve the structure of the tree while not harming the set of warranted arguments, we define:

\emph{\textbf{Semantics-Preserving Selection Criterion.} Given a line $\lambda\in\bundle$ where \bundle determines \dtree{\Arg}{\PP}, $\selcrit=\{(\Brg_1,\Brg_2)\ |\ \Brg_1\in\interf, \Brg_2\in\interf,$ and $\Brg_2\in\upsegm{\lambda}{\Brg_1}$ and neither $\Brg_1$ nor $\Brg_2$ are warranted from $\PP\}$}


\begin{Example}\label{ex:principles1}\
\begin{window}[0,r,{\mbox{\epsfig{file=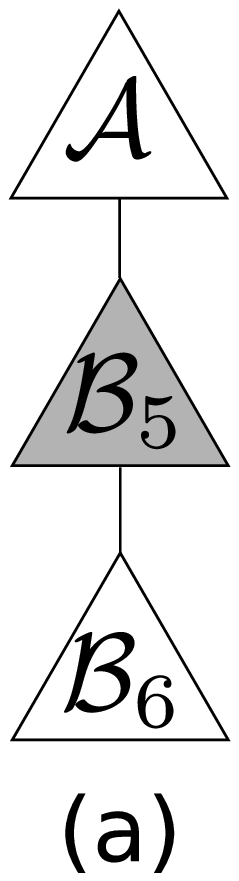,scale=.32}}},{}]
Consider the program $\PP_{\ref{ex:delp}}$ being revised by argument \Arg\ and the corresponding temporary tree of Ex.~\ref{ex:tree}. The criterion trying to \textbf{preserve program rules} would select arguments $\Brg_1$ and $\Brg_2$. From Ex.~\ref{ex:cautious}, we know that there is a way of incising $\Brg_1$ while collaterally incising $\Brg_2$, which is $\incise(\Brg_1) = \{\drule{\no a}{y}\}$. Therefore, the resulting tree is as depicted on the right, and the revised program would lose just one rule: $\PP^1_R = \pair{\SSet_{\ref{ex:delp}}}{\DD_{\ref{ex:delp}}\cup\{\Arg\}\setminus\{\drule{\no a}{y}\}}$.
%
\end{window}
\begin{window}[0,r,{\mbox{\epsfig{file=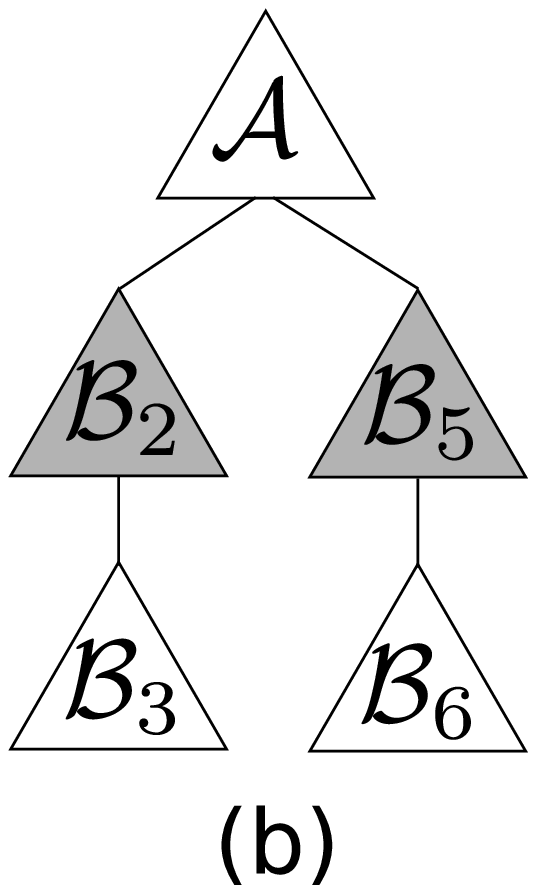,scale=.32}}},{}]
Following the \textbf{tree-preserving} minimal change principle, lower selections are considered first, thus the selected arguments are $\Brg_1$ and $\Brg_4$. Now the incision over $\Brg_1$ will avoid collateral incisions, \ie will be strict-preserving; hence, $\incise(\Brg_1) = \{\drule{y}{x}\}$. Since $\Brg_4$ is a cautious selection (see Ex.~\ref{ex:cautious}) and has one rule, the only possible incision is: $\incise(\Brg_4) = \{\drule{\no w}{t}\}$. Finally, the resulting tree is as depicted on the right, and its corresponding program is: $\PP^2_R = \pair{\SSet_{\ref{ex:delp}}}{\DD_{\ref{ex:delp}}\cup\{\Arg\}\setminus\{(\drule{y}{x}),(\drule{\no w}{t})\}}$.
%
%
\end{window}
%
\end{Example}

\begin{Example}
\label{ex:principles2}
Let consider a modification of the program $\PP_{\ref{ex:delp}}$ used in Ex.~\ref{ex:principles1}:

\begin{center}
$\PP_{\ref{ex:principles2}} = (\SSet_{\ref{ex:delp}}, \DD_{\ref{ex:delp}}\cup\left\{\begin{array}{c}
(\drule{a}{x}),(\drule{x}{z}),\\
(\drule{b}{\no a}),(\drule{\no a}{p}),\\
(\drule{\no b}{t}),(\drule{b}{z})
\end{array}
\right\})$
\end{center}

If we revise $\PP_{\ref{ex:principles2}}$ by $\Arg = \Ar{\{\drule{\no b}{p}\}}{\no b}$, we can build the temporary dialectical tree depicted below, annotated with the defeasible rules used in each argument. Following the \textbf{rules-preserving} criterion, the argument to be incised would be $\Brg_1 = \Ar{\{\drule{b}{z}\}}{b}$, that is, the revision would consist of adding the rule from \Arg and deleting the single rule from $\Brg_1$.

\begin{window}[0,r,{\mbox{\epsfig{file=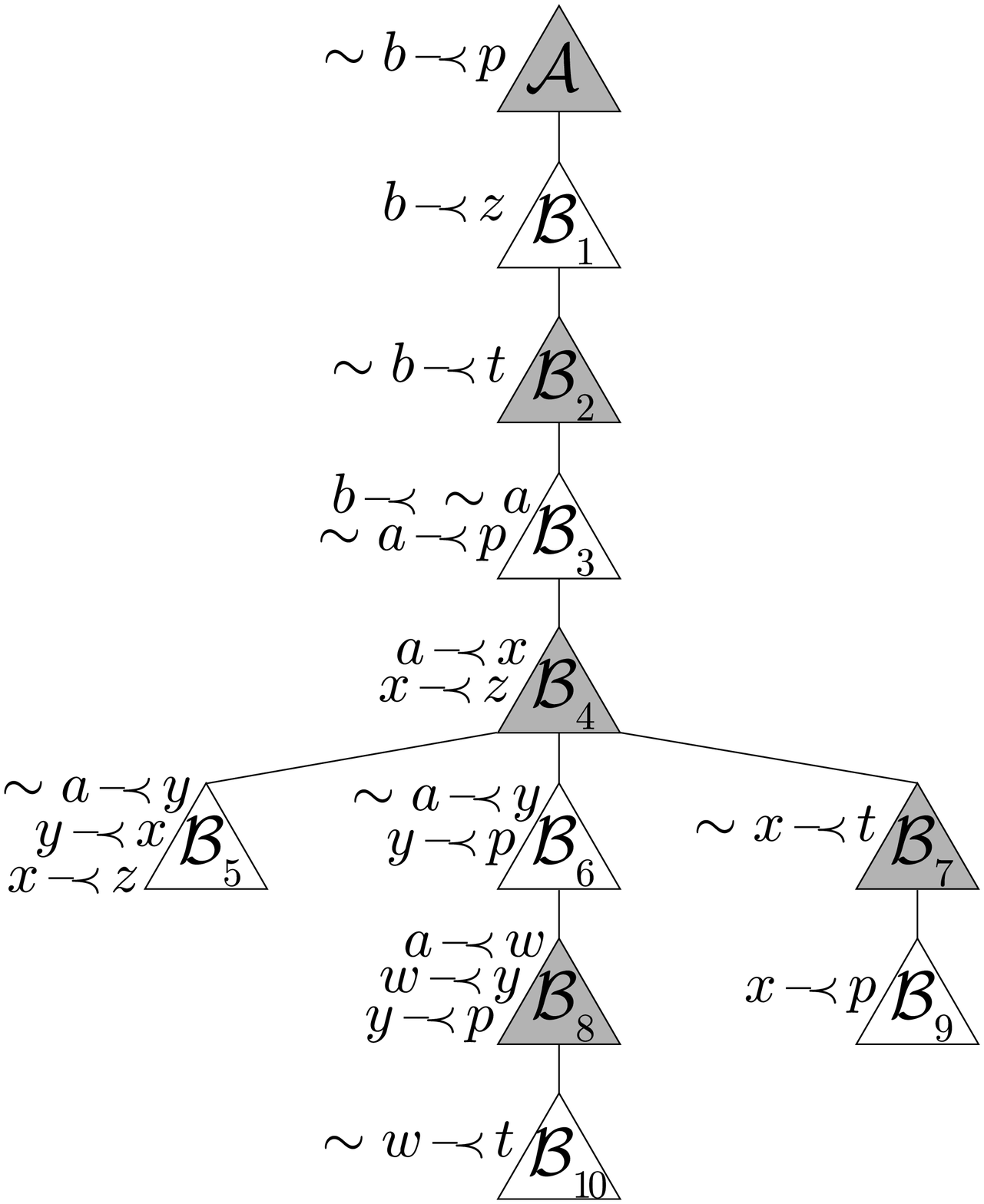,scale=.3}}},{}]
If the chosen minimal change criterion attempts to preserve the \textbf{tree structure}, $\Brg_5 = \Ar{\{(\drule{\no a}{y}),(\drule{y}{x}),(\drule{x}{z})\}}{\no a}$ and $\Brg_{10} = \Ar{\{\drule{\no w}{t}\}}{\no w}$ are selected. The unique choice to incise $\Brg_{10}$ does not collide with any other argument, that is, it satisfies the requirement of being strict-preserving. The only possible incision over $\Brg_5$ is $(\drule{y}{x})$, as the other rules are shared with $\Brg_4$ and $\Brg_6$.

\indent When considering the combined criterion that attempts to preserve both the \textbf{tree structure and program rules}, there would be two choices: $\Brg_1$ and $\Brg_3$. These are the only con arguments belonging to the shared segment of the two attacking lines in the tree. Finally, the criterion chooses the lowest one in the line aiming to preserve the tree structure; that is, $\Brg_3$. Note that this argument does not intersect with any other argument in the tree and, therefore, this is a clean incision.
\end{window}
\end{Example}

\subsection{An Algorithm for Argument Revision}

Next, we present a \prolog-like program as an approach for an implementation of argument revision. \ojo{The given algorithms constitute part of the implementation that we are currently working in towards the first fully implemented ATC approach. A computational complexity analysis is underway. However, we believe that such a detailed analysis would fall out of the scope of this article: the main objective in this section is to show how the proposed theory can be easily implemented in \prolog-like programs by taking advantage of some distinctive characteristics of the logic paradigm like backtracking.}

\begin{algorithm}
	[!h]
	\begin{footnotesize}
		\caption{Argument Revision}\label{algo}
		\begin{algorithmic}
\STATE
		\REQUIRE \delp\ $\PP = \SD$ and an argument \Aalpha
		\ENSURE Revised \delp\ $\PP \OPwparp \Arg = (\Pi,\DD_R)$
\STATE 
			\STATE $revise(\SD,\Aalpha,(\Pi,\DD_R)) \leftarrow$
			\STATE \hspace{0.5cm} $union(\DD,\Arg,\DD_{\tiny \Arg}),$
			\STATE \hspace{0.5cm} $assert\_lines((\SSet,\DD_{\tiny \Arg}),\Arg),$ \ \ \texttt{\scriptsize \%facts line/1}
			\STATE \hspace{0.5cm} $initialize\_selection\_orders,$ \ \ \texttt{\scriptsize \%facts order/2}
			\STATE \hspace{0.5cm} $assert\_att\_set,$ \ \ \texttt{\scriptsize \%facts attacking/1}
			\STATE \hspace{0.5cm} $get\_incisions,$
			\STATE \hspace{0.5cm} $get\_inc\_aware\_alteration\_set(\aware{\Arg}{\PP}),$
			\STATE \hspace{0.5cm} $findall(\sigma,(member(\lambda,\aware{\Arg}{\PP}), incision(\sigma,\_,\lambda)),\Sigma),$
			\STATE \hspace{0.5cm} $subtract(\DD_{\tiny \Arg},\Sigma,\DD_R).$
			\STATE
			\STATE $get\_incisions \leftarrow$
			\STATE \hspace{0.5cm} $retractall(incision(\_,\_,\_)),$
			\STATE \hspace{0.5cm} $forall(line(\lambda), get\_alteration(\lambda)),$
 			\STATE \hspace{0.5cm} $forall(line(\lambda), preservation(\lambda)), !,$
			\STATE \hspace{0.5cm} $(not(update\_order\_wrt\_upper(\_,\_));
			update\_orders,get\_incisions).$
			\STATE
			\STATE $get\_alteration(\lambda) \leftarrow$
			\STATE \hspace{0.5cm} $select(\gamma,\lambda),$
			\STATE \hspace{0.5cm} $incise(\sigma,\gamma),$
			\STATE \hspace{0.5cm} $assert(incision(\sigma,\gamma,\lambda)).$
			\STATE
			\STATE $preservation(\lambda') \leftarrow$
			\STATE \hspace{0.5cm} $incision(\sigma',\_,\lambda'),$
			\STATE \hspace{0.5cm} $forall(get\_upmost\_collateral(\sigma',\lambda,\Brg),$
			\STATE \hspace{0.5cm} $(incision(\sigma,\gamma,\lambda), in\_upper\_segment(\gamma,\Brg,\lambda)$
			\STATE \hspace{0.5cm} $;$
			\STATE \hspace{0.5cm} $assert(update\_order\_wrt\_upper(\lambda,\Brg))).$
			\STATE
			\STATE $select(\lambda,\gamma) \leftarrow$
			\STATE \hspace{5mm} $order(\lambda,[\gamma|\_]).$
		\end{algorithmic}
	\end{footnotesize}
\end{algorithm}

In Algorithm~\ref{algo}\footnote{Note that every symbol in the program is a variable, \ie there are no atoms.}, the main predicate is $revise/3$, which takes a program and an (possibly external) argument, performs the revision, and returns the revised program. 
The algorithm begins by inserting argument \Arg into the set of defeasible rules of \PP, obtaining a set $\Delta_{\scriptsize\Arg}$. Next, it asserts facts $line/1$ (through predicate $assert\_lines/2$), one per line in the tree rooted in \Arg, each of which holds a list representing a sequence of arguments. Then, from those facts, it initializes the selection orders according to the criterion through predicate $initialize\_selection\_orders/0$, which asserts facts $order/2$ mapping lines to the ordering assigned to their interference sets. Afterwards, the algorithm recognizes the subset of lines belonging to the attacking set of the tree and asserts facts $attacking/1$, through predicate $assert\_att\_set/0$.

Predicate $get\_incisions/0$ first gathers all incisions and selections by asserting facts $incision/3$, one per line, through the call to predicate $get\_alteration/1$. Then, selections and incisions in every line are checked to satisfy the preservation principle. Whenever some line does not satisfy preservation, the order there has to be updated through predicate $update\_order\_wrt\_upper/2$, which removes from $order/2$ those pairs including an argument below the collateral incision, so that the selection is restricted to the collateral incision's upper segment, and $get\_incisions/0$ is invoked again. This iterative process ends when the current selection orders yield incisions satisfying preservation. The algorithm always terminates because there is always a configuration of incisions and selections that satisfies this principle (see Theorem~\ref{theo.warranting}).

Once all the selections and incisions are verified against preservation, the incision-aware alteration set of the tree at issue is calculated through the predicate\linebreak $get\_inc\_aware\_alteration\_set/1$, and for each line in it, the incisions are gathered into a set $\Sigma$, which is afterwards removed from the set $\Delta_{\scriptsize\Arg}$ of defeasible rules, thus obtaining the revised program $(\Pi,\Delta_R)$.

Algorithm~\ref{algo.inc-aware} shows the predicates to obtain the incision-aware alteration set of a dialectical tree \dtree{\Arg}{\PP} through predicate $get\_inc\_aware\_alteration\_set/1$. This predicate firstly computes the alteration set $\Theta$ ($get\_alteration\_set/2$) and then, for every candidate $X$ in $\Theta$'s power set (which is sorted by cardinality), checks whether they comply with the condition of having the context\_sensitive open set within the context-sensitive closed set. When this property is satisfied, $X$'s supersets are marked in order to avoid its evaluation. After all the candidate sets of lines (minimal wrt. $\subseteq$) are obtained, their sets of incisions are calculated and then compared among them to get the one that yields the least amount of change, according to the rule-based criterion adopted. Regarding the computation of the power set of the alteration set, a few optimizations can be done, following properties from Section~\ref{sec.atc.set.interrelation}. For instance, this would allow us to greatly simplify the computation of the incision-aware alteration set, whenever we recognise the conditions stated by Theorems~\ref{theo.alines.equals.alteration.1},~\ref{theo.weak.profitability},~\ref{theo.aware.equals.alteration} and~\ref{theo.equal.incisions}.


Predicates $context\_open/3$ and $context\_closed/2$ respond to their corresponding definitions. The former gathers all lines that are open by lines in $X$ and then performs the union with the attacking set; the latter includes every line in the tree such that it receives a collateral incision over its selection, and any other collateral incision does not affect the selection's upper segment.

\begin{algorithm}[!h]
\begin{footnotesize}\caption{Incision-aware Alteration Set}\label{algo.inc-aware}
\begin{algorithmic}
\STATE $get\_inc\_aware\_alteration\_set(IncAwareSet) \leftarrow$
	\STATE \hspace{0.5cm} $retractall(avoid\_supersets(\_)),$
	\STATE \hspace{0.5cm} $get\_open\_set(\open),$
	\STATE \hspace{0.5cm} $get\_alteration\_set(\open, \alteration{\Arg}{\PP}),$
	\STATE \hspace{0.5cm} $powerset(\alteration{\Arg}{\PP}, P),$
	\STATE \hspace{0.5cm} $findall(X,$
		\STATE \hspace{1.0cm} $(member(X,P),$
		\STATE \hspace{1.0cm} $inc\_aware(\open,X),$
		\STATE \hspace{1.0cm} $assert(avoid\_supersets(X))),$
		\STATE \hspace{1.0cm} $MinSubsets),$
	\STATE \hspace{0.5cm} $minimal\_change(MinSubsets,[IncAwareSet|\_]).$
	\STATE
\STATE $inc\_aware(\open,X) \leftarrow$
	\STATE \hspace{0.5cm} $not((avoid\_supersets(S),subset(S,X))),$
	\STATE \hspace{0.5cm} $context\_open(\open,X,\openContext),$
	\STATE \hspace{0.5cm} $context\_closed(X,\closedContext),$
	\STATE \hspace{0.5cm} $subset(\openContext,\closedContext).$
\STATE
\STATE $context\_open(\open,X,CO) \leftarrow$
	\STATE \hspace{0.5cm} $findall(OL,$
		\STATE \hspace{1.0cm} $(member(o(O,OL),\open),member(O,X)),$
		\STATE \hspace{1.0cm} $F),$
	\STATE \hspace{0.5cm} $flatten(F,XOpen), attacking\_set(Att), append(Att,XOpen,CO).$
\STATE
\STATE $context\_closed(X,CC) \leftarrow$
	\STATE \hspace{0.5cm} $lines(Lines),$
	\STATE \hspace{0.5cm} $findall(\lambda',$
		\STATE \hspace{1.0cm} $(member(\lambda',Lines),$
		\STATE \hspace{1.0cm} $member(\lambda,X),$
		\STATE \hspace{1.0cm} $selection(\lambda,S),incision(S,I),$
		\STATE \hspace{1.0cm} $con\_args(\lambda',Con),member(\Crg,Con),rules\_in(\Crg,Cr),$
		\STATE \hspace{1.0cm} $intersection(I,Cr,NonEmpty),NonEmpty \backslash= [],$
		\STATE \hspace{1.0cm} $forall(member(\lambda'',X),$
		 	\STATE \hspace{1.5cm} $(selection(\lambda'',S2),incision(S2,I2),$
			\STATE \hspace{1.5cm} $args(\lambda',LPargs),member(\Brg,LPargs),rules\_in(\Brg,Br),$
		 	\STATE \hspace{1.5cm} $intersection(I2,Br,NonEmpty2),$
		 	\STATE \hspace{1.5cm} $(NonEmpty2 \backslash= [],$
		 	\STATE \hspace{1.5cm} $upper\_segment(B,\lambda',U),member(C,U)$
		 	\STATE \hspace{1.5cm} $;true))$
		 	\STATE \hspace{1.5cm} $)),CC).$
\end{algorithmic}
\end{footnotesize}
\end{algorithm}

\begin{algorithm}[!h]
\begin{footnotesize}\caption{Strict Preservation}\label{algo.strict}
\begin{algorithmic}
  \STATE $get\_incisions \leftarrow$
  \STATE \hspace{5mm} $retractall(incision(\_,\_,\_)),$
  \STATE \hspace{0.5cm} $forall(line(\lambda), get\_alteration),$ \STATE \hspace{0.5cm} $forall(line(\lambda), strict\_preservation(\lambda)),$
  \STATE \hspace{0.5cm} $(not(update\_to\_next\_selection(\_))$
  \STATE \hspace{0.5cm} $;$
  \STATE \hspace{0.5cm} $update\_orders\_wrt\_strict\_preservation,get\_incisions).$
  \STATE
  \STATE $strict\_preservation(\lambda') \leftarrow$
  \STATE \hspace{8mm} $incision(\sigma',\gamma',\lambda'),$
  \STATE \hspace{8mm} $forall((get\_upmost\_collateral(\sigma',\lambda,\Brg),\Brg\neq\gamma'),$
  \STATE \hspace{8mm} $(\Brg=[]$
  \STATE \hspace{8mm} $;$
  \STATE \hspace{8mm} $assert(update\_to\_next\_selection(\lambda')))).$
\end{algorithmic}
\end{footnotesize}
\end{algorithm}

In addition to preservation, we give the alternative to pursue extra restrictions to control the third axis of change, regarding the desired behavior of incisions and selections. Algorithm~\ref{algo.strict} shows another rule for predicate $get\_incisions/0$, intending to achieve a selection plus incision satisfying strict preservation, \ie an incision which does not collaterally affect any argumentation line. In this case, the convention is that $get\_upmost\_collateral/3$ returns in \Brg an empty list. Whenever the original order does not meet this condition (\ie $\Brg\neq []$), the order is updated and strict preservation is checked again, until an order satisfies this principle or an update is no longer possible. That is, unlike preservation, the strict preservation principle could fail to be satisfied. It is important to note that once strict preservation is satisfied, so is preservation (see Proposition~\ref{prop.preservation}.2) and thus, there is no need to check if the latter holds. In case strict preservation fails, the revision procedure should be restarted to meet only preservation. 

Algorithm~\ref{algo.profitability} implements profitability, which requires collateral incisions to affect selections in lines belonging to the attacking set. Again, if a given selection in a line does not satisfy profitability, an update of the order is asserted for that line.

\begin{algorithm}[!h]
\begin{footnotesize}\caption{Profitability}\label{algo.profitability}
\begin{algorithmic}
  \STATE $get\_incisions \leftarrow$
  \STATE \hspace{0.5cm} $retractall(incision(\_,\_,\_)),$
  \STATE \hspace{0.5cm} $forall(line(\lambda), get\_alteration),$ \STATE \hspace{0.5cm} $forall(line(\lambda), profitability(\lambda)),$
  \STATE \hspace{0.5cm} $(not(update\_selection\_to\_collinc(\_,\_))$
  \STATE \hspace{0.5cm} $;$
  \STATE \hspace{0.5cm} $update\_orders\_wrt\_profitability,get\_incisions).$
  \STATE
  \STATE $profitability(\lambda') \leftarrow$
  \STATE \hspace{5mm} $incision(\sigma',\gamma',\lambda'),$
  \STATE \hspace{5mm} $forall((get\_upmost\_collateral(\sigma',\lambda,\Brg),\Brg\neq\gamma'),$
  \STATE \hspace{5mm} $(attacking(\lambda),incision(\sigma,\Brg,\lambda)$
  \STATE \hspace{5mm} $;$
  \STATE \hspace{5mm} $assert(update\_selection\_to\_collinc(\lambda,\Brg)))).$
\end{algorithmic}
\end{footnotesize}
\end{algorithm}

\section{Related Work} \label{sec:related.work}

In general, there is no literature directly related to Argument Theory Change, although some authors have developed systems that relate belief revision and argumentation \cite{Cayrol,Boella}. One of the papers closely related to our approach studies revision of logic programs \cite{Delgrande08}. Next, we will describe several approaches and their relation to our work. Afterwards, we will briefly introduce the article~\cite{jigpal} which presents a variant of ATC applied to propositional argumentation.

Regarding ideas from the classic belief revision theory applied to non-monotonic theories, in~\cite{gov.revision.nmr}, the authors study the dynamics of a simpler variant of \emph{defeasible logic} through the definition of expansion, revision and contraction operators. Here, a \emph{defeasible theory} contains facts, defeasible rules and \emph{defeaters}. The first two elements are similar to those in \DLP, whereas defeaters are rules that, instead of being used to draw conclusions, they prevent their achievement. The focus of the paper, unlike the approach we presented, is to provide a full account of postulates, which are closely related to those from the AGM model. The intuitions behind each operator do not need any special consideration, and each one of them is formally checked to comply with the corresponding set of postulates.

\cite{Benferhat95howto} presented an article primarily oriented towards the
treatment of inconsistency caused by the use of multiple sources of information.
Knowledge bases are stratified, namely each formula in the knowledge
base is associated with its level of certainty corresponding to the
layer to which it belongs.
They suggest that it is not necessary to restore consistency
in order to make sensible inferences from an inconsistent knowledge base.
Likewise, argumentation-based inference can derive conclusions supported by reasons
to believe in them, independently of the consistency of the knowledge base.

\cite{BRandEpist} studied the dynamic of a belief revision system
considering relations among beliefs in a ``derivational approach'' trying to obtain
a theory of belief revision from a more concrete epistemological theory.
According to them, one of the goals of belief revision is to generate a knowledge base
in which each piece of information is justified (by perception) or warranted by arguments containing previously held beliefs. 
The difficulty is that the set of justified beliefs can exhibit all kinds of
logical incoherences because it represents an intermediate stage in reasoning.
Therefore, they propose a theory of belief revision concerned with warrant rather
than justification.

\cite{FKS02} proposed a kind of non-prioritized revision
operator based on the use of explanations.
The idea is that an agent, before incorporating information that is
inconsistent with its knowledge, requests an explanation supporting it.
They presented a framework oriented to defeasible reasoning.
One of the most interesting ideas of this work is the generation of
defeasible conditionals from a revision process.
This approach preserves consistency in the strict knowledge
and it provides a mechanism to dynamically
qualify the beliefs as strict or defeasible.

\cite{Paglieri2006} joined argumentation and belief revision in the same conceptual framework,
highlighting the important role played by Toulmin's layout of argument in fostering such integration.
They consider argumentation as ``persuasion to believe'' and this restriction is useful to make more explicit
the connection with belief revision.
They propose a model of belief dynamics alternative to the AGM approach: \emph{data-oriented belief revision} (DBR).
Two basic informational categories (data and beliefs) are put forward in their model, to account for
the distinction between pieces of information that are simply gathered and stored by
the agent (\emph{data}), and pieces of information that the agent considers (possibly up to a
certain degree) truthful representations of states of the world (\emph{beliefs}).
Whenever a new piece of evidence is acquired through perception or communication,
it affects directly the agent's data structure and only indirectly his beliefs.
Belief revision is often triggered by information update either on a fact or on a
source: the agent receives a new piece of information, rearranges his data structure
accordingly, and possibly changes his beliefs.

\cite{Boella} showed a direct relation between argumentation
and belief revision.
They consider argumentation as persuasion to believe and that persuasion should be related
to belief revision.
More recently, \cite{Boellaplus} presented the interrelation
between argumentation and belief revision on multi-agent systems.
When an agent uses an argument to persuade another one,
he must consider not only the proposition supported by the argument, but also the
overall impact of the argument on the beliefs of the addressee.

\cite{Cayrol} proposed a revision theory upon Dung-style abstract argumentation systems. 
The main issue of any argumentation system is the selection of acceptable sets of arguments.
An argumentation semantics defines the properties required for a set of arguments to be acceptable.
The selected sets of arguments under a given semantics are called \emph{extensions} of that semantics.
Then, by considering how the set of extensions is modified under the revision process,
they propose a typology of different revisions: decisive revision and expansive revision. 
A strong restriction is posed: the newly added argument must have at most one interaction (via attack) with an argument in the system. This restriction greatly simplifies the revision problem, as multiple interactions with the original system are more common to occur, and could become difficult to handle. In ATC, this is addressed with the inclusion of subarguments and through the handle of collateralities. Moreover, the objective of \cite{Cayrol} differs from ours in that we apply (assuming it is allowed) additional change to the original argumentative framework (and consequently, to the \delp) pursuing warrant of a single argument through the analysis of dialectical trees, whereas they study how the addition of a given argument would affect the set of extensions, by looking at an arguments graph.

\cite{Delgrande08} address the problem of belief revision in (non-monotonic)
logic programming under answer set semantics: given two logic
programs $P$ and $Q$, the goal is to determine a program $R$ that
corresponds to the revision of $P$ by $Q$, denoted $P \ast Q$.
They proposed formal techniques analogous to those of distance-based
belief revision in propositional logic.
They investigate two specific operators: (logic program)
expansion and a revision operator based on the distance
between the SE models of logic programs.
However, our approach is very different.
First, we use defeasible logic programs instead of logic programs:
it is clear that defeasible logic programs are more general and more expressive than logic programs.
Second, since we want an external argument $\mathcal{A}$ to end up undefeated after the revision,
we must modify the defeasible logic program so that the conclusion of $\mathcal{A}$ is warranted.

\subsection{\ojo{ATC Applied to Propositional Argumentation}}

In the recently published article~\cite{jigpal}, ATC is applied to a propositional argumentation framework (AF) with the objective of dealing with the dynamics of knowledge of an underlying inconsistent propositional KB from where the AF is built. Thus, similarly to the proposal given in the present article, handling dynamics of arguments of the AF allows to deal with the dynamics of knowledge of the underlying inconsistent KB. The main difference regarding the utilized KBs is that in this article \delp s are used as a kind of KB, while in~\cite{jigpal}, a potentially inconsistent propositional KB is given in a more classical way.

A set of rationality postulates adapted to argumentation is also given, and therefore, the proposed model of change is related to the postulates through the corresponding representation theorem. \cite{jigpal} constitutes the main ATC approach given that it is fully axiomatizated within the theory of belief revision. Nevertheless, the theory proposed in the present article introduces an important result regarding the application of ATC to an implemented sort of argumentation system: \DLP.

In contrast to the ATC model upon which we rely in this article, the alteration of dialectical trees in~\cite{jigpal} is achieved according to an alternative but more general viewpoint: incisions are applied globally to the dialectical tree, and therefore, no selection function is needed to determine a precise argument from each argumentation line to which the incision is applied. Hence, a global incision function determines a possible set of beliefs to be removed in order to effectively alter all the necessary lines at once.

The usage of a selection function in the present article, allows to specify different criteria of minimal change as has been introduced in Section~\ref{sec.principles}: removing as few beliefs as possible from the \delp, altering as few argumentation lines as possible from the tree, and preserving the tree structure as much as possible by removing arguments placed as low as possible in each line, getting closer to the leaves.

In addition, the model presented in~\cite{jigpal} does not pursue such an extensive variety of minimal change criteria as the ones discussed here, but only avoids to lose beliefs that are not related to the revision through the postulate of \textit{core-retainment} (see page~\pageref{core-ret}). Moreover, it is important to remark that the notion of minimality is usually subjective: most approaches in classic belief revision do not obtain real minimality, but approximations to it by specifying different criteria interpreting the meaning of minimal change as we have done in this article.

\section{Conclusions}\label{sec.conclusion}

Argument Theory Change is an abstract formalism that applies the concept of revision from classic theory change to argumentation. Concretely, ATC looks for the incorporation of a new argument to the current argumentation theory, upon which it performs the necessary modifications in order for the newly inserted argument to end up warranted. In this article we focus on an implemented, working argumentation system: Defeasible Logic Programming. In \DLP, arguments are built from sets of rules, checked for minimality and consistency, and warrant of an argument is determined by building and evaluating a dialectical tree. All these elements were taken into account in this reification of ATC, yielding a very detailed version of it. Given the specific nature of this approach, Section~\ref{sec.marking} was devoted to study the properties of the \DLP\ marking procedure utilized to evaluate dialectical trees. These results constitute the foundations for elements presented afterwards. 

The complete change machinery was addressed in Section~\ref{sec.atc}: the classical notions of \emph{selection} and \emph{incision} were redefined in terms of ATC, and the argumentation-related difficulties (namely, \emph{collateral incisions}), controlled by proper, concrete principles. Desirable properties were also analyzed, characterizing certain combinations of selections/incisions. Special attention was paid to the determination of what argumentation lines to alter within the dialectical tree at issue. This alteration set was thoroughly investigated from the somewhat \emph{na\"ive} notion of \emph{attacking set} up to the evolved concept of \emph{incision-aware alteration set}, which minimizes the amount of incisions performed to the tree. Interrelations among these different kinds of sets of lines to be altered were studied, and also their relation to several properties; some of these results would be useful in the implementation. 

The necessary change operations composing the argument warranting revision were provided in Section~\ref{sec.atc} for an external argument (an argument that cannot be built from the worked \delp\ \PP and that only derives its claim when considering the set of strict rules from \PP). Regarding the main formal results, the paper provides justifications for the classification of argumentation lines, and also to ensure the correctness of the revision operators. The latter assertion refers to the two main objectives pursued throughout this article: (1) change \delp s in a controlled manner (through some kind of minimal change) towards (2) achieving warrant for the claim of the newly inserted argument. We proposed both objectives to refer to two well-known principles of change in the classic theory of belief revision: persistence of prior knowledge and primacy of new information, respectively, as originally introduced in~\cite{Dal88}. Both principles were addressed through the proposal of two of the usual postulates from belief revision readapted to argumentation theory: core-retainment and success. 

Finally, Section~\ref{sec.principles} addresses the implementation of ATC over \DLP. Several minimal change principles are proposed and discussed, clarifying the intuitive ideas given throughout the article. Most importantly, a \prolog-like algorithm is provided, showing a possible implementation for the argument revision operator. The main operations are given in detail and optimizations are suggested, by following the properties established in Section~\ref{sec.atc}, specially those relieving the potential exhaustiveness when looking for the subset of lines representing the incision-aware alteration set. Within certain conditions this computation could be even avoided.

\bibliographystyle{acmtrans}
\bibliography{my-references,related,grsrefs}

\end{document}